\documentclass[letterpaper]{article} 
\usepackage{aaai2026}  
\usepackage{times}  
\usepackage{helvet}  
\usepackage{courier}  
\usepackage[hyphens]{url}  
\usepackage{graphicx} 
\usepackage{natbib}  
\usepackage{caption} 
\usepackage{algorithm}
\usepackage{algorithmic}

\usepackage{newfloat}
\usepackage{listings}

\usepackage{amsmath,amssymb,amsfonts,amsthm}
\newtheorem{thm}{Theorem}
\newtheorem{lem}{Lemma}
\newtheorem{prop}{Proposition}

\usepackage{amsthm}
\newtheorem{proposition}{Proposition}

\usepackage{subcaption}

\usepackage{amsmath,amssymb,amsfonts,amsthm}

\usepackage{amsthm}
\usepackage{subcaption}

\usepackage{booktabs}       
\usepackage{tabularx}
\usepackage[colorlinks=true, linkcolor=black, urlcolor=black, citecolor=black]{hyperref}

\DeclareCaptionStyle{ruled}{labelfont=normalfont,labelsep=colon,strut=off} 
\floatstyle{ruled}
\newfloat{listing}{tb}{lst}{}
\floatname{listing}{Listing}
\title{HCPO: Hierarchical Conductor-Based Policy Optimization in Multi-Agent Reinforcement Learning}
\author{
    Zejiao Liu\textsuperscript{\rm 1}\equalcontrib, Junqi Tu\textsuperscript{\rm 2}\equalcontrib, Yitian Hong\textsuperscript{\rm 2}\equalcontrib, Luolin Xiong\textsuperscript{\rm 2}, Yaochu Jin\textsuperscript{\rm 3}\footnotemark[2], Yang Tang\textsuperscript{\rm 2}\thanks{Corresponding authors.}, Fangfei Li\textsuperscript{\rm 1}\footnotemark[2]\\
}
\affiliations{
    \textsuperscript{\rm 1}The School of Mathematics, East China University of Science and Technology, Shanghai, China\\
    \textsuperscript{\rm 2}The Key Laboratory of Smart Manufacturing in Energy Chemical Process, Ministry of Education, East China University of Science and Technology, Shanghai, China\\
    \textsuperscript{\rm 3}The School of Engineering, Westlake University, Hangzhou, China\\
       
\{liuzejiao, 23012389,
       y20200105\}@mail.ecust.edu.cn, xiongluolin@gmail.com, jinyaochu@westlake.edu.cn, \{yangtang, lifangfei\}@ecust.edu.cn
       
}

\usepackage{bibentry}

\begin{document}

\maketitle

\begin{abstract}
In cooperative Multi-Agent Reinforcement Learning (MARL), efficient exploration is crucial for optimizing the performance of joint policy. However, existing methods often update joint policies via independent agent exploration, without coordination among agents, which inherently constrains the expressive capacity and exploration of joint policies. To address this issue, we propose a conductor-based joint policy framework that directly enhances the expressive capacity of joint policies and coordinates exploration. In addition, we develop a Hierarchical Conductor-based Policy Optimization (HCPO) algorithm that instructs policy updates for the conductor and agents in a direction aligned with performance improvement. A rigorous theoretical guarantee further establishes the monotonicity of the joint policy optimization process. By deploying local conductors, HCPO retains centralized training benefits while eliminating inter-agent communication during execution. Finally, we evaluate HCPO on three challenging benchmarks: StarCraftII Multi-agent Challenge, Multi-agent MuJoCo, and Multi-agent Particle Environment. The results indicate that HCPO outperforms competitive MARL baselines regarding cooperative efficiency and stability.
\end{abstract}

\section{Introduction}\label{introsec}

Cooperative Multi-Agent Reinforcement Learning (MARL) methods have driven significant progress across various fields, including autonomous driving~\cite{chen2025hierarchical}, robot cooperative control~\cite{gu2023safe}, and smart grid~\cite{zhang2022multistep}. However, the increasing number of agents in the environment leads to the exponential growth of the state space and the joint action space, which brings the scalability challenge in MARL. A widely adopted solution to tackle this challenge is the Centralized Training with Decentralized Execution (CTDE) paradigm~\cite{feng2024hierarchical,na2024lagma}. It updates agents’ policies with global information during training, while ensuring that agents make decisions only based on their own local information during execution. Typical CTDE algorithms such as MADDPG~\cite{lowe2017multi}, QMIX~\cite{rashid2020monotonic}, and MAPPO~\cite{yu2022surprising} attract widespread attention for their enhanced coordination and overall effectiveness.

Under the CTDE paradigm, efficient exploration is important in MARL~\cite{zheng2021episodic, xu2023exploration,zhang2023self}. Since parameter sharing restricts the behavioral diversity among agents, consequently impairing their exploration capabilities and impeding task completion~\cite{likaleidoscope, li2025toward}, researchers have developed heterogeneous MARL algorithms~\cite{kuba2022trust, likaleidoscope}. Specifically, to tackle the non-stationarity problem arising from simultaneous agent decisions in heterogeneous settings, sequential update algorithms such as Heterogeneous-Agent Trust Region Policy Optimisation (HATRPO)~\cite{kuba2022trust} and Agent-by-agent Policy Optimization (A2PO)~\cite{wang2023order} have been proposed. This technique allows subsequent agents to integrate the action and policy information of previous agents within each training iteration. However, most of the existing CTDE algorithms, such as HATRPO and A2PO, presume that joint policy is expressed as the product of individual policies. This limits the expressive capacity of joint policy, making it difficult for multi-agent systems to explore the optimal joint policy during training. Therefore, we propose a hierarchical joint policy framework to address the limitation in expression and enhance exploration for better performance.

\begin{figure}[t]
	\centering
	\includegraphics[width=1\columnwidth]{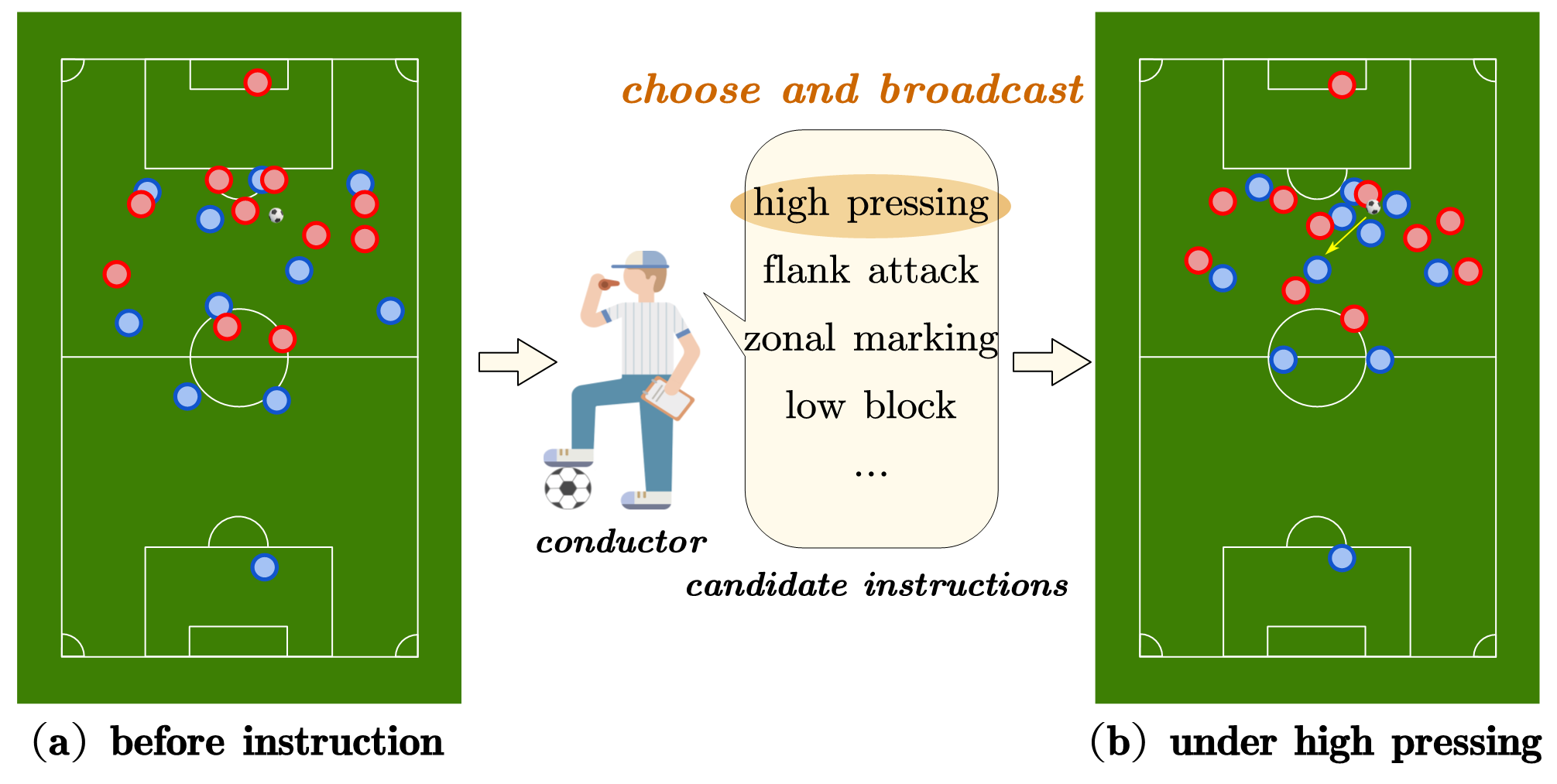} 
	\caption{Visualization of conductor-based instructional impact in multi-agent learning: Blue and red dots represent players from opposing teams, with the blue team conductor providing strategic instructions.}
	\label{conductorfig}
\end{figure}

Inspired by soccer training and competition, we propose a conductor-based method to enhance the ability of joint policy expression and exploration, as shown in Figure~\ref{conductorfig}. A centralized conductor provides instructions to the entire team by considering the on-field status of both sides. Specifically, the team can be instructed to adopt offensive instructions such as high pressing, flank attack, or defensive instructions such as zonal marking and low block. Players make decisions according to the instruction of the conductor and their own observations. For example, after the conductor chooses and broadcasts the ``high pressing'' instruction to all blue teammates, each agent produces its own action by synthesizing this global instruction with its local observation. Consequently, some players drop back to consolidate defensive coverage, while others advance to compress space, jointly generating a coordinated pressing pattern.

In this work, we develop a heterogeneous sequential MARL framework, termed \textbf{H}ierarchical \textbf{C}onductor-based \textbf{P}olicy \textbf{O}ptimization (HCPO). Firstly, a joint policy model based on the conductor's instructions is constructed for strong expressive capacity, which is inspired by a Gaussian mixture model. Subsequently, we establish a theoretically grounded update mechanism with strict monotonic improvement guarantees. By decomposing the joint policy optimization into conductor-level instruction and agent-level execution, we derive dual trust region constraints that ensure robust policy evolution. To enable practical decentralized execution, we further deploy local conductors to all agents and adapt their policies through the cross-entropy method with the centralized conductor, eliminating dependency on inter-agent communication. In addition, we perform extensive empirical validation across standardized benchmarks, including StarCraft II Multi-Agent Challenge (SMAC), Multi-agent MuJoCo (MA-MuJoCo), and Multi-agent Particle Environment (MPE) environments. The results show that HCPO outperforms existing competitive MARL algorithms. The main contributions are summarized as follows.
\begin{itemize}
	\item \textbf{Hierarchical Conductor-based Policy Expression:} To enhance policy expressive capacity and guide multi-agent exploration, we propose a conductor-based joint policy framework: \(\boldsymbol{\pi }_{\textup{mar}}(\boldsymbol{a}|s)\triangleq \mathbb{E} _{M\sim w(\cdot |s)}\boldsymbol{\pi }(\boldsymbol{a}|s,M)\). 
	\item \textbf{HCPO Algorithm and Monotonic Improvement Guarantee:} With the above conductor-based joint policy expression, we derive a new decomposition approach that breaks down the joint policy's KL divergence into two components: the conductor policies' KL divergence and the agent policies' KL divergence. Then, we present a policy improvement inequality and design a two-level policy update mechanism for the conductor and agents. Finally, we prove that HCPO can ensure a monotonic improvement of the joint policy. 
	\item \textbf{Extensive Experimental Validation:} Comprehensive evaluations on MARL benchmarks demonstrate the superior performance of HCPO over strong MARL baselines. The results show improvements in cooperative efficiency, policy stability, and exploration.
\end{itemize}

\section{Related Work}
\textbf{Exploration:} In MARL, the exponential expansion of joint state and action spaces as the number of agents increases severely challenges their ability to efficiently identify high-value states and actions~\cite{liu2021cooperative}. To improve the performance of MARL algorithm, research on the exploration of state and policy spaces is crucial. In terms of state space exploration, the classical work Multi-Agent Variational Exploration (MAVEN)~\cite{mahajan2019maven} guides multi-agent systems to learn diverse exploration patterns by maximizing the mutual information between agent trajectories and latent variables. To further improve exploration efficiency, methods like~\cite{jo2024fox} and~\cite{li2025toward} are proposed. Regarding policy space exploration, the current research focuses on constructing policy diversity incentive mechanisms~\cite{dou2024measuring}. For example, the paper~\cite{xu2023exploration} proposes an exploration method based on joint policy diversity for sparse-reward multi-agent tasks. It drives agents to explore new policies by maximizing the cross-entropy between the current joint policy and previous joint policies. However, most existing methods, either explicitly or implicitly, assume that the joint policy is the product of individual policies~\cite{kuba2022trust, dou2024measuring, jo2024fox}. This decoupled method limits the expressive capacity of joint policy, and sometimes hampers agents' exploration during training. To address this, we design a conductor-based framework that enhances the joint policy's expressive capacity. Specifically, our method provides instructions (latent variables~\cite{mahajan2019maven, ibrahim2022hierarchical}) to guide agents' exploration in the policy space during training, enabling them to explore new and potentially high-value policies that traditional methods might overlook. 

\textbf{Hierarchical mechanism:} In recent years, hierarchical mechanisms have been increasingly adopted~\cite{vezhnevets2017feudal, ahilan2019feudal, paolo2025tag}. MAVEN~\cite{mahajan2019maven} embeds a latent space for hierarchical control within the CTDE framework, which alleviates the expressive limitations introduced by the monotonicity hypothesis in QMIX~\cite{rashid2020monotonic}. Through MAVEN, agents condition their behavior on the shared latent variable to enhance exploration and mitigate issues related to suboptimal policies. Furthermore, HAVEN~\cite{xu2023haven} proposes a QMIX-style policy optimization framework with a dual coordination mechanism between layers and agents. Skill discovery represents another key direction in hierarchical MARL~\cite{he2020skill, zhang2023discovering}, enabling agents to autonomously learn diverse teamwork skills without requiring manually designed rewards~\cite{liu2022heterogeneous, liulearning}. For example, hierarchical learning with skill discovery method~\cite{yang2020hierarchical} is a two-level MARL algorithm that uses latent skill variables and intrinsic rewards for unsupervised skill discovery. The hierarchical multi-agent skill discovery~\cite{yang2023hierarchical} further extends the research on skill discovery by introducing team and individual skills. By employing the probabilistic graphical model, it formulates multi-agent skill discovery as an inference problem and leverages transformer structure to assign skills for coordination. To handle dynamic team composition, COPA~\cite{liu2021coach} proposes a coach-player hierarchy where a centralized coach periodically broadcasts strategies derived from global information. This design, however, retains the monotonicity constraint of QMIX and relies on communication during execution. Consequently, existing methods mainly use the monotonicity hypothesis imposed by QMIX and optimize hierarchical policies through variational inference. In contrast, our HCPO provides a theoretical guarantee of monotonic improvement in joint policy performance, without requiring any hypothesis regarding the joint action-value function. Additionally, our approach combines a conductor-based framework with trust region and sequential update methods, diverging from the mutual information and variational inference approaches commonly employed.

\section{Background }
\subsection{Cooperative MARL Problem Formulation}
In this paper, we consider fully cooperative multi-agent task as a Decentralized Markov Decision Process (DEC-MDP)~\cite{bernstein2002complexity, kuba2022trust, wang2023order}, which is usually modeled as a tuple \(G=\langle \mathcal{N},\mathcal{S},\mathcal{A},\mathcal{P},r,\gamma\rangle\). Here,  \(\mathcal{N}=\{1, 2, \cdots, N\}\) is the set of agents.  \(\mathcal{S}\) and  \(\mathcal{A}\) represent the state space of the environment and the whole joint action space, respectively. Each agent \(i\) takes action \(a^i \in \mathcal{A}^i\), and \(\mathcal{A}=\prod_{i=1}^N{\mathcal{A}^i}\).
\(\mathcal{P}:\mathcal{S}\times{\mathcal{A}\times\mathcal{S}}\to[0,1]\) is the transition probability function. \(r: \mathcal{S}\times \mathcal{A}\rightarrow \mathbb{R}\) is the reward function shared by all agents and \(\gamma \in [0, 1)\) is the discount factor. To improve coordinated exploration, we propose a hierarchical conductor-based framework that adapts agent policies with the instructions from the conductor. Specifically, let \(M\) denote the conductor’s instructional decision, sampled from a \textit{centralized} instruction preference distribution \(w(\cdot |s)\). The multi-agent joint policy \(\boldsymbol{\pi }_{\textup{mar}}\) is formulated as the expectation over all possible instructions \(M\): \(\boldsymbol{\pi }_{\textup{mar}}(\boldsymbol{a}|s)\triangleq \mathbb{E} _{M\sim w(\cdot |s)}\boldsymbol{\pi }(\boldsymbol{a}|s,M)\), where \(\boldsymbol{a}=(a^1, a^2, \cdots, a^N)\in \mathcal{A}\) is the joint action. To facilitate decentralized execution, we equip each agent with an independent \textit{local} conductor \(w^i(\cdot|o^i)\), sharing the same instruction space as the \textit{centralized} conductor. For clarity, the term ``conductor” is primarily used to denote the centralized conductor unless otherwise specified. Hereafter, we assume that the conductor has \(K\) discrete instructions available to choose at each time. For any given instruction \(M\), the corresponding instruction-conditional joint policy is defined as the product of individual agent policies conditioned on \(M\), \textit{i.e.} \( \boldsymbol{\pi }(\boldsymbol{a}|s,M)=\prod_{i=1}^N{\pi ^i(a^i|s,M)}\). Therefore, under the conductor-based framework, our goal is to maximize the expected cumulative reward:
\begin{equation}\label{maxJ}
	J(\boldsymbol{\pi }_{\textup{mar}})\triangleq  \mathbb{E} _{\mathfrak{s}_{\rho_{\boldsymbol{\pi }_{\textup{mar}}}}^{0:\infty}, \mathfrak{M}_{w}^{0:\infty},\mathfrak{a}_{\boldsymbol{\pi }}^{0:\infty} } \left[ \sum_{t=0}^{\infty}{\gamma ^tr(s_t,\boldsymbol{a}_t)} \right].
\end{equation}
In the above equation, we denote ``\(s_{0:\infty}\sim \rho _{\boldsymbol{\pi }_{\textup{mar}}}^{0:\infty}\)'' as ``\(\mathfrak{s}_{\rho_{\boldsymbol{\pi }_{\textup{mar}}}}^{0:\infty}\)'', ``\(M_{0:\infty}\sim w_{0:\infty}\)'' as ``\(\mathfrak{M}_{w}^{0:\infty}\)'',  and ``\(\boldsymbol{a}_{0:\infty}\sim \boldsymbol{\pi }_{0:\infty}\left( M_{0:\infty} \right)\)'' as ``\(\mathfrak{a}_{\boldsymbol{\pi }}^{0:\infty}\)'' for the sake of brevity. Hereafter, we use this notation wherever no ambiguity arises. Consequently, the state value function \(V_{\boldsymbol{\pi}_{\textup{mar}}}(s)\) is defined as the expected cumulative return under the multi-agent joint policy \(\boldsymbol{\pi}_{\textup{mar}}\): 
\begin{equation}\label{Vtot}
	V_{\boldsymbol{\pi }_{\textup{mar}}}(s)\triangleq \mathbb{E} _{\mathfrak{M}_{w}^{0:\infty},\mathfrak{a}_{\boldsymbol{\pi }}^{0:\infty}, \mathfrak{s}_{\rho_{\boldsymbol{\pi }_{\textup{mar}}}}^{1:\infty} } \left. [\sum_{t=0}^{\infty}{\gamma ^t}r_t|s_0=s \right. ].
\end{equation}
Similarly, we define the state-action value function~\(   Q_{\boldsymbol{\pi }_{\textup{mar}}}(s,\boldsymbol{a})\triangleq \mathbb{E} _{\mathfrak{s}_{\rho_{\boldsymbol{\pi }_{\textup{mar}}}}^{1:\infty}, \mathfrak{M}_{w}^{1:\infty},\mathfrak{a}_{\boldsymbol{\pi }}^{1:\infty} }\left[\sum_{t=0}^{\infty}{\gamma ^t}r_t|s_0=s,\boldsymbol{a}_0=\boldsymbol{a}\right]\).~The joint advantage function is written as:~\(A_{\boldsymbol{\pi }_{\textup{mar}}}(s,\boldsymbol{a})\triangleq Q_{\boldsymbol{\pi }_{\textup{mar}}}(s,\boldsymbol{a})-V_{\boldsymbol{\pi }_{\textup{mar}}}(s).\)

\subsection{Sequential Policy Update Mechanism}
To address non-stationarity issues in MARL, the sequential update for agents has been widely investigated~\cite{kuba2022trust,wang2023order, dou2024measuring,wan2025srsv}. HATRPO extends the single-agent Trust Region Policy Optimization (TRPO) to multi-agent domain, utilizing multi-agent advantage decomposition to facilitate the sequential updates of agent policies. In HATRPO, each agent updates its policy parameters through the following~protocol:
\begin{align}
	\theta_{k+1}^{i_n} &= \arg \max_{\theta^{i_n}} \mathbb{E}_{s, \boldsymbol{a}^{i_{1:n-1}}, a^{i_n}} \left[ A_{\boldsymbol{\pi}_{\boldsymbol{\theta}_k}}^{i_n} \left( s, \boldsymbol{a}^{i_{1:n-1}}, a^{i_n} \right) \right], \label{eq3}\\
	&\text{subject to } \mathbb{E}_{s \sim \rho_{\boldsymbol{\pi}_{\boldsymbol{\theta}_k}}} \left[ \mathrm{D}_{\mathrm{KL}} \left( \pi^{i_{n}}_{\theta_k^{i_n}}(\cdot | s), \pi^{i_{n}}_{\theta^{i_n}}(\cdot | s) \right) \right] \leq \delta.
\end{align}
In~\eqref{eq3}, the compact notation \(\mathbb{E}_{s, \boldsymbol{a}^{i_{1:n-1}}, a^{i_n}}\) stands for the full expectation \(\mathbb{E}_{s \sim \rho_{\boldsymbol{\pi}_{\boldsymbol{\theta}_k}}, \boldsymbol{a}^{i_{1:n-1}} \sim \boldsymbol{\pi}^{i_{1:n-1}}_{\boldsymbol{\theta}_{k+1}^{i_{1:n-1}}}, a^{i_n} \sim \pi_{\theta_k^{i_n}}}\). Here, \(\theta_{k+1}^{i_n}\) denotes the policy parameter for agent \(i_n\) in episode \(k+1\). \(\mathrm{D}_{\mathrm{KL}} \left(\cdot , \cdot \right)\) is the KL-divergence between two policies, and \(\delta\) is a hyperparameter. It is important to note that during the \(k+1\)-th episode, policy \(\theta_{k+1}^{i_n}\) leverages the updated policies \(\boldsymbol{\pi}^{i_{1:n-1}}_{\boldsymbol{\theta}_{k+1}^{i_{1:n-1}}}\) from the preceding agents \(i_{1:n-1}\). This is the core idea behind the sequential update mechanism.

\section{Methods}

\begin{figure*}[t]
	\centering
	\includegraphics[width=0.7\textwidth]{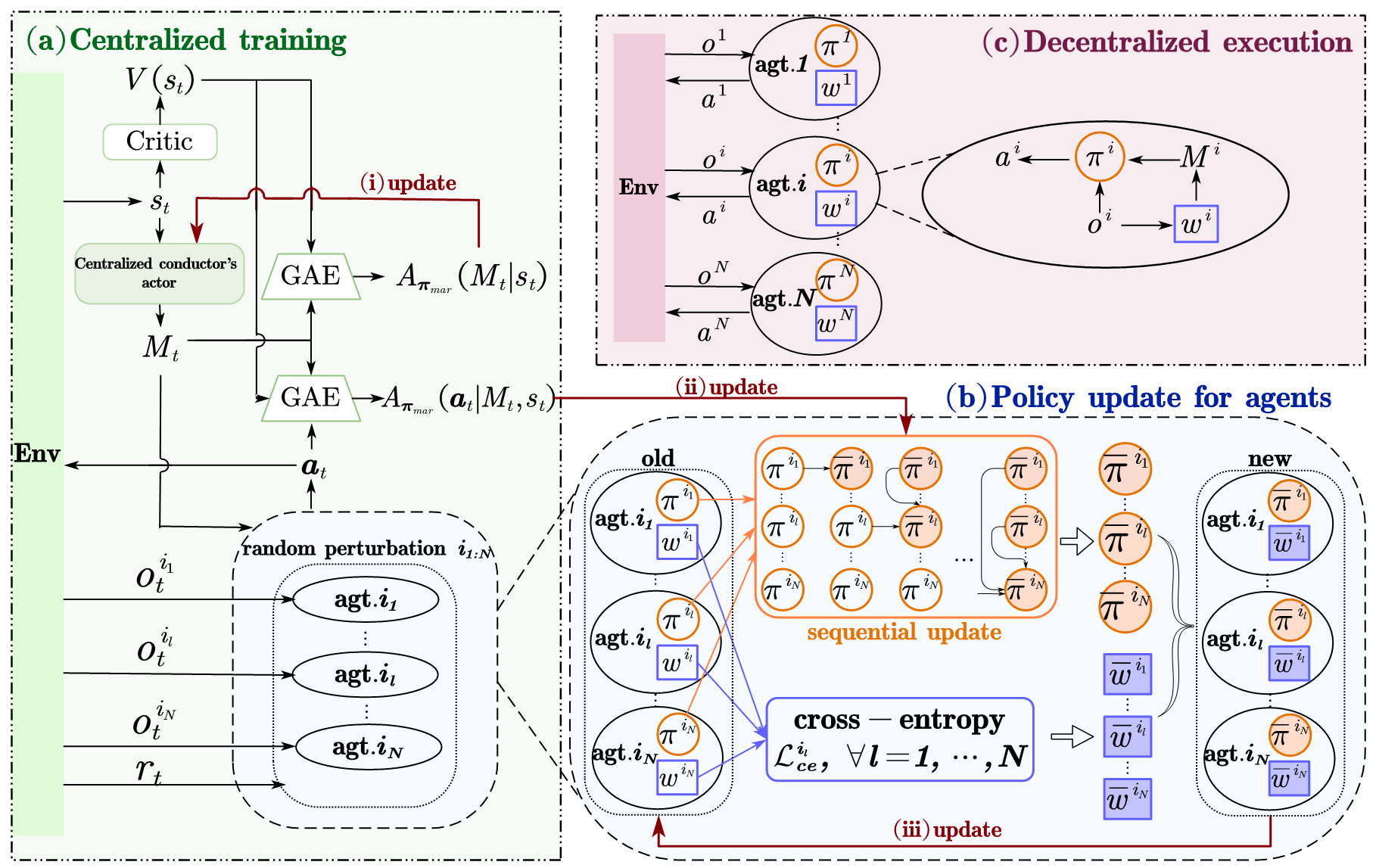} 
	\caption{The overall framework of HCPO. \textbf{(a) Centralized training:} A two-level policy update mechanism with a virtual centralized conductor is proposed, leveraging well-designed advantage functions. \textbf{(b) Policy update for agents:} Here, local agents' policies denoted by orange ellipses are optimized through sequential updates, and local conductors' policies denoted by blue rectangles are optimized through the cross-entropy method. During this iteration, policies with shaded outlines represent updated versions, while those without shading indicate unmodified ones. \textbf{(c) Decentralized execution:} HCPO enables agents to make decisions based only on local information.}
	\label{figHCPO}
\end{figure*}

\subsection{Value Functions}
In the previous section, we presented a conductor-based framework that enhances learning through instruction guidance. To evaluate this framework, we now need quantitative metrics to assess both the learning process and policy performance.
First, we define the value function for the conductor's instruction \(M\) as:~\( Q_{\boldsymbol{\pi }_{\textup{mar}}}(M|s) \triangleq \mathbb{E} _{\mathfrak{a}_{\boldsymbol{\pi }}^{0:\infty}, \mathfrak{s}_{\rho_{\boldsymbol{\pi }_{\textup{mar}}}}^{1:\infty}, \mathfrak{M}_{w}^{1:\infty}  } \left. \big[\sum_{t=0}^{\infty}{\gamma ^t}r_t|s_0=s,M_0=M \right.\big].
\)
The instruction advantage function is written as: 
\begin{equation}
	A_{\boldsymbol{\pi }_{\textup{mar}}}(M|s)\triangleq Q_{\boldsymbol{\pi }_{\textup{mar}}}(M|s)-V_{\boldsymbol{\pi }_{\textup{mar}}}(s).
\end{equation}
Here, \(A_{\boldsymbol{\pi }_{\textup{mar}}}(M|s)\) measures the relative benefit of instruction \(M\) compared to all other instructions \(M^\prime\sim w\left( \cdot |s \right) \). By choosing \(M\) to maximize \(A_{\boldsymbol{\pi }_{\textup{mar}}}(M|s)\) with generalized advantage estimation (GAE)~\cite{schulman2018high}, we optimize the conductor's policy \(w\), as illustrated in Figure~\ref{figHCPO}(i). Additionally, we define a joint action advantage function for agents' joint actions as follows:
\begin{equation}
	A_{\boldsymbol{\pi }_{\textup{mar}}}(\boldsymbol{a}|s,M)\triangleq Q_{\boldsymbol{\pi }_{\textup{mar}}}(s,\boldsymbol{a})-Q_{\boldsymbol{\pi }_{\textup{mar}}}(M|s).
\end{equation}
The advantage function \(A_{\boldsymbol{\pi }_{\textup{mar}}}(\boldsymbol{a}|s,M)\) evaluates the joint action \(\boldsymbol{a}\) over all other possible joint actions \(\boldsymbol{a}^\prime\sim \boldsymbol{\pi }(\cdot|s,M)\). As illustrated in Figure~\ref{figHCPO}(ii), maximizing this advantage function enables the optimization of the instruction-conditional joint policy \(\boldsymbol{\pi }(\cdot|s,M)\), thereby favoring actions that yield superior expected returns. Based on the above definitions, we can find that:
\begin{equation}
	A_{\boldsymbol{\pi}_{\textup{mar}}}(s,\boldsymbol{a})= A_{\boldsymbol{\pi}_{\textup{mar}}}(M|s) + A_{\boldsymbol{\pi}_{\textup{mar}}}(\boldsymbol{a}|s,M).
\end{equation}

After the conductor chooses an instruction \(M^j, j\in \left\{1,2,\cdots, K \right\} \), we update each individual agent's policy~\(\pi ^{i_l}\left( a^{i_l}|s,M^j \right)\) in the order determined by the random perturbation set \(i_{1:N}=\left\{ i_1, i_2, \cdots , i_N \right\}\). In addition, at the state \(s\), based on the conductor's any instruction \(M^j\), after the previous agents \(i_{1:l}\) take joint action \(\boldsymbol{a}^{i_{1:l}}\), we define the expected value of \(\boldsymbol{a}^{i_{1:l}}\) as: 
\begin{align}
	Q_{\boldsymbol{\pi }_{\textup{mar}}}^{i_{1:l}}(\boldsymbol{a}^{i_{1:l}}&|s,M^j)\triangleq \mathbb{E} _{\boldsymbol{a}_{0}^{-i_{1:l}}\sim \boldsymbol{\pi }_{0}^{-i_{1:l}}\left( \cdot |s,M^j \right) ,\mathfrak{s}_{\rho_{\boldsymbol{\pi }_{\textup{mar}}}}^{1:\infty}, \mathfrak{M}_{w}^{1:\infty},\mathfrak{a}_{\boldsymbol{\pi }}^{1:\infty}}  \notag \\
	& \left[\sum_{t=0}^{\infty}{\gamma ^t}r_t|s_0=s,M_0=M^j,\boldsymbol{a}_{0}^{i_{1:l}}=\boldsymbol{a}_{}^{i_{1:l}}\right],
\end{align}
where \(-i_{1:l}\) is the complement of \(i_{1:l}\).

\begin{lem}
	\label{prop1}
	~For any instruction \(M^j, j\in \left\{ 1,2,...,K \right\} \) chosen by the conductor, the conditional \(Q\)-function for agents \(i_{1:l}\) satisfies:
	\begin{equation}
		Q_{\boldsymbol{\pi }_{\textup{mar}}}^{i_{1:l}}(\boldsymbol{a}^{i_{1:l}}|s,M^j)=\mathbb{E} _{\boldsymbol{a}^{-i_{1:l}}\sim \boldsymbol{\pi }^{-i_{1:l}}\left( \cdot |s,M^j \right)}\left[ Q_{\boldsymbol{\pi }_{\textup{mar}}}^{}\left( s,\boldsymbol{a} \right) \right], 
	\end{equation}
	where \(\boldsymbol{a}=\left( \boldsymbol{a}^{i_{1:l}},\boldsymbol{a}^{-i_{1:l}} \right)\).   
\end{lem}
The proof is proposed in Appendix~A.1. This lemma provides the basis for the subsequent definitions of advantage functions, which are crucial for evaluating the relative advantages of specific actions compared to average actions. Specifically, at the state \(s\), based on the conductor's any action \(M^j\), after the agents \(i_{1:l}\) first take joint action \(\boldsymbol{a}^{i_{1:l}}\), and agents \(-i_{1:l}\) take joint action \(\boldsymbol{a}^{-i_{1:l}}\sim \boldsymbol{\pi }^{-i_{1:l}}\left( \cdot |s,M^j \right) \), we define the advantage function for \(\boldsymbol{a}^{i_{1:l}}\) as:
\begin{equation}\label{defn_adv1}
	A_{\boldsymbol{\pi }_{\textup{mar}}}^{i_{1:l}}\left( \boldsymbol{a}^{i_{1:l}}|s,M^j \right) \triangleq Q_{\boldsymbol{\pi }_{\textup{mar}}}^{i_{1:l}}(\boldsymbol{a}^{i_{1:l}}|s,M^j)-Q_{\boldsymbol{\pi }_{\textup{mar}}}(M^j|s).
\end{equation}
It is noted that when \(l=0\), we have \(A_{\boldsymbol{\pi }_{\textup{mar}}}^{i_{1:l}}\left( \boldsymbol{a}^{i_{1:l}}|s,M^j \right) =0\). 
Besides, at the state \(s\), based on the conductor's any instruction \(M^j\), for any individual agent \(i_l\), we define the advantage of its individual action \(a^{i_l}\) over all actions \({a^{i_l}}^\prime \sim {\pi _{}^{i_l}}^\prime(\cdot |s,M^j)\):
\begin{align}\label{defn_adv2}
		A_{\boldsymbol{\pi }_{\textup{mar}}}^{i_l}(\boldsymbol{a}^{i_{1:l-1}},a^{i_l}|s,M^j)&\triangleq Q_{\boldsymbol{\pi }_{\textup{mar}}}^{i_{1:l}}(\boldsymbol{a}^{i_{1:l}}|s,M^j) \notag\\&-Q_{\boldsymbol{\pi }_{\textup{mar}}}^{i_{1:l-1}}(\boldsymbol{a}^{i_{1:l-1}}|s,M^j).
\end{align}

In the subsequent part, we introduce the conductor-based multi-agent advantage function decomposition lemma. As shown in Figure~\ref{figHCPO}(ii), the lemma is designed to facilitate the transition from updating the instruction-conditional joint policy \(\boldsymbol{\pi }(\boldsymbol{a}|s,M^j)\) to updating individual agents' policies~\(\pi ^{i_l}\left( a^{i_l}|s,M^j \right), l\in\mathcal{N} \).

\begin{lem}\label{lemma1}
	(Conditional Advantage Decomposition) Consider a cooperative Markov game with a joint policy \(\boldsymbol{\pi }_{\textup{mar}}\). For any state \(s\), any instruction \(M^j\), and any subset of agents \(i_{1:n}=\left\{ i_1, i_2, \cdots , i_n \right\}\subseteq \mathcal{N}\), the following equation holds for all states \(s\), joint actions \(\boldsymbol{a}^{i_{1:n}}\), and \(M^j\sim w\):
	\begin{equation}\label{eq14}
		A_{\boldsymbol{\pi }_{\textup{mar}}}^{i_{1:n}}\left( \boldsymbol{a}^{i_{1:n}}|s,M^j \right) =\sum_{l=1}^n{A_{\boldsymbol{\pi }_{\textup{mar}}}^{i_l}\left( \boldsymbol{a}^{i_{1:l-1}},a^{i_l}|s,M^j \right)}. 
	\end{equation}
\end{lem}
The proof is proposed in Appendix~A.1.

\subsection{Quantifying Policy Updates for HCPO}
TRPO~\cite{schulman2015trust} is a reinforcement learning algorithm that improves learning stability by constraining policy update magnitude. In this section, we explore a conductor-based mechanism that incorporates TRPO to measure the difference in expected returns between the new policy and its predecessor. This analysis serves as the foundation for designing effective policy update algorithms.
\begin{proposition}\label{prop22}
	As defined in Equation~\eqref{maxJ}, the relationship between the expected return of the new policy \(\bar{\boldsymbol{\pi}}_{\textup{mar}}\) and the old policy \(\boldsymbol{\pi}_{\textup{mar}}\) is expressed as:
	\begin{equation}~\label{eq15}
		J(\bar{\boldsymbol{\pi}}_{\textup{mar}})=J(\boldsymbol{\pi }_{\textup{mar}})+\mathbb{E} _{\tau \sim \bar{\boldsymbol{\pi}}_{\textup{mar}}}\left[ \sum_{t=0}^{\infty}{\gamma ^t}A_{\boldsymbol{\pi }_{\textup{mar}}}(s_t,\boldsymbol{a}_t) \right],
	\end{equation}
	where \(\tau:=(s_0, M_0, \boldsymbol{a}_0, s_1, M_1, \boldsymbol{a}_1, \cdots)\).
\end{proposition}
The proof is proposed in Appendix~A.1. Similar to HATRPO, we introduce the approximation function \(L_{\boldsymbol{\pi }_{\textup{mar}}}(\bar{\boldsymbol{\pi}}_{\textup{mar}})\), which serves as an alternative objective function for the new policy's performance function \(J(\bar{\boldsymbol{\pi}}_{\textup{mar}})\): 
\begin{equation}\label{alterfun}
	L_{\boldsymbol{\pi }_{\textup{mar}}}(\bar{\boldsymbol{\pi}}_{\textup{mar}})
	=J(\boldsymbol{\pi }_{\textup{mar}})+\mathbb{E} _{s\thicksim \rho _{ \boldsymbol{\pi }_{\textup{mar}}},M\sim\bar{w},\boldsymbol{a}\thicksim \bar{\boldsymbol{\pi}}} \left[ A_{\boldsymbol{\pi }_{\textup{mar}}}(s,\boldsymbol{a}) \right],
\end{equation}
where the compact notation \(\mathbb{E} _{s\thicksim \rho _{ \boldsymbol{\pi }_{\textup{mar}}},M\sim\bar{w},\boldsymbol{a}\thicksim \bar{\boldsymbol{\pi}}}\) stands for the full expectation \(\mathbb{E} _{s\thicksim \rho _{ \boldsymbol{\pi }_{\textup{mar}}},M\sim\bar{w}\left( \cdot |s \right) ,\boldsymbol{a}\thicksim \bar{\boldsymbol{\pi}}\left( \cdot |s,M \right)}\). Therefore, we can derive the following theorem.
\begin{thm}\label{thm11}
	Under the proposed conductor-based framework, a significant policy improvement inequality holds for the joint policy \(\boldsymbol{\pi }_{\textup{mar}}\):
	\begin{align}\label{eq27}
	&J(\bar{\boldsymbol{\pi}}_{\textup{mar}}) \geqslant J(\boldsymbol{\pi }_{\textup{mar}}) + \mathbb{E} _{\mathrm{s}\thicksim \rho _{ \boldsymbol{\pi }_{\textup{mar}}}}\left.\Big[ 
	\mathbb{E} _{M\sim \bar{w}\left( \cdot |s \right)}A_{ \boldsymbol{\pi }_{\textup{mar}}}(M|s) \right.	\notag \\
	&\left.  - C\mathrm{D}_{\mathrm{KL}}^{\max}\left( w ,\overline{w} \right) + \mathbb{E} _{M\sim \bar{w}\left( \cdot |s \right) ,\boldsymbol{a}\thicksim \bar{\boldsymbol{\pi}}\left( \cdot |s,M \right)}A_{ \boldsymbol{\pi }_{\textup{mar}}}(\boldsymbol{a}|s,M)\right.	\notag \\
	&\left. -\max C\sum_{j=1}^K{w\left( M^j|s \right) \mathrm{D}_{\mathrm{KL}}\left( \boldsymbol{\pi }(\cdot |s,M^j),\overline{\boldsymbol{\pi }}(\cdot |s,M^j) \right)} \right.\Big],
	\end{align}
	where \(C = \frac{4\gamma \max_{s,a} |A_{\boldsymbol{\pi }_{\textup{mar}}}(s,\boldsymbol{a})|}{(1-\gamma)^2}\),~\(\mathrm{D}_{\mathrm{KL}}^{\max} \left( w_k, \bar{w} \right)=\max_s\mathrm{D}_{\mathrm{KL}}\left( w_k\left( \cdot |s \right), \bar{w}\left( \cdot |s \right) \right)\).
\end{thm}
For proof see Appendix~A.2. The inequality quantifies the expected return difference between the new policy \(\bar{\boldsymbol{\pi}}_{\textup{mar}}\) and the existing policy \(\boldsymbol{\pi }_{\textup{mar}}\)
under the conductor-based mechanism. This finding offers precise guidance for the subsequent policy update process.

\subsection{Guaranteed Monotonic Joint Policy Optimization}
In this section, we formulate the policy update mechanisms for both the centralized conductor's policy and individual agents' policies. Our theoretical analysis aims to establish the monotonic improvement guarantee for the conductor-based joint policy \(\boldsymbol{\pi }_{\textup{mar}}\) through a two-level optimization framework. This guarantee constitutes a critical theoretical foundation, ensuring progressive performance improvement through successive policy iterations while validating the conductor's instructions. The specific two-level policy update mechanisms in episode \(k+1\) are as follows:

(i)~The conductor's policy \(w\left(\cdot|s \right) =w_k\left(\cdot|s \right) \) is updated first according to the rule:
\begin{align}\label{eq34}
	w_{k+1} = \underset{\bar{w}}{\mathrm{arg\,max}}&~\Bigg[
		\mathbb{E}_{s\sim \rho_{\boldsymbol{\pi}_{\textup{mar},k}}, M \sim \bar{w}} A_{\boldsymbol{\pi}_{\textup{mar},k}}(M | s)  \notag \\&  - C \mathrm{D}_{\mathrm{KL}}^{\max} \left( w_k, \bar{w} \right)   \Bigg].
\end{align}
	
(ii)~For each \( M^j\) where \(j\in \left\{ 1,2,...,K \right\} \), the agents sequentially update their policies in accordance with the order \(i_{1:N}\), following the update rule:
\begin{align}\label{eq35}
	&\pi _{k+1}^{i_l}(\cdot |s,M^j)\notag \\&=\underset{\bar{\pi}^{i_l}\left( \cdot |s,M^j \right)}{arg\max}\left[ w_{k+1}\left( M^j|s \right)  L_{ \boldsymbol{\pi }_{\textup{mar}, k}}^{i_{1:l}}\left( \boldsymbol{\pi }_{k+1}^{i_{1:l-1}},\bar{\pi}^{i_l}|s,M^j \right) \right. \notag \\&     \left.-\max C w_k\left( M^j|s \right) \mathrm{D}_{\mathrm{KL}}\left( \pi _{k}^{i_l}(\cdot |s,M^j),\bar{\pi}^{i_l}(\cdot |s,M^j) \right)  \right],
\end{align}
	where \(L_{ \boldsymbol{\pi }_{\textup{mar},k}}^{i_{1:l}}\left( \boldsymbol{\pi }_{k+1}^{i_{1:l-1}},\bar{\pi}^{i_l}|s,M^j \right) \triangleq \mathbb{E} _{\boldsymbol{a}^{i_{1:l-1}}\sim \boldsymbol{\pi }_{k+1}^{i_{1:l-1}},a^{i_l}\sim \bar{\pi}^{i_l}}  \left[ A_{ \boldsymbol{\pi }_{\textup{mar},k}}^{i_l}\left( \boldsymbol{a}^{i_{1:l-1}},a^{i_l}|s,M^j \right) \right]\).~For proof see Appendix~A.4. Equation~\eqref{eq35} presents a policy update mechanism using \(w_k\) to regulate the update magnitude between the agents' new and existing policies. This approach optimizes the exploration-exploitation trade-off, helps avoid local optima, and enhances the adaptability and effectiveness of the policy update process. In HCPO's practical implementation, we employ the CTDE framework to mitigate the limitations posed by communication interference, as shown in Figure~\ref{figHCPO}(c). During training, we incorporate a virtual \textit{centralized} conductor \(w\) parameterized by \(\Psi\), and each individual agent is equipped with a \textit{local} conductor network \(w^i\) (parameterized by \(\psi^i\)) and an actor network \(\pi^i\) (parameterized by~\(\theta^i\)). To maintain the theoretical assumption that agents update their policies using only local observations and the broadcast instruction \(M\), we employ a two-stage training protocol. First, we collect experience with the centralized conductor and optimize its policy parameter \(\Psi\). Then, every agent’s actor network \(\theta^i\) is sequentially updated to improve the policy. Second, we fix \(\Psi\) and distill its policy into~\(\psi^i \) via cross-entropy loss~\cite{chen2024ptde}, enabling fully decentralized execution at evaluation. We summarize our HCPO as Algorithm~1 of Appendix~B. 
\begin{figure*}[htbp] 
	\centering
	\begin{subfigure}{0.27\textwidth}
		\includegraphics[width=\linewidth]{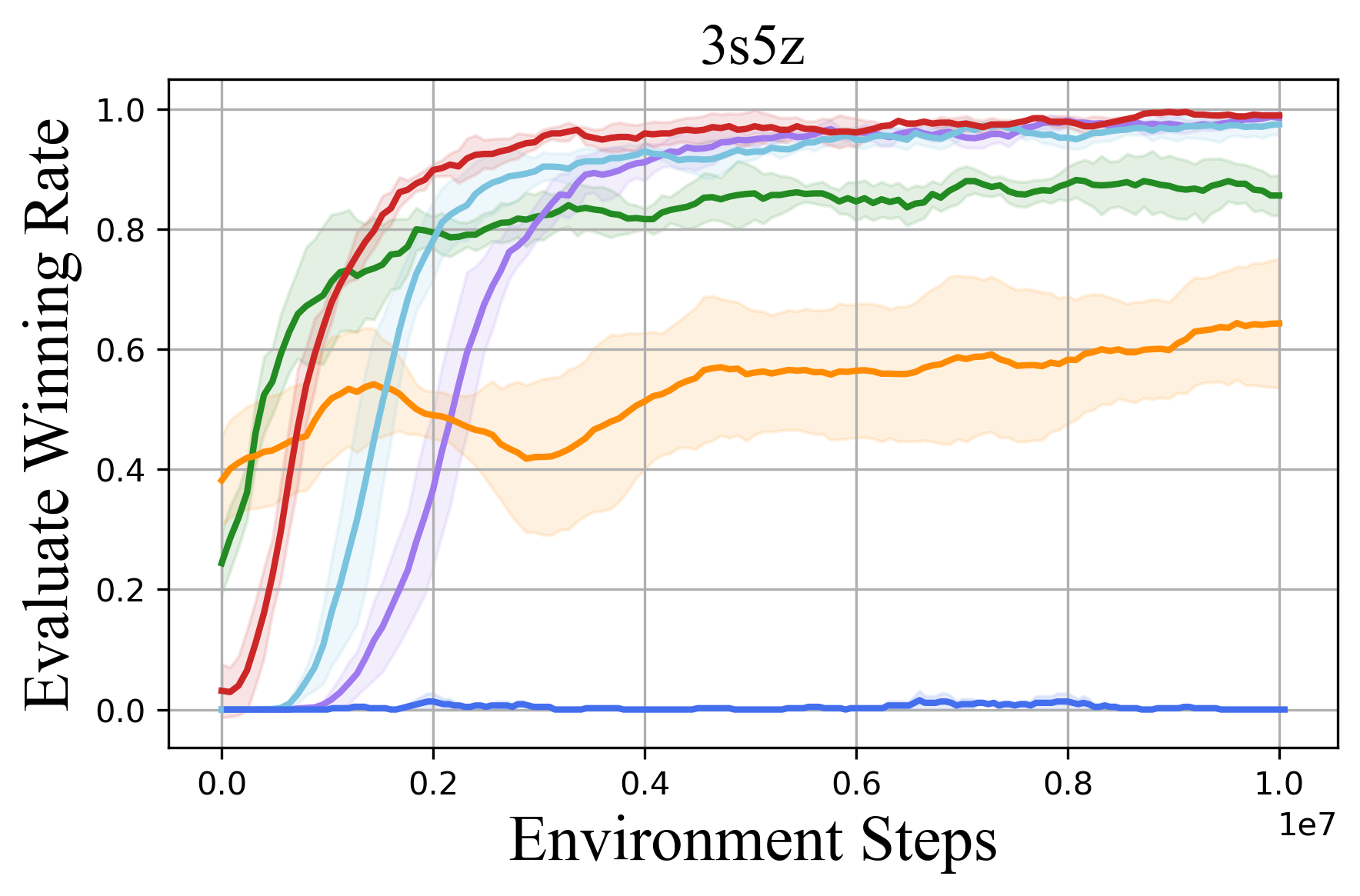}
	\end{subfigure}%
	\begin{subfigure}{0.27\textwidth}
		\includegraphics[width=\linewidth]{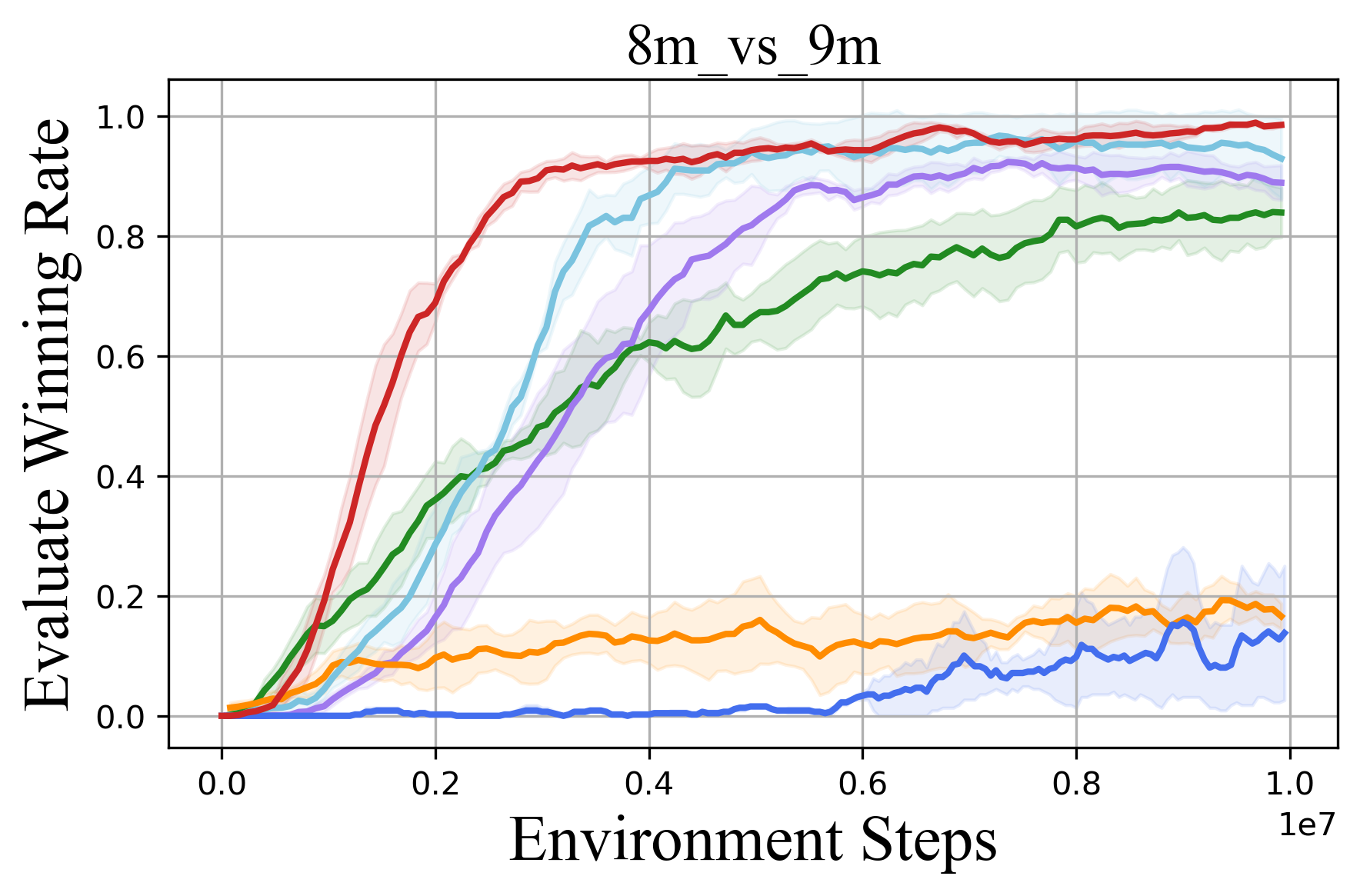}
	\end{subfigure}%
	\begin{subfigure}{0.27\textwidth}
		\includegraphics[width=\linewidth]{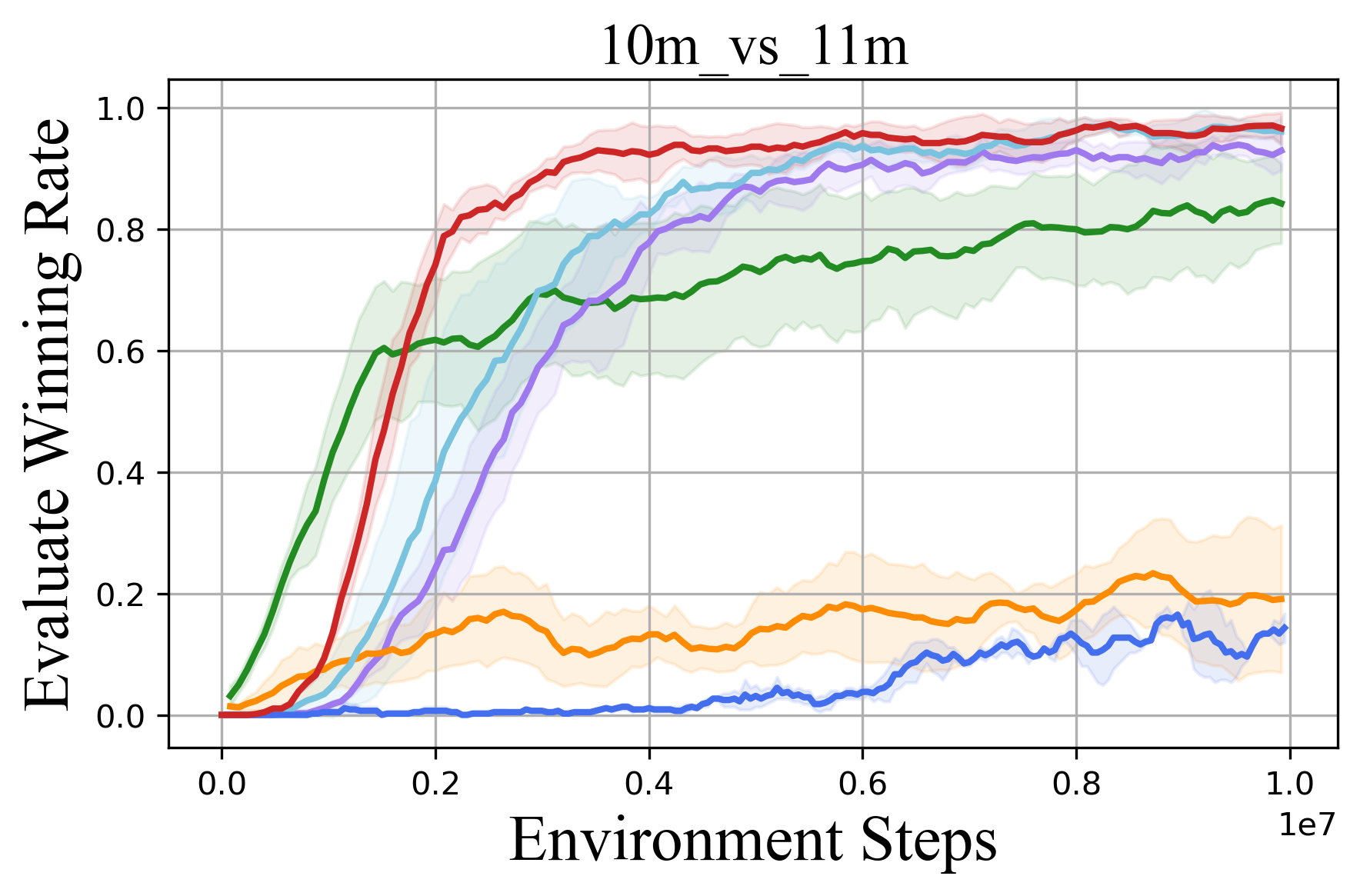}
	\end{subfigure}
	
	\begin{subfigure}{0.27\textwidth}
		\includegraphics[width=\linewidth]{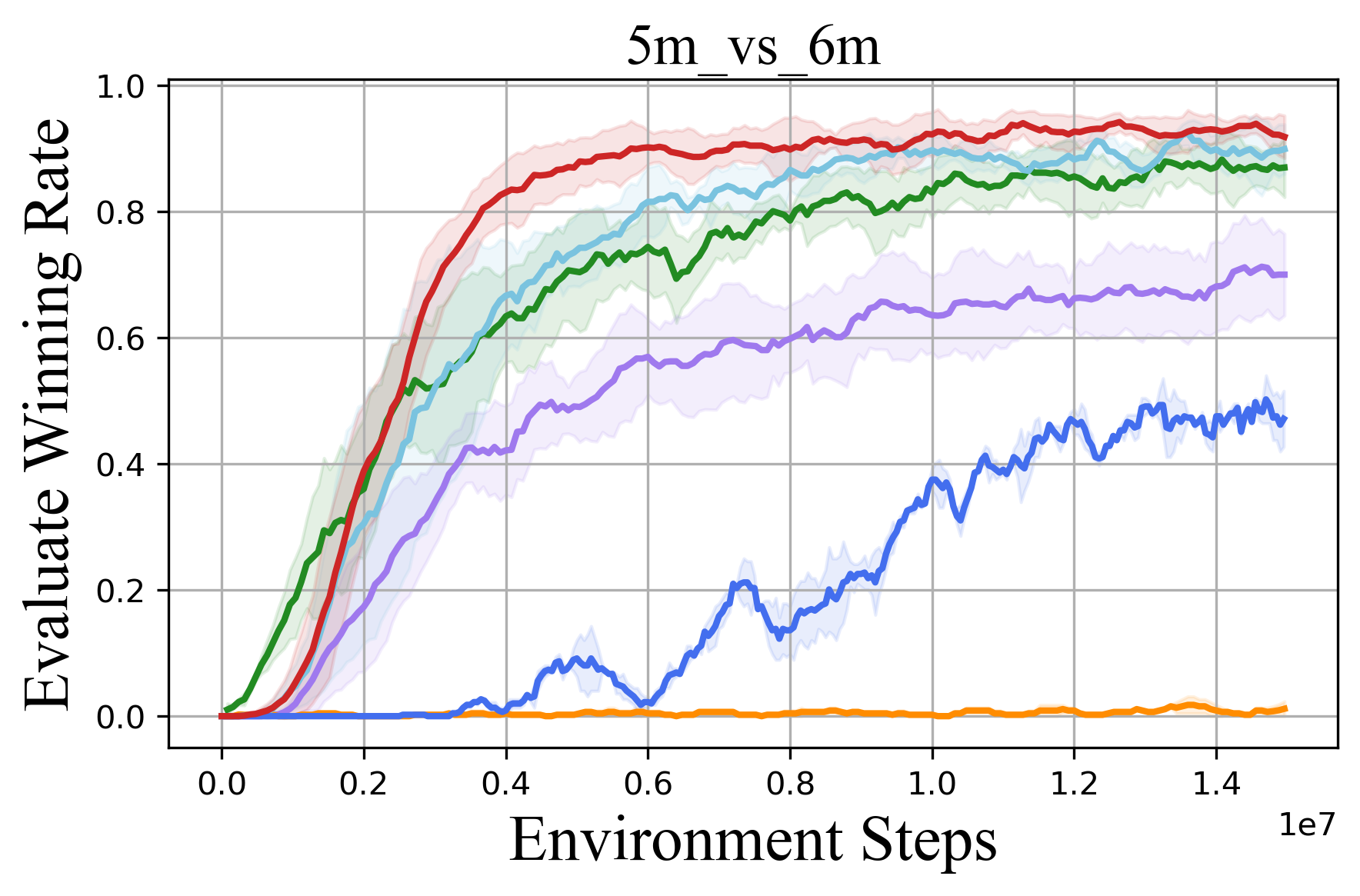}
	\end{subfigure}%
	\begin{subfigure}{0.27\textwidth}
		\includegraphics[width=\linewidth]{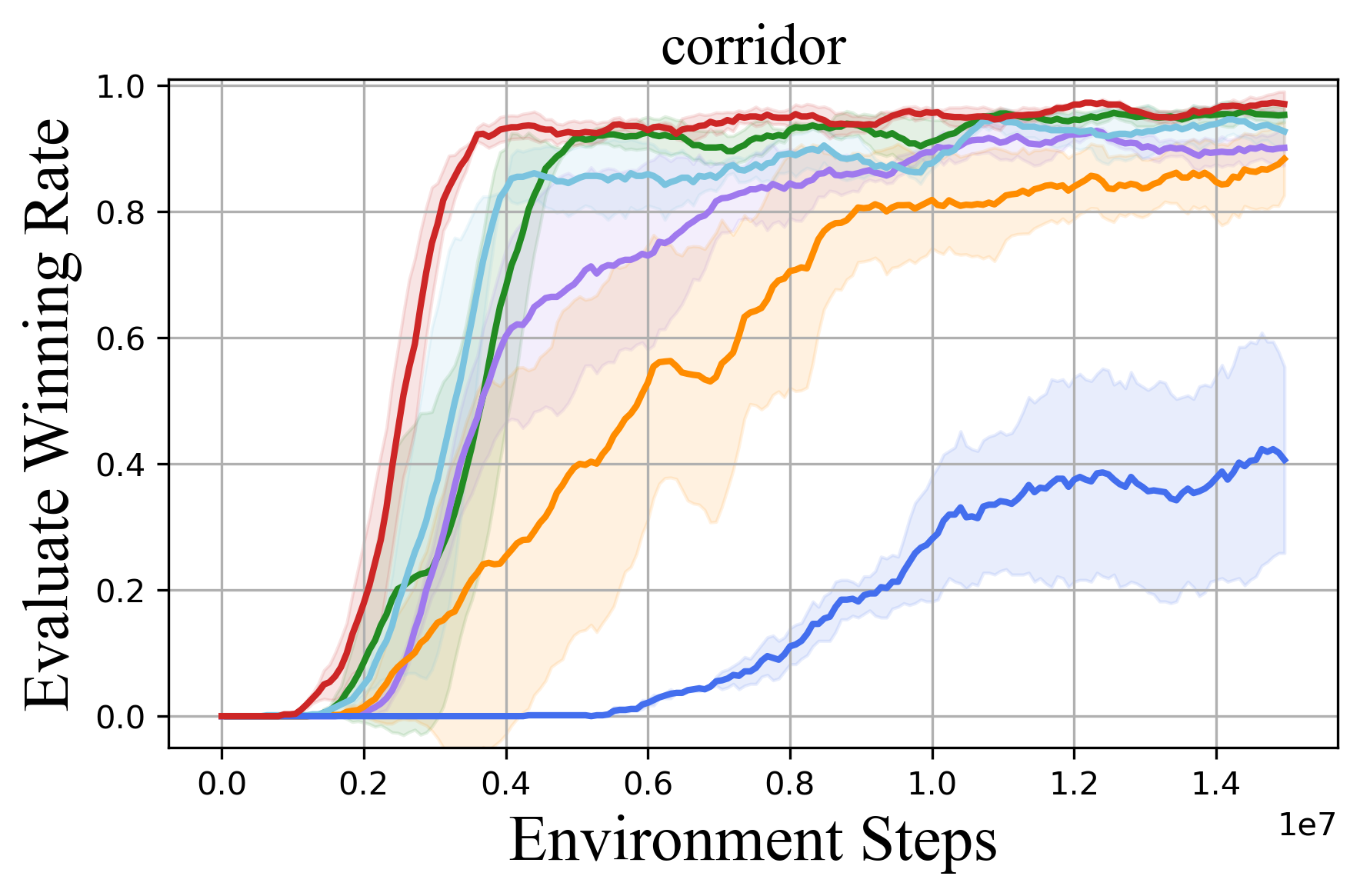}
	\end{subfigure}
	\begin{subfigure}{0.13\textwidth}
		\includegraphics[width=\linewidth]{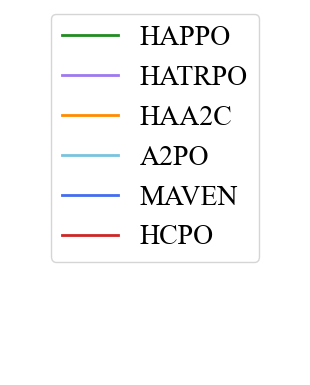}
	\end{subfigure}
	\caption{Performance comparison on SMAC. With the conductor-based joint policy enhancing learning efficiency, HCPO reliably outperforms all baselines.}\label{smacfigs}
\end{figure*}
\section{Experiments}
In this section, we evaluate HCPO on standard MARL benchmarks including SMAC~\cite{whiteson2019starcraft}, MA-MuJoCo~\cite{likaleidoscope}, and MPE~\cite{lowe2017multi}. It is compared against strong baselines: HATRPO~\cite{kuba2022trust}, HAPPO~\cite{kuba2022trust}, A2PO~\cite{wang2023order}, HAA2C~\cite{zhong2024heterogeneous}. Detailed experimental setup and results are presented below. 
\begin{figure}[htbp]
	\centering
	\begin{subfigure}{0.49\columnwidth}
		\includegraphics[width=\linewidth]{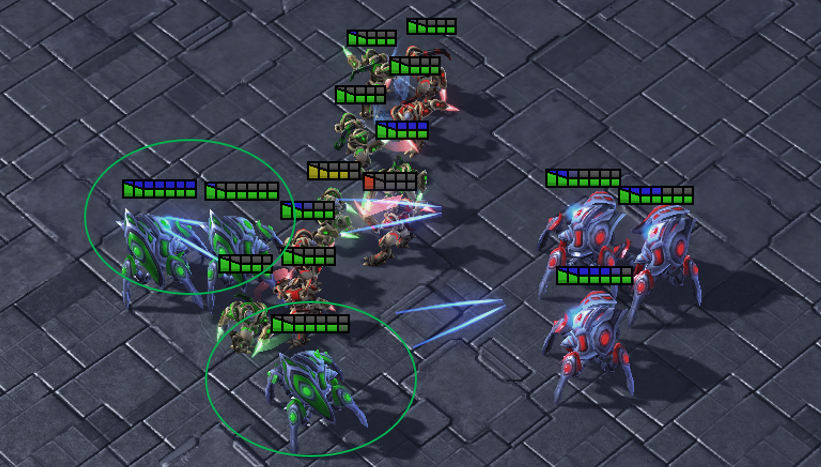}
		\caption{}
		\label{subfig4b}
	\end{subfigure}
	\begin{subfigure}{0.49\columnwidth}
		\includegraphics[width=\linewidth]{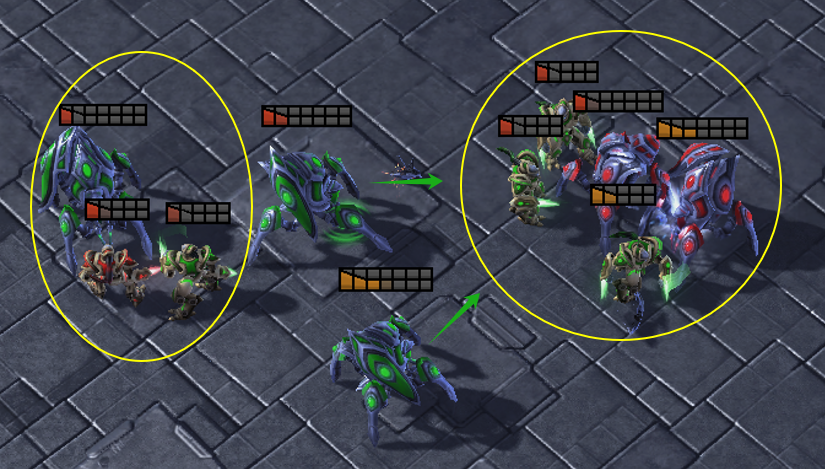}
		\caption{}
		\label{subfig4c}
	\end{subfigure}
		\caption{Effective coordination in SMAC on the 3s5z map: A visual analysis of agent strategies.}\label{SMACvisfigsmain}
\end{figure}
\subsection{Settings and Performance}\label{setandperform}

\textbf{SMAC:} We evaluate algorithms on five SMAC maps, including the widespread hierarchical algorithm MAVEN~\cite{mahajan2019maven}. Each algorithm is tested with five different random seeds to ensure the robustness and reliability. The results demonstrate that our HCPO achieves outstanding performance in all test maps, as illustrated in Figure~\ref{smacfigs}, with shadows showing standard deviation across different runs. Specifically, HCPO is the first to achieve a 90\% winning rate on all maps, enhancing exploration and learning efficiency. Furthermore, HCPO exhibits the lowest standard deviation, demonstrating its high stability. In addition, we visualize the gameplay scenarios on the \textit{3s5z} map in Figure~\ref{SMACvisfigsmain}. Figure~\ref{SMACvisfigsmain}(a) shows the early game strategy where our team uses one stalker to draw enemy fire, while two other stalkers attack from behind the zealots. Figure~\ref{SMACvisfigsmain}(b) presents the overall force division into two groups, with two stalkers providing support. These coordinated actions highlight the agents' effective collaboration, driven by the latent instructions of HCPO, which ultimately leads to victory. Detailed analysis and the evaluation method are presented in~Appendix~C.1.

\begin{figure}[htbp]
	\centering
	\begin{subfigure}{0.48\columnwidth}
		\includegraphics[width=\linewidth]{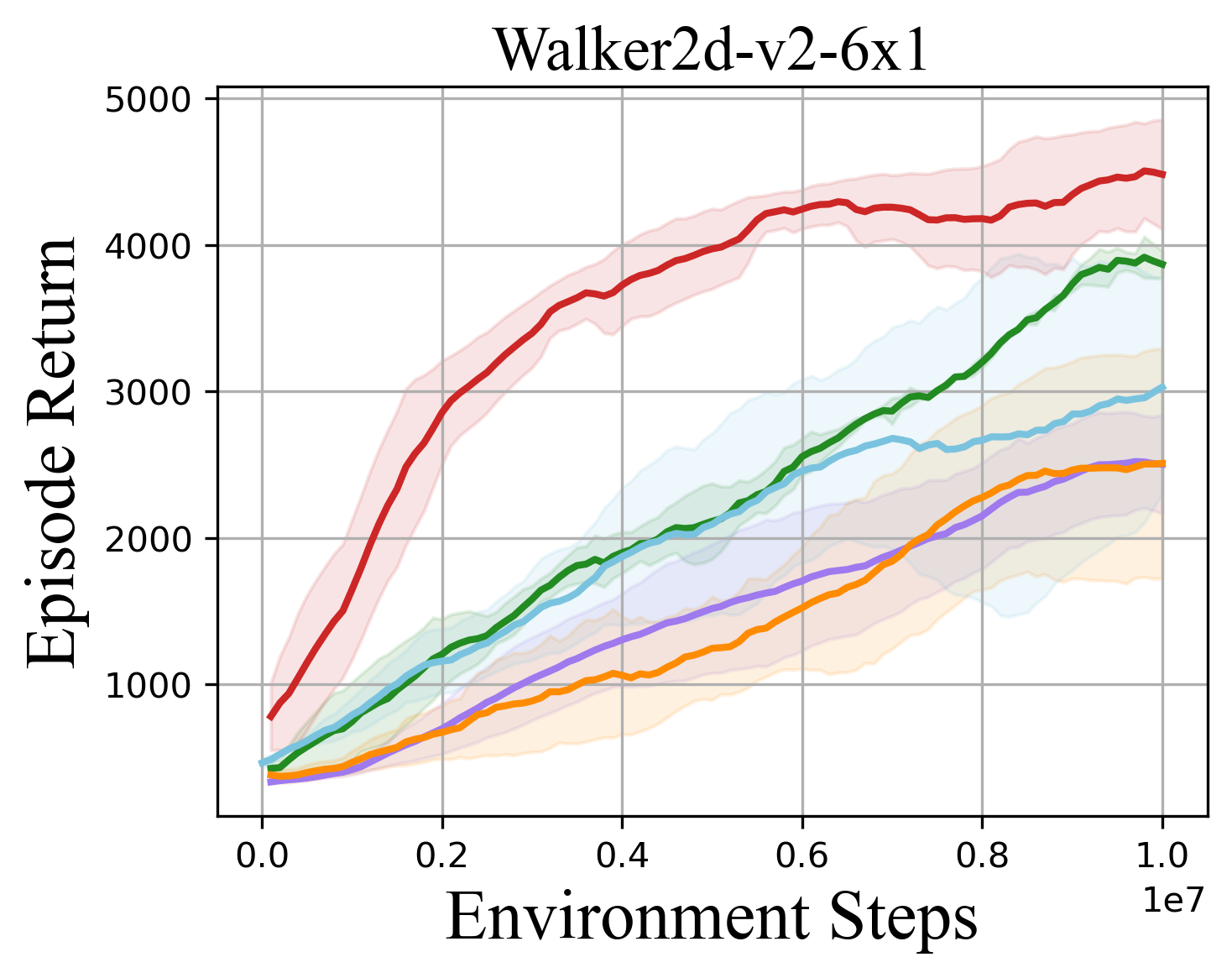}
	\end{subfigure}%
	\begin{subfigure}{0.48\columnwidth}
		\includegraphics[width=\linewidth]{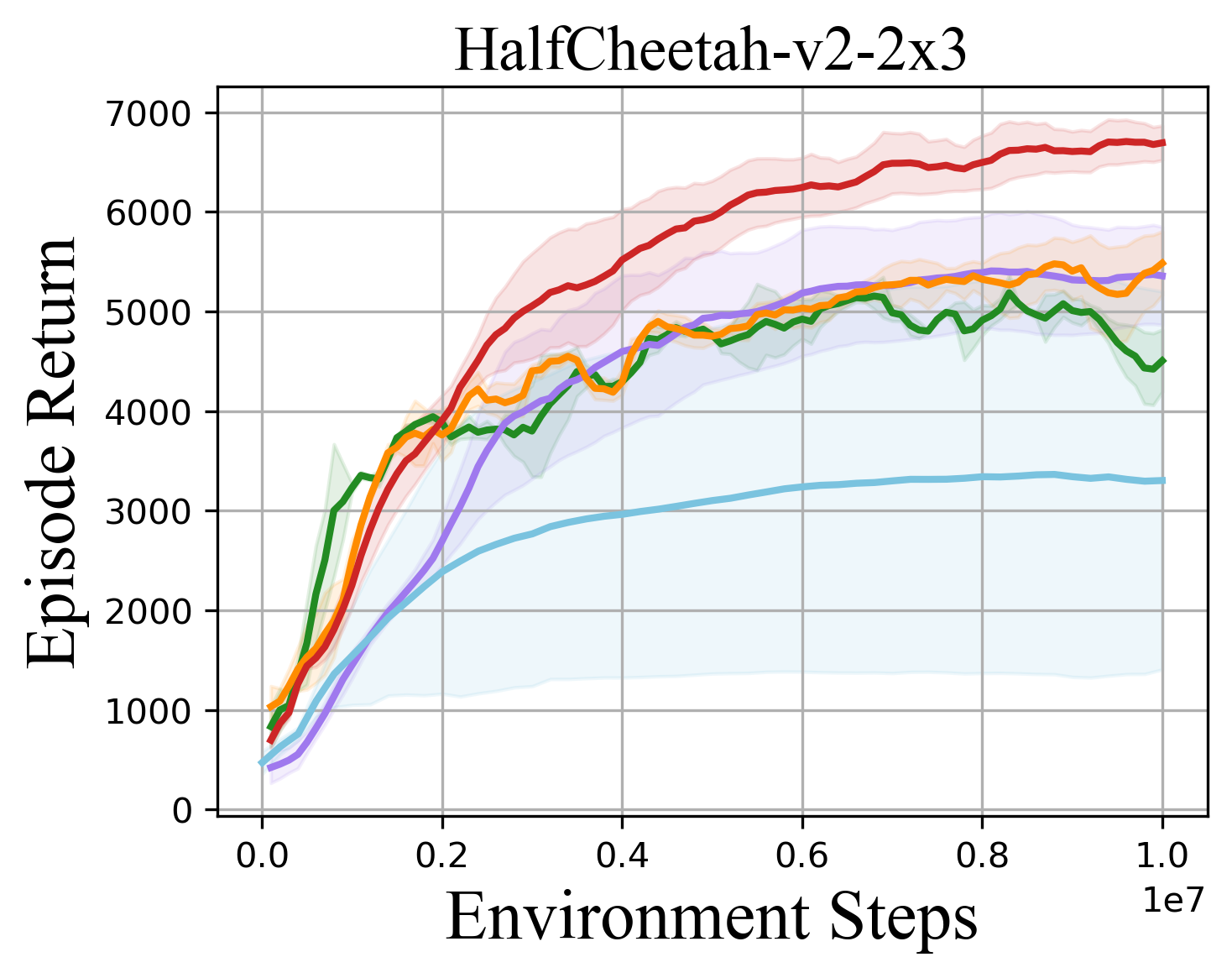}
	\end{subfigure}
	\begin{subfigure}{0.48\columnwidth}
		\includegraphics[width=\linewidth]{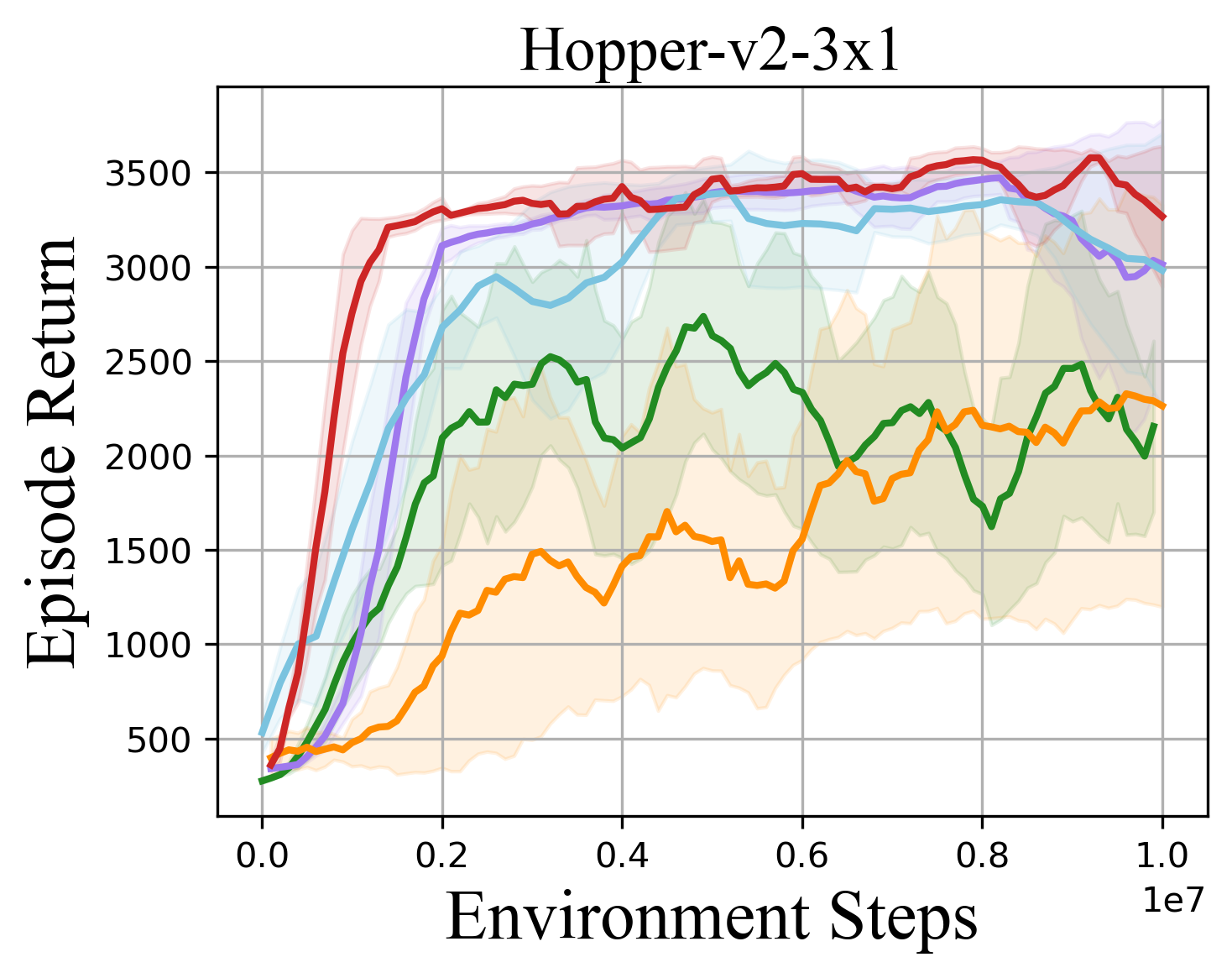}
	\end{subfigure}
	\begin{subfigure}{0.48\columnwidth}
		\includegraphics[width=\linewidth]{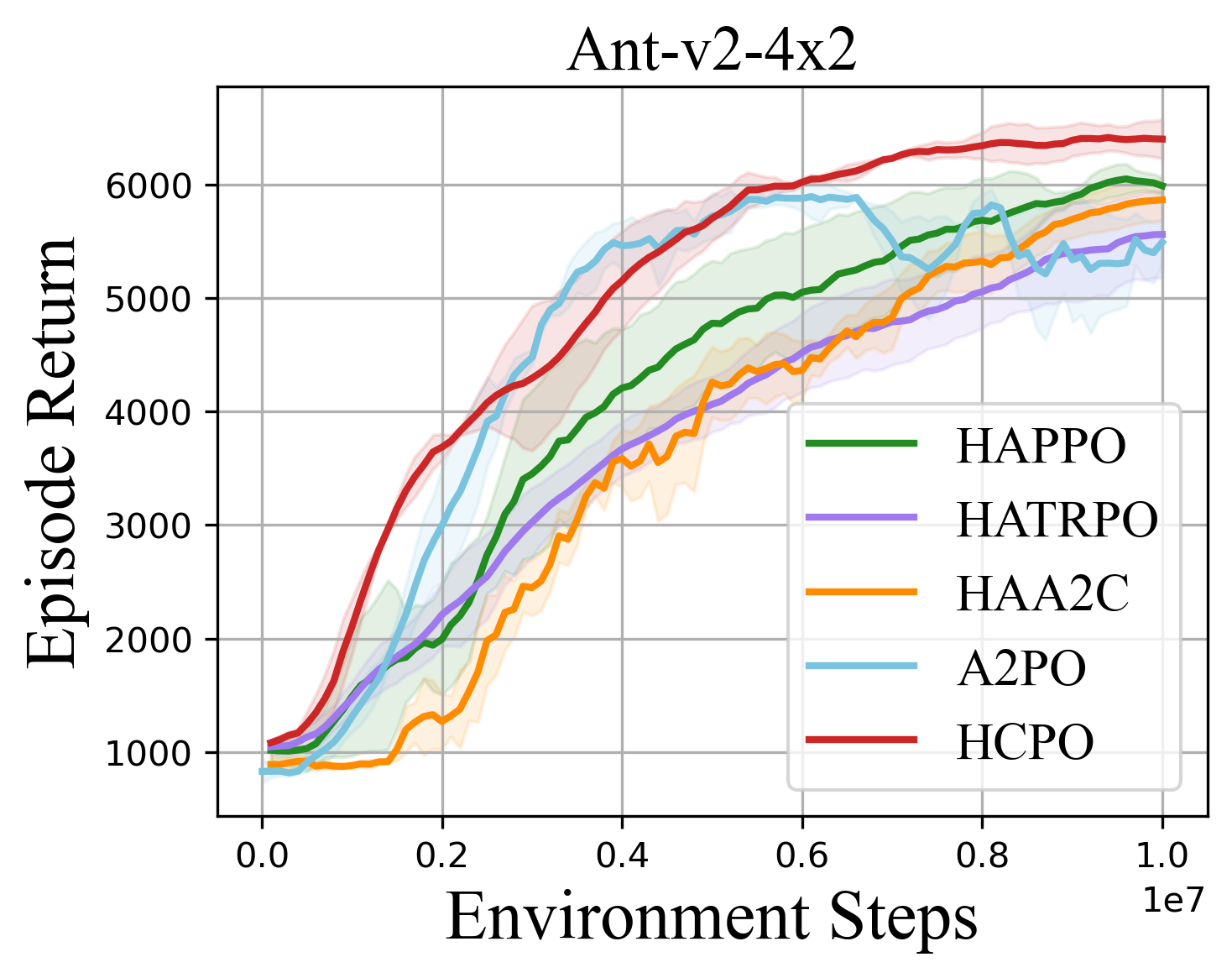}
	\end{subfigure}
	\caption{Comparative evaluation on MA-MuJoCo.}\label{mujocofigs}
\end{figure}
\begin{figure*}[t]
	\centering
	\includegraphics[width=0.72\textwidth]{"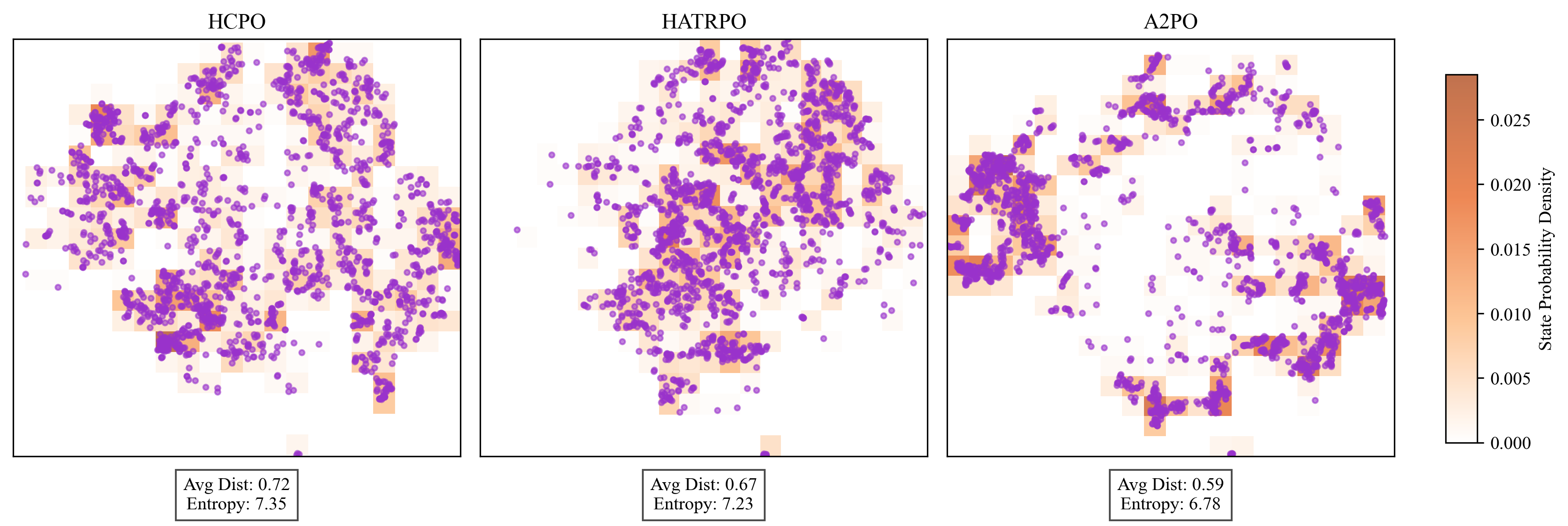"}
	\caption{Exploration comparison: t-SNE visualization and entropy analysis in \textit{Walker2d-v2-6\(\times\)1}.}\label{HCPOentropy1}
\end{figure*}
\textbf{MA-MuJoCo:} We compare HCPO with four advanced on-policy MARL algorithms using three different random seeds. As shown in Figure~\ref{mujocofigs}, HCPO not only achieves significantly higher final returns but also exhibits a lower standard deviation, indicating superior exploration and stability. Specifically, in \textit{HalfCheetah-v2-2\(\times\)3}, HCPO achieves around 23.42\% higher final returns than the next best algorithm HAA2C. In Figure~\ref{HCPOentropy1}, we employ t-SNE (t-Distributed Stochastic Neighbor Embedding) technique to project the states explored by HCPO, HATRPO, and A2PO algorithms during the early stage of training in the \textit{Walker2d-v2-6\(\times\)1} task onto a 2-D plane. Through entropy analysis and calculation of the average nearest neighbor distance, we can observe that HCPO demonstrates superior exploration. Detailed analysis is presented in~Appendix~C.2.

\textbf{MPE:} The results on MPE under three random seeds are shown in Figure~\ref{mpefigs}. HCPO exhibits rapid policy improvement in the early stage (0-2 million steps) of training, indicating that HCPO has a high cooperative efficiency and sufficient exploration. Furthermore, compared with HATRPO and A2PO, HCPO shows significant stability and~robustness. More experiments are illustrated in~Appendix C.3.

\begin{figure}[h] 
	\centering
	\begin{subfigure}{0.46\columnwidth}
		\includegraphics[width=\linewidth]{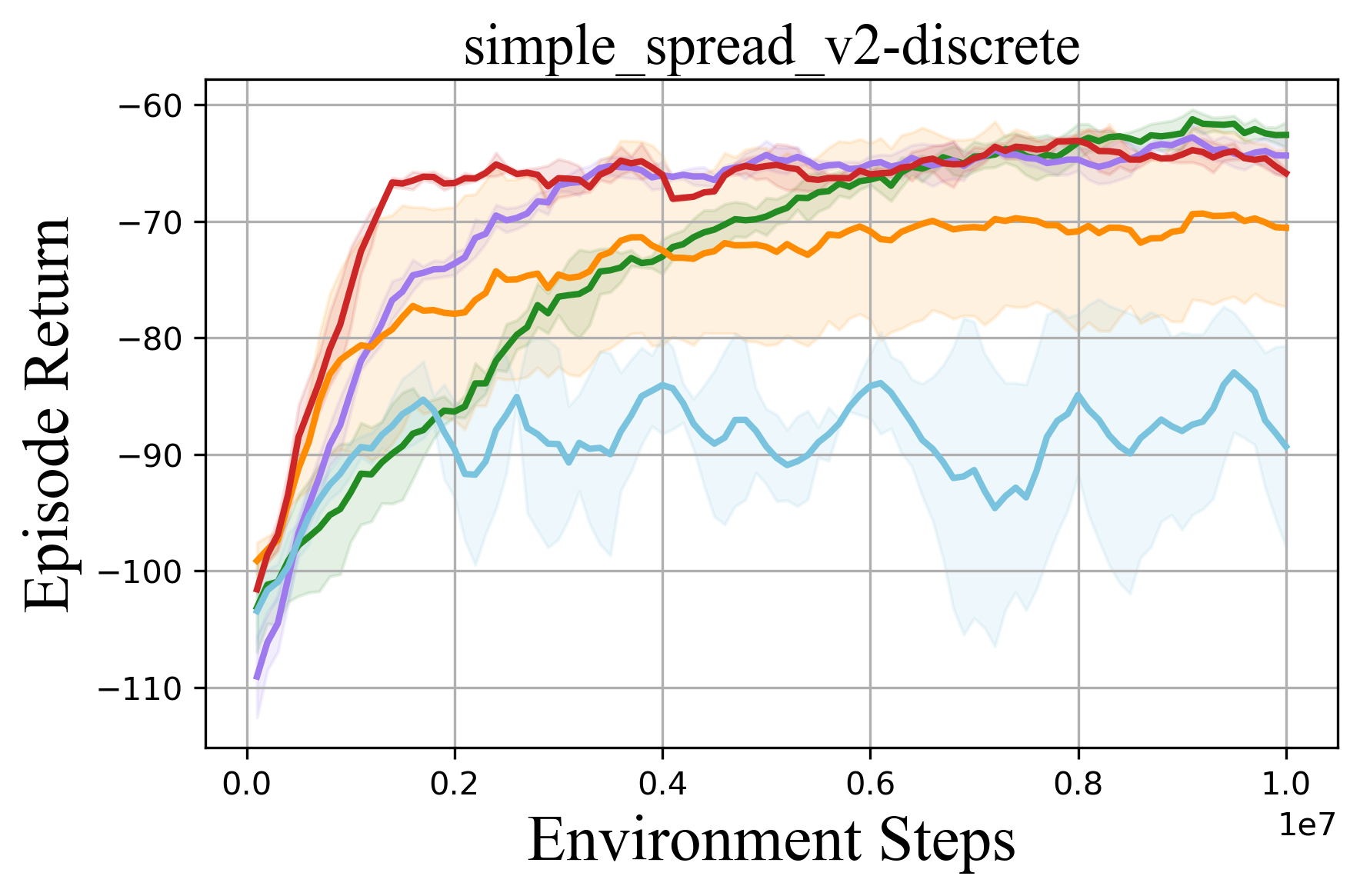}
	\end{subfigure}
	\begin{subfigure}{0.46\columnwidth}
		\includegraphics[width=\linewidth]{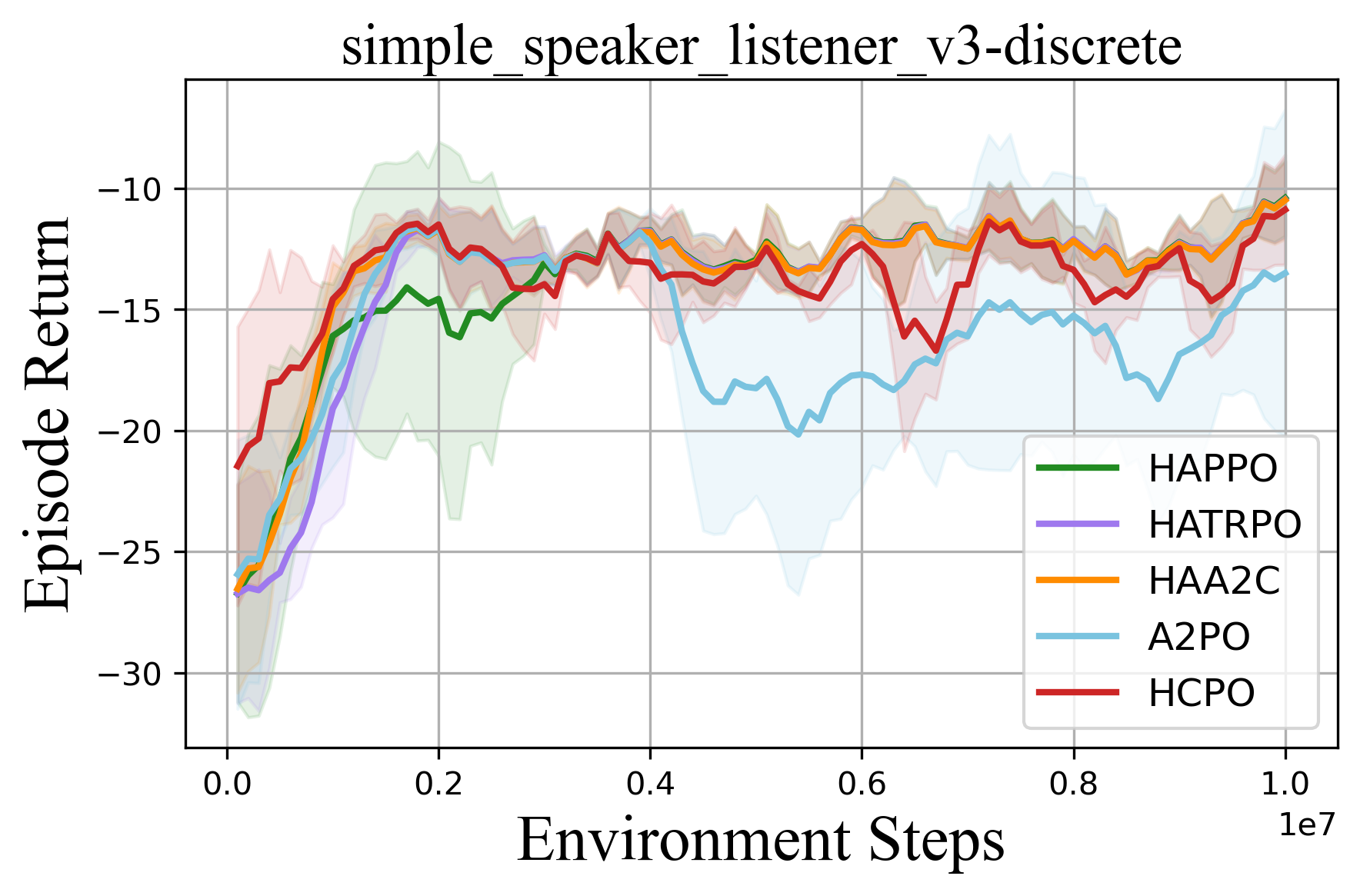}
	\end{subfigure}
	\caption{Performance comparison of HCPO and other strong MARL algorithms across different MPE tasks.}
	\label{mpefigs}
\end{figure}

\subsection{Ablation Studies}
For several key components of HCPO, we conduct ablation studies with results shown in Figure~\ref{ablfigs}. Figure~\ref{ablfigs}(a) examines the impact of the conductor and varying the number of instructions (hyperparameter \(K\)) on performance. HCPO with the conductor shows a faster increase in winning rate and a higher final rate than without any conductor, demonstrating its effectiveness in boosting cooperation efficiency. The performance is also influenced by \( K \). It is important to balance performance and resource consumption when selecting \( K \). In Figure~\ref{ablfigs}(b), we examine the hyperparameter \(\delta_1\), which represents the KL-divergence constraint during the conductor's policy update. In Figure~\ref{ablfigs}(c), we evaluate HCPO under four conductor configurations: a centralized conductor, a random conductor (non-learning baseline), no conductor and local conductors (the core of our HCPO algorithm). Figure~\ref{ablfigs}(d) presents the final episode returns in a boxplot format. The results show that HCPO with local conductors achieves a median return comparable to that of a centralized conductor, while substantially outperforming the variant without any conductor. Furthermore, replacing the learned instruction preference distribution with a non-learning conductor that outputs uniformly random instructions leads to inferior performance. These findings validate the effectiveness of our proposed update mechanisms. We present detailed analysis in~Appendix~C.4.
\begin{figure}[h]
	\centering
	\begin{subfigure}{0.46\columnwidth}
		\includegraphics[width=\linewidth]{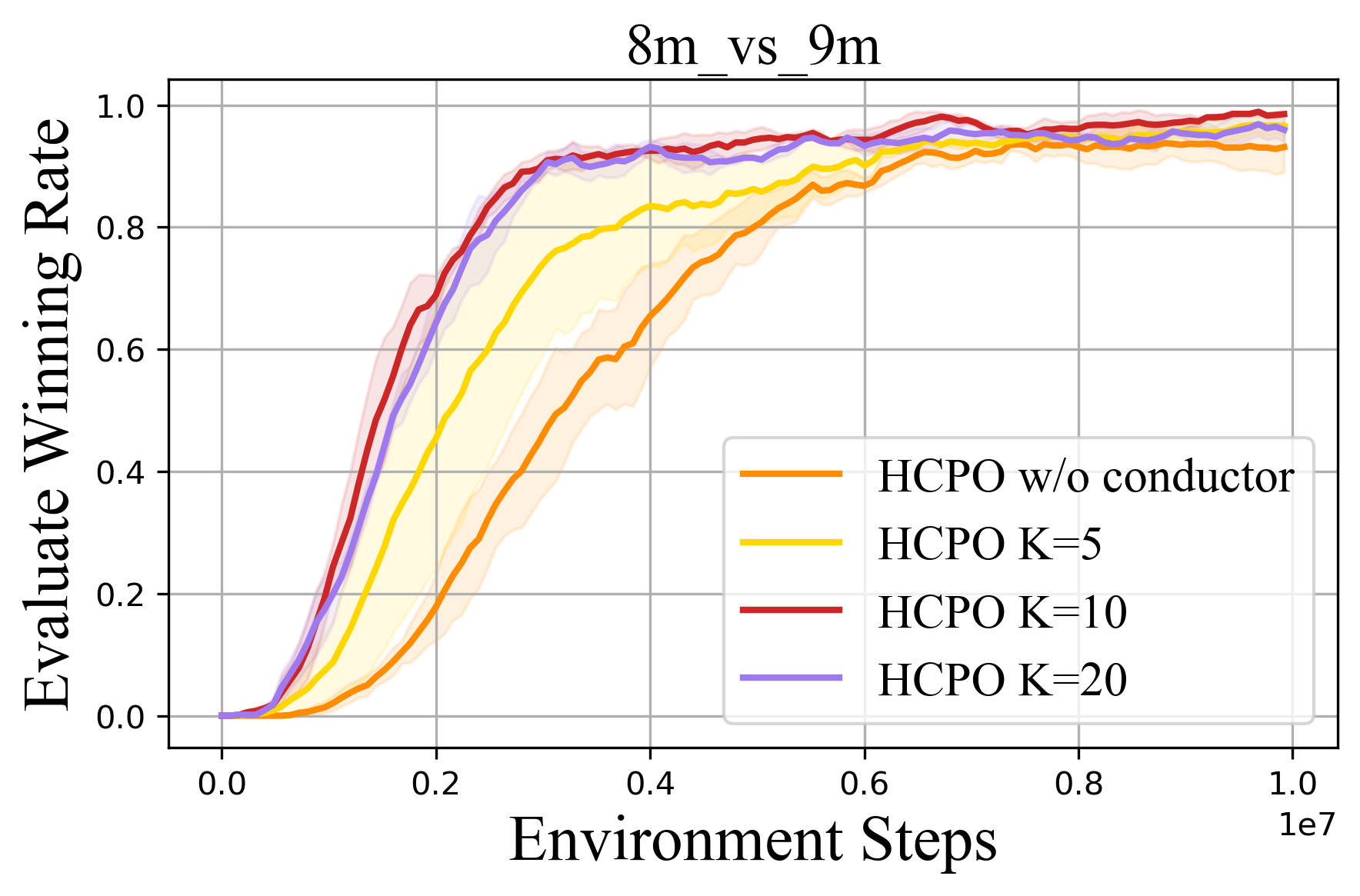}
		\caption{}
		\label{subfig6a}
	\end{subfigure}%
	\begin{subfigure}{0.46\columnwidth}
		\includegraphics[width=\linewidth]{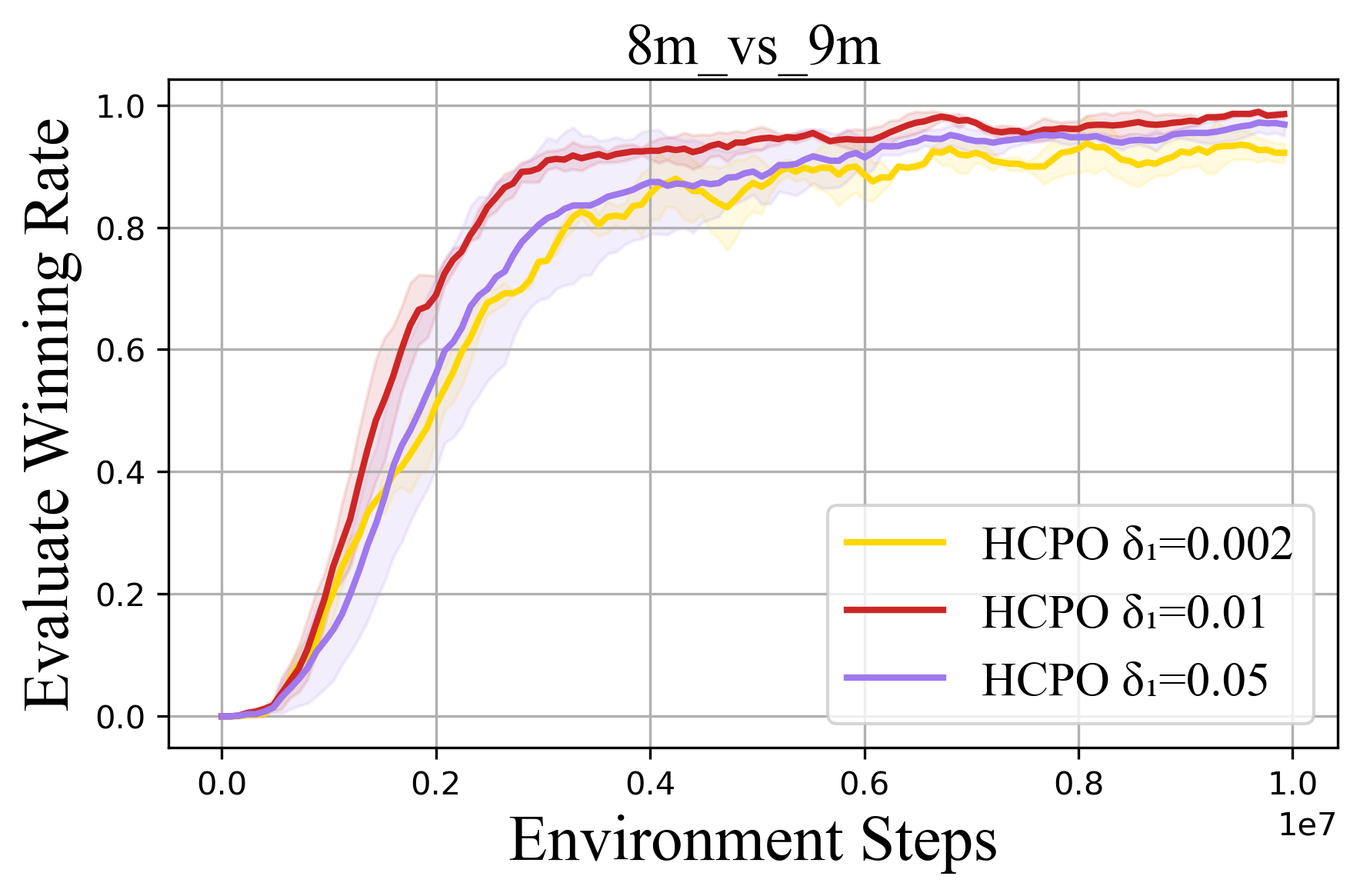}
		\caption{}
		\label{subfig6b}
	\end{subfigure}
	\begin{subfigure}{0.46\columnwidth}
		\includegraphics[width=\linewidth]{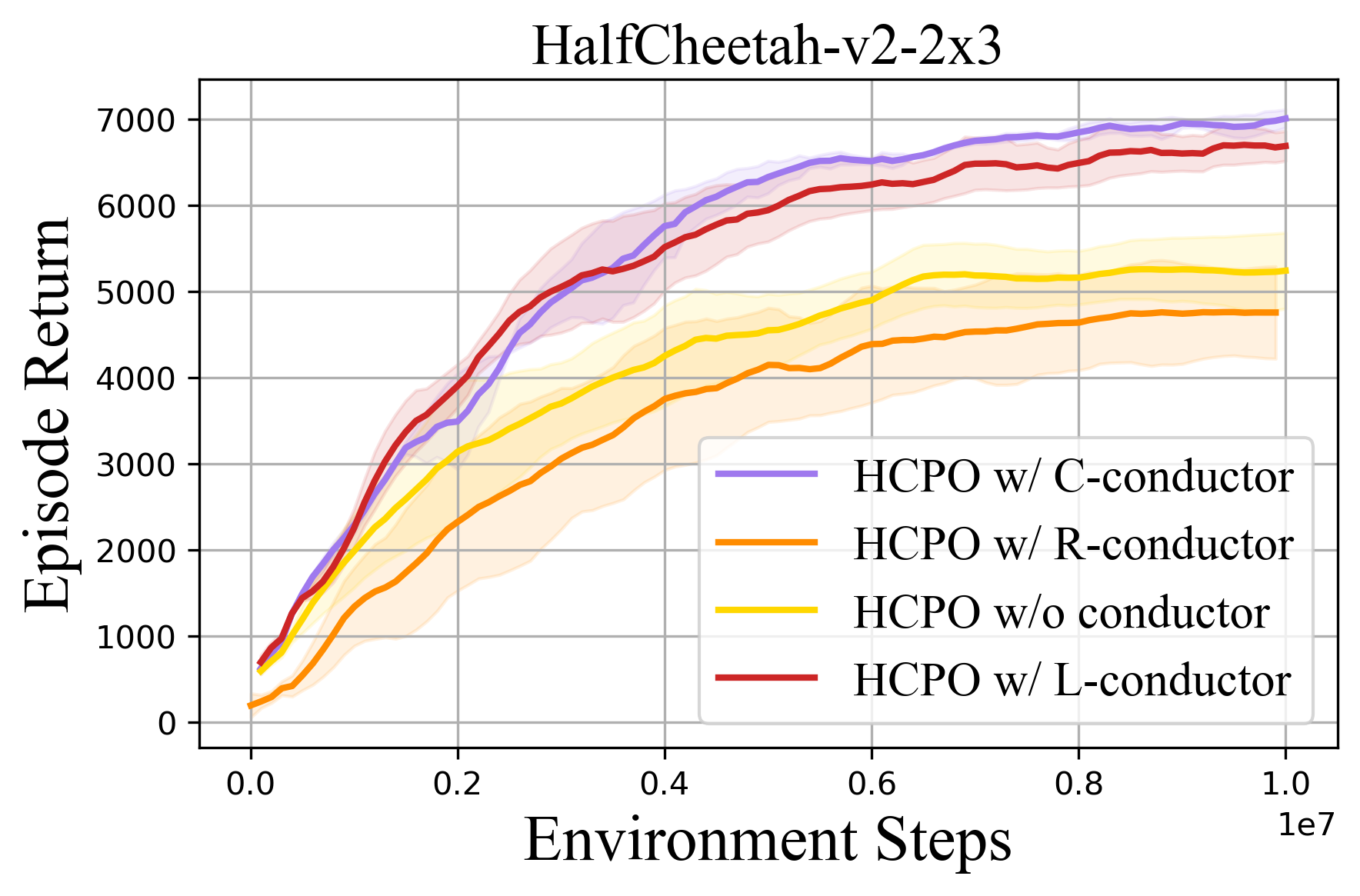}
		\caption{}
		\label{subfig6c}
	\end{subfigure}
	\begin{subfigure}{0.46\columnwidth}
		\includegraphics[width=\linewidth]{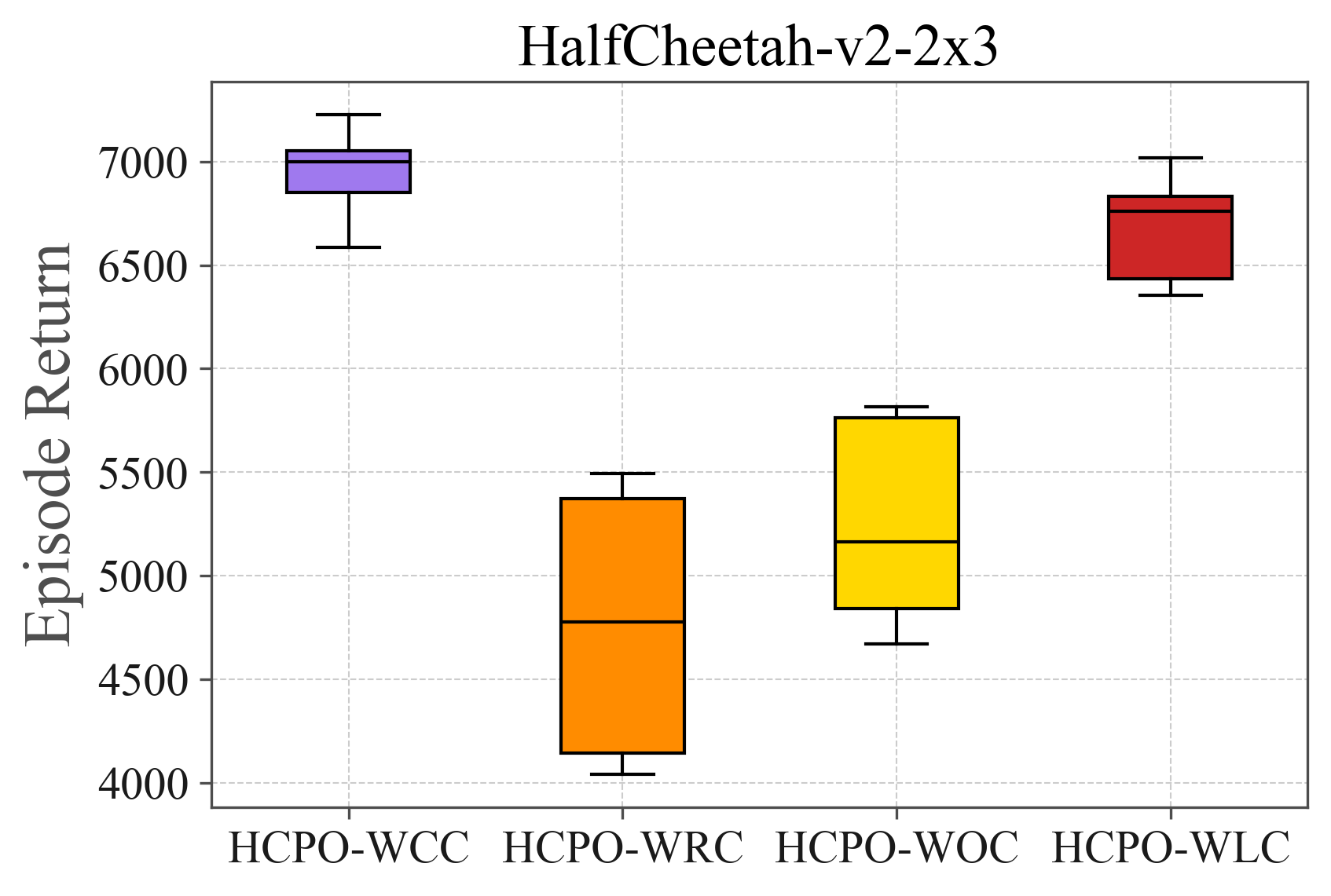}
		\caption{}
		\label{subfig6d}
	\end{subfigure}
	\caption{Ablation studies.}\label{ablfigs}
\end{figure}
\section{Conclusion and Future Work}\label{conclusionsec}
In this paper, we introduce HCPO, a hierarchical MARL algorithm designed to enhance the expressive capacity of joint policies and improve exploration. By leveraging specific advantage functions, we propose a two-level policy update mechanism with monotonic improvement guarantees without a monotonicity hypothesis in QMIX. Based on the cross-entropy method, each agent is equipped with a local conductor. This not only improves cooperation but also avoids the limitations imposed by communication constraints. Through comprehensive experiments on SMAC, MA-MuJoCo, and MPE, HCPO demonstrates superior performance compared to competitive algorithms. The main limitation of this work is that we design an on-policy algorithm for our hierarchical framework. In the future, we plan to integrate the conductor-based mechanism of HCPO into off-policy algorithms to improve sample efficiency.

\section{Acknowledgments}

This work was supported by National Natural Science Foundation of China under Grants 62233005, U2441245, 62573198, 62173142, and in part by the Shanghai Institute for Mathematics and Interdisciplinary Sciences (SIMIS) under Grant SIMIS-ID-2025-SP.

\bibliography{aaai2026}

	\onecolumn

\appendix    
	\renewcommand{\thesection}{\Alph{section}}   
	\setcounter{section}{0}                      
	
\section*{\centering \huge \bfseries Appendix}
\vspace{0.5cm}
	\section{A. Additional Details for HCPO Framework}\label{apxdetailpro}
		\subsection{A.1 Useful Lemmas}\label{uselema}
	\renewcommand{\thelem}{\ref{prop1}}
	\begin{lem}
		\textcolor{black}{~For any instruction \(M^j, j\in \left\{ 1,2,...,K \right\} \) chosen by the conductor, the conditional \(Q\)-function for agents \textcolor{black}{\(i_{1:l}\)} satisfies:}
		\begin{equation}
			Q_{\boldsymbol{\pi }_{\textup{mar}}}^{i_{1:l}}(\boldsymbol{a}^{i_{1:l}}|s,M^j)=\mathbb{E} _{\boldsymbol{a}^{-i_{1:l}}\sim \boldsymbol{\pi }^{-i_{1:l}}\left( \cdot |s,M^j \right)}\left[ Q_{\boldsymbol{\pi }_{\textup{mar}}}^{}\left( s,\boldsymbol{a} \right) \right], 
		\end{equation}
		where \(\boldsymbol{a}=\left( \boldsymbol{a}^{i_{1:l}},\boldsymbol{a}^{-i_{1:l}} \right)\).   
	\end{lem}
	\begin{proof}
		\begin{align}
			Q_{\boldsymbol{\pi }_{\textup{mar}}}^{i_{1:l}}(\boldsymbol{a}^{i_{1:l}}|s,M^j)&=\mathbb{E} _{\boldsymbol{a}_{0}^{-i_{1:l}}\sim\boldsymbol{\pi }_{0}^{-i_{1:l}}\left( \cdot |s,M^j \right)} \Bigl[ \mathbb{E} _{\mathfrak{s}_{\rho_{\boldsymbol{\pi }_{\textup{mar}}}}^{1:\infty}, \mathfrak{M}_{w}^{1:\infty},\mathfrak{a}_{\boldsymbol{\pi }}^{1:\infty}}  [\sum_{t=0}^{\infty}{\gamma ^t}r_t|s_0=s,M_0=M^j, \boldsymbol{a}_{0}^{i_{1:l}}=\boldsymbol{a}_{}^{i_{1:l}} ,\boldsymbol{a}_{0}^{-i_{1:l}}] \Bigr] \notag \\
			&=\mathbb{E} _{\boldsymbol{a}_{0}^{-i_{1:l}}\sim\boldsymbol{\pi }_{0}^{-i_{1:l}}\left( \cdot |s,M^j \right)}\Bigl[ \mathbb{E} _{\mathfrak{s}_{\rho_{\boldsymbol{\pi }_{\textup{mar}}}}^{1:\infty}, \mathfrak{M}_{w}^{1:\infty},\mathfrak{a}_{\boldsymbol{\pi }}^{1:\infty}}   \left.[\sum_{t=0}^{\infty}{\gamma ^t}r_t|s_0=s, \boldsymbol{a}_{0}^{i_{1:l}}=\boldsymbol{a}_{}^{i_{1:l}} ,\boldsymbol{a}_{0}^{-i_{1:l}}\right.]  \Bigr] \notag \\
			&=\mathbb{E} _{\boldsymbol{a}^{-i_{1:l}}\sim \boldsymbol{\pi }^{-i_{1:l}}\left( \cdot |s,M^j \right)}\left[ Q_{\boldsymbol{\pi }_{\textup{mar}}}^{}\left( s,\boldsymbol{a} \right) \right].
		\end{align}
	\end{proof}

	\renewcommand{\thelem}{\ref{lemma1}}
	\begin{lem}
		(Conditional Advantage Decomposition) Consider a cooperative Markov game with a joint policy \(\boldsymbol{\pi }_{\textup{mar}}\). For any state \(s\), any instruction \(M^j\), and any subset of agents \(i_{1:n}=\left\{ i_1, i_2, \cdots , i_n \right\}\subseteq \mathcal{N}\), the following equation holds for all states \(s\), joint actions \(\boldsymbol{a}^{i_{1:n}}\), and \(M^j\sim w\):
		\begin{equation}\label{eq14apx}
			A_{\boldsymbol{\pi }_{\textup{mar}}}^{i_{1:n}}\left( \boldsymbol{a}^{i_{1:n}}|s,M^j \right) =\sum_{l=1}^n{A_{\boldsymbol{\pi }_{\textup{mar}}}^{i_l}\left( \boldsymbol{a}^{i_{1:l-1}},a^{i_l}|s,M^j \right)}. 
		\end{equation}
		\begin{proof}
			According to the definition of the advantage function as given in \eqref{defn_adv2}, we proceed with the right-hand side as follows:
			\begin{align}
				\sum_{l=1}^n{A_{\boldsymbol{\pi }_{\textup{mar}}}^{i_l}\left( \boldsymbol{a}^{i_{1:l-1}},a^{i_l}|s,M^j \right)} &=\sum_{l=1}^n{\left[ Q_{\boldsymbol{\pi }_{\textup{mar}}}^{i_{1:l}}(\boldsymbol{a}^{i_{1:l}}|s,M^j)-Q_{\boldsymbol{\pi }_{\textup{mar}}}^{i_{1:l-1}}(\boldsymbol{a}^{i_{1:l-1}}|s,M^j) \right]} \notag  \\&=Q_{\boldsymbol{\pi }_{\textup{mar}}}^{i_{1:n}}(\boldsymbol{a}^{i_{1:n}}|s,M^j)-Q_{\boldsymbol{\pi }_{\textup{mar}}}^{}(M^j|s)  \notag \\ &=A_{\boldsymbol{\pi }_{\textup{mar}}}^{i_{1:n}}(\boldsymbol{a}^{i_{1:n}}|s,M^j).
			\end{align}
		\end{proof}
	\end{lem}
	
	\renewcommand{\theprop}{\ref{prop22}}
	\begin{prop}
		As defined in \eqref{maxJ}, the relationship between the expected return of the new policy \(\bar{\boldsymbol{\pi}}_{\textup{mar}}\) and the old policy \(\boldsymbol{\pi}_{\textup{mar}}\) is expressed as:
		\begin{equation}~\label{eq15apx}
			J(\bar{\boldsymbol{\pi}}_{\textup{mar}})=J(\boldsymbol{\pi }_{\textup{mar}})+\mathbb{E} _{\tau \sim \bar{\boldsymbol{\pi}}_{\textup{mar}}}\left[ \sum_{t=0}^{\infty}{\gamma ^t}A_{\boldsymbol{\pi }_{\textup{mar}}}(s_t,\boldsymbol{a}_t) \right]\textcolor{black}{,}
		\end{equation}
		\textcolor{black}{where \(\tau:=(s_0, M_0, \boldsymbol{a}_0, s_1, M_1, \boldsymbol{a}_1, \cdots)\).}
	\end{prop}
	\begin{proof}
		We begin by expanding the expectation term involving the advantage function:
		\begin{align}\label{eq19}
			\mathbb{E} _{\tau \sim \bar{\boldsymbol{\pi}}_{\textup{mar}}}\left[ \sum_{t=0}^{\infty}{\gamma ^t}A_{\boldsymbol{\pi }_{\textup{mar}}}(s_t,\boldsymbol{a}_t) \right] &=\mathbb{E} _{\tau \sim \bar{\boldsymbol{\pi}}_{\textup{mar}}}\left[ \sum_{t=0}^{\infty}{\gamma ^t}\left[ Q_{\boldsymbol{\pi }_{\textup{mar}}}(s_t,\boldsymbol{a}_t)-V_{\boldsymbol{\pi }_{\textup{mar}}}(s_t) \right] \right] \notag \\&=\mathbb{E} _{\tau \sim \bar{\boldsymbol{\pi}}_{\textup{mar}}}\left[ \sum_{t=0}^{\infty}{\gamma ^t}\left[ r_t+\gamma V_{\boldsymbol{\pi }_{\textup{mar}}}(s_{t+1})-V_{\boldsymbol{\pi }_{\textup{mar}}}(s_t) \right] \right] \notag \\&=\mathbb{E} _{\tau \sim \bar{\boldsymbol{\pi}}_{\textup{mar}}}\left[ -V_{\boldsymbol{\pi }_{\textup{mar}}}(s_0)+\sum_{t=0}^{\infty}{\gamma ^t}r_t \right].
		\end{align}
		By the definition of~\eqref{maxJ} and~\eqref{Vtot}, we can obtain that:
		\begin{equation}
			J(\boldsymbol{\pi }_{\textup{mar}})=\mathbb{E} _{s_0,M_0,\boldsymbol{a}_0,s_1,M_1,\boldsymbol{a}_1,\cdots}\left( \sum_{t=0}^{\infty}{\gamma ^tr_t} \right),
		\end{equation}
		\begin{equation}
			V_{\boldsymbol{\pi }_{\textup{mar}}}(s_0)=\mathbb{E} _{M_0,\boldsymbol{a}_0,s_1,M_1,\boldsymbol{a}_1,\cdots}[\sum_{t=0}^{\infty}{\gamma ^t}r_t|s_0=s],
		\end{equation}
		\begin{equation}
			J(\boldsymbol{\pi }_{\textup{mar}})=\mathbb{E} _{s_0}\left[ V_{\boldsymbol{\pi }_{\textup{mar}}}(s_0) \right].
		\end{equation}
		Therefore,~\eqref{eq19} can be further simplified into:
		\begin{equation}
			\begin{split}
				\mathbb{E} _{\tau \sim \bar{\boldsymbol{\pi}}_{\textup{mar}}}\left[ \sum_{t=0}^{\infty}{\gamma ^t}A_{\boldsymbol{\pi }_{\textup{mar}}}(s_t,\boldsymbol{a}_t) \right] &=\mathbb{E} _{\tau \sim \bar{\boldsymbol{\pi}}_{\textup{mar}}}\left[ -V_{\boldsymbol{\pi }_{\textup{mar}}}(s_0)+\sum_{t=0}^{\infty}{\gamma ^t}r_t \right] \notag \\
				&=-J(\boldsymbol{\pi }_{\textup{mar}})+J(\bar{\boldsymbol{\pi}}_{\textup{mar}}).
			\end{split}
		\end{equation}
	\end{proof}

	\subsection{A.2 Proof of Theorem 1}\label{prothm}
\renewcommand{\thethm}{\ref{thm11}}
	\begin{thm}
		Under the proposed conductor-based framework, a significant policy improvement inequality holds for the joint policy \(\boldsymbol{\pi }_{\textup{mar}}\):
			\begin{align}\label{eq27apx}
				J(\bar{\boldsymbol{\pi}}_{\textup{mar}}) &\geqslant J(\boldsymbol{\pi }_{\textup{mar}}) + \mathbb{E} _{\mathrm{s}\thicksim \rho _{ \boldsymbol{\pi }_{\textup{mar}}}}\left.\Big[ 
				\mathbb{E} _{M\sim \bar{w}\left( \cdot |s \right)}A_{ \boldsymbol{\pi }_{\textup{mar}}}(M|s) \right. - C\textcolor{black}{ \mathrm{D}_{\mathrm{KL}}^{\max}\left( w ,\overline{w} \right) \notag }  \\&
				\left. + \mathbb{E} _{M\sim \bar{w}\left( \cdot |s \right) ,\boldsymbol{a}\thicksim \bar{\boldsymbol{\pi}}\left( \cdot |s,M \right)}A_{ \boldsymbol{\pi }_{\textup{mar}}}(\boldsymbol{a}|s,M) - \textcolor{black}{\max C\sum_{j=1}^K{w\left( M^j|s \right) \mathrm{D}_{\mathrm{KL}}\left( \boldsymbol{\pi }(\cdot |s,M^j),\overline{\boldsymbol{\pi }}(\cdot |s,M^j) \right)}} \right.\Big],
			\end{align}
		where \(C = \frac{4\gamma \max_{s,a} |A_{\boldsymbol{\pi }_{\textup{mar}}}(s,\boldsymbol{a})|}{(1-\gamma)^2}\),~\(\mathrm{D}_{\mathrm{KL}}^{\max} \left( w_k, \bar{w} \right)=\max_s\mathrm{D}_{\mathrm{KL}}\left( w_k\left( \cdot |s \right), \bar{w}\left( \cdot |s \right) \right)\).
	\end{thm}
	\begin{proof}
		
		First, from the alternative objective function~\eqref{alterfun}, we can obtain that:
		\begin{align}
			L_{\boldsymbol{\pi }_{\textup{mar}}}(\bar{\boldsymbol{\pi}}_{\textup{mar}})
			\notag &=J(\boldsymbol{\pi }_{\textup{mar}})+\mathbb{E} _{s\thicksim \rho _{ \boldsymbol{\pi }_{\textup{mar}}},M\sim\bar{w}\left( \cdot |s \right) ,\boldsymbol{a}\thicksim \bar{\boldsymbol{\pi}}\left( \cdot |s,M \right)} \left[ A_{\boldsymbol{\pi }_{\textup{mar}}}(s,\boldsymbol{a}) \right] \notag \\ &=J(\boldsymbol{\pi }_{\textup{mar}}) +\mathbb{E} _{s\thicksim \rho _{\boldsymbol{\pi }_{\textup{mar}}},M\sim\bar{w}\left( \cdot |s \right) ,\boldsymbol{a}\thicksim \bar{\boldsymbol{\pi}}\left( \cdot |s,M \right)}\left.[ A_{ \boldsymbol{\pi }_{\textup{mar}}}(M|s) +A_{ \boldsymbol{\pi }_{\textup{mar}}}(\boldsymbol{a}|s,M) \right] \notag  \\&=J(\boldsymbol{\pi }_{\textup{mar}})+\mathbb{E} _{s\thicksim \rho _{\boldsymbol{\pi }_{\textup{mar}}}}\left.[ \mathbb{E} _{M\sim\bar{w}\left( \cdot |s \right)}A_{\boldsymbol{\pi }_{\textup{mar}}}(M|s)  +\mathbb{E} _{M\sim\bar{w}\left( \cdot |s \right) ,\boldsymbol{a}\thicksim \bar{\boldsymbol{\pi}}\left( \cdot |s,M \right)}A_{\boldsymbol{\pi }_{\textup{mar}}}(\boldsymbol{a}|s,M) \right].
		\end{align}
		Therefore, combining Theorem 1 in~\citet{kuba2022trust}, we can derive the following inequality:
		\begin{align}\label{eq25}
			J(\bar{\boldsymbol{\pi}}_{\textup{mar}}) &\geqslant L_{\boldsymbol{\pi }_{\textup{mar}}}(\bar{\boldsymbol{\pi}}_{\textup{mar}})-C\mathrm{D}_{\mathrm{KL}}^{\max}(\boldsymbol{\pi }_{\textup{mar}},\bar{\boldsymbol{\pi}}_{\textup{mar}}) \notag \\&=J(\boldsymbol{\pi }_{\textup{mar}})+\mathbb{E} _{\mathrm{s}\thicksim \rho _{\boldsymbol{\pi }_{\textup{mar}}}}\Bigl[ \mathbb{E} _{M\sim\bar{w}\left( \cdot |s \right)}A_{ \boldsymbol{\pi }_{\textup{mar}}}(M|s)  +\mathbb{E} _{M\sim\bar{w}\left( \cdot |s \right) ,\boldsymbol{a}\thicksim \bar{\boldsymbol{\pi}}\left( \cdot |s,M \right)}A_{\boldsymbol{\pi }_{\textup{mar}}}(\boldsymbol{a}|s,M) \Bigr]   -C\mathrm{D}_{\mathrm{KL}}^{\max}(\boldsymbol{\pi }_{\textup{mar}},\bar{\boldsymbol{\pi}}_{\textup{mar}}),
		\end{align}
		where \(C = \frac{4\gamma \max_{s,a} |A_{\boldsymbol{\pi }_{\textup{mar}}}(s,\boldsymbol{a})|}{(1-\gamma)^2}\)~\citet{kuba2022trust}. Next, we simplify \(\mathrm{D}_{\mathrm{KL}}^{\max}(\boldsymbol{\pi }_{\textup{mar}},\bar{\boldsymbol{\pi}}_{\textup{mar}})\) by expressing it in terms of the instruction preference distribution \(w\left( \cdot |s \right) \) and the instruction-conditional joint policy \(\boldsymbol{\pi }(\boldsymbol{a}|s,M)\):
		\begin{equation}\label{eq26}
			\mathrm{D}_{\mathrm{KL}}\left( \boldsymbol{\pi }_{\textup{mar}}(\cdot |s),\bar{\boldsymbol{\pi}}_{\textup{mar}}(\cdot |s) \right)\textcolor{black}{\leqslant} \mathrm{D}_{\mathrm{KL}}\left( w\left(\cdot|s \right) ,\overline{w}\left(\cdot|s \right) \right)+\sum_{j=1}^K{w\left( M^j|s \right) \mathrm{D}_{\mathrm{KL}}\left( \boldsymbol{\pi }(\cdot |s,M^j),\overline{\boldsymbol{\pi }}(\cdot |s,M^j) \right), } 
		\end{equation}
		where the detail derivation of~\eqref{eq26} can be found in~section A.3. \textcolor{black}{By taking maximum over state \(s\),} we have:
		\begin{align}
			\textcolor{black}{\mathrm{D}_{\mathrm{KL}}^{\max}\left( \boldsymbol{\pi }_{\textup{mar}},\bar{\boldsymbol{\pi}}_{\textup{mar}} \right)} \textcolor{black}{\leqslant \mathrm{D}_{\mathrm{KL}}^{\max}\left( w ,\overline{w} \right) +\max \sum_{j=1}^K{w\left( M^j|s \right) \mathrm{D}_{\mathrm{KL}}\left( \boldsymbol{\pi }(\cdot |s,M^j),\overline{\boldsymbol{\pi }}(\cdot |s,M^j) \right)} }.
		\end{align}
		In summary,~\eqref{eq25} can be further decomposed into:
		\begin{equation}
			\begin{split}
				J(\bar{\boldsymbol{\pi}}_{\textup{mar}}) &\geqslant J(\boldsymbol{\pi }_{\textup{mar}}) + \mathbb{E} _{\mathrm{s}\thicksim \rho _{ \boldsymbol{\pi }_{\textup{mar}}}}\left.\Big[ 
				\mathbb{E} _{M\sim \bar{w}\left( \cdot |s \right)}A_{ \boldsymbol{\pi }_{\textup{mar}}}(M|s) \right. - C\textcolor{black}{ \mathrm{D}_{\mathrm{KL}}^{\max}\left( w ,\overline{w} \right) \notag } \\&
				\left. + \mathbb{E} _{M\sim \bar{w}\left( \cdot |s \right) ,\boldsymbol{a}\thicksim \bar{\boldsymbol{\pi}}\left( \cdot |s,M \right)}A_{ \boldsymbol{\pi }_{\textup{mar}}}(\boldsymbol{a}|s,M) -\textcolor{black}{\max C\sum_{j=1}^K{w\left( M^j|s \right) \mathrm{D}_{\mathrm{KL}}\left( \boldsymbol{\pi }(\cdot |s,M^j),\overline{\boldsymbol{\pi }}(\cdot |s,M^j) \right)}} \right.\Big].
			\end{split}
		\end{equation}
		
	\end{proof}

	\subsection{A.3 Derivation of the Equation~\eqref{eq26} for KL-divergence}~\label{deri26}
	\begin{align}
		\mathrm{D}_{\mathrm{KL}}\left( \boldsymbol{\pi }_{\textup{mar}}(\cdot |s),\bar{\boldsymbol{\pi}}_{\textup{mar}}(\cdot |s) \right)&=\mathbb{E} _{\boldsymbol{a}\sim \boldsymbol{\pi }_{\textup{mar}}}\left[ \log \boldsymbol{\pi }_{\textup{mar}}(\boldsymbol{a}|s)-\log \bar{\boldsymbol{\pi}}_{\textup{mar}}(\boldsymbol{a}|s) \right] \notag \\	&=\sum_{\boldsymbol{a}}^{}{\boldsymbol{\pi }_{\textup{mar}}(\boldsymbol{a}|s)\log \frac{\boldsymbol{\pi }_{\textup{mar}}(\boldsymbol{a}|s)}{\bar{\boldsymbol{\pi}}_{\textup{mar}}(\boldsymbol{a}|s)}} \notag \\	&=\sum_{\boldsymbol{a}}^{}{ \boldsymbol{\pi }_{\textup{mar}}(\boldsymbol{a}|s) \log \frac{\sum_{j=1}^K{w\left( M^j|s \right)}\boldsymbol{\pi }(\boldsymbol{a}|s,M^j)}{\sum_{j=1}^K{\overline{w}\left( M^j|s \right)}\overline{\boldsymbol{\pi }}(\boldsymbol{a}|s,M^j)}} \notag \\	&=\sum_{\boldsymbol{a}}^{}{\sum_{j=1}^K{w\left( M^j|s \right)}\boldsymbol{\pi }(\boldsymbol{a}|s,M^j)\log \frac{\sum_{j=1}^K{w\left( M^j|s \right)}\boldsymbol{\pi }(\boldsymbol{a}|s,M^j)}{\sum_{j=1}^K{\overline{w}\left( M^j|s \right)}\overline{\boldsymbol{\pi }}(\boldsymbol{a}|s,M^j)}} \notag \\	 &\leqslant \sum_{\boldsymbol{a}}^{}{\sum_{j=1}^K{\left[ w\left( M^j|s \right) \boldsymbol{\pi }(\boldsymbol{a}|s,M^j)\log \frac{w\left( M^j|s \right) \boldsymbol{\pi }(\boldsymbol{a}|s,M^j)}{\overline{w}\left( M^j|s \right) \overline{\boldsymbol{\pi }}(\boldsymbol{a}|s,M^j)} \right]}} \textcolor{black}{\textup{(by log sum inequality)}}\notag \\	&=\sum_{\boldsymbol{a}}^{}{\sum_{j=1}^K{w\left( M^j|s \right)}\boldsymbol{\pi }(\boldsymbol{a}|s,M^j)\left.[ \log \frac{w\left( M^j|s \right)}{\overline{w}\left( M^j|s \right)}+ \log \frac{\boldsymbol{\pi }(\boldsymbol{a}|s,M^j)}{\overline{\boldsymbol{\pi }}(\boldsymbol{a}|s,M^j)} \right.]} \notag \\	&=\sum_{j=1}^K{w\left( M^j|s \right)}\log \frac{w\left( M^j|s \right)}{\overline{w}\left( M^j|s \right)} \sum_{\boldsymbol{a}}^{}{\boldsymbol{\pi }(\boldsymbol{a}|s,M^j)+\sum_{\boldsymbol{a}}^{}{\sum_{j=1}^K{w\left( M^j|s \right)}\boldsymbol{\pi }(\boldsymbol{a}|s,M^j) \log \frac{\boldsymbol{\pi }(\boldsymbol{a}|s,M^j)}{\overline{\boldsymbol{\pi }}(\boldsymbol{a}|s,M^j)}}} \notag \\	&= \mathrm{D}_{\mathrm{KL}}\left( w\left(\cdot|s \right) ,\overline{w}\left(\cdot|s \right) \right)+\sum_{j=1}^K{w\left( M^j|s \right) \mathrm{D}_{\mathrm{KL}}\left( \boldsymbol{\pi }(\cdot |s,M^j),\overline{\boldsymbol{\pi }}(\cdot |s,M^j) \right) }.
	\end{align}

	\subsection{A.4 Proof of Guaranteed Monotonic Joint Policy Optimization for HCPO}\label{proHCPO}
	
	Before proof of monotonic improvement guarantee, we simplify certain in inequality~\eqref{eq27apx} to prepare for the subsequent steps. We begin by dealing with \(\mathbb{E} _{M\sim \bar{w}\left( \cdot |s \right) ,\boldsymbol{a}\thicksim \bar{\boldsymbol{\pi}}\left( \cdot |s,M \right)}A_{ \boldsymbol{\pi }_{\textup{mar}}}(\boldsymbol{a}|s,M)\), which is vital for assessing the influence of policy updates on the system's overall performance.
	\begin{align}~\label{eq28}
		\mathbb{E} _{M\sim \bar{w}\left( \cdot |s \right) ,\boldsymbol{a}\thicksim \bar{\boldsymbol{\pi}}\left( \cdot |s,M \right)}A_{\boldsymbol{\pi }_{\textup{mar}}}(\boldsymbol{a}|s,M)  &=\mathbb{E} _{M\sim \bar{w}\left( \cdot |s \right)}\left[ \mathbb{E} _{\boldsymbol{a}\thicksim \bar{\boldsymbol{\pi}}\left( \cdot |s,M \right)}A_{ \boldsymbol{\pi }_{\textup{mar}}}(\boldsymbol{a}|s,M) \right] \notag \\&=\sum_{j=1}^K{\bar{w}\left( M^j|s \right)}\left[ \mathbb{E} _{\boldsymbol{a}\thicksim \bar{\boldsymbol{\pi}}(\cdot |s,M^j)}A_{ \boldsymbol{\pi }_{\textup{mar}}}(\boldsymbol{a}|s,M^j) \right].
	\end{align}
	To facilitate subsequent representations, define:
	\begin{align}\label{eq29}
		L_{ \boldsymbol{\pi }_{\textup{mar}}}^{i_{1:l}}\left( \bar{\boldsymbol{\pi}}^{i_{1:l-1}},\hat{\pi}^{i_l}|s,M \right) \triangleq \mathbb{E} _{\boldsymbol{a}^{i_{1:l-1}}\thicksim \bar{\boldsymbol{\pi}}^{i_{1:l-1}}\left( \cdot |s,M \right) ,a^{i_l}\thicksim \hat{\pi}^{i_l}\left( \cdot |s,M \right)}  \left[ A_{ \boldsymbol{\pi }_{\textup{mar}}}^{i_l}\left( \boldsymbol{a}^{i_{1:l-1}},a^{i_l}|s,M \right) \right].
	\end{align}
	It is easy to obtain that for any \(\bar{\boldsymbol{\pi}}^{i_{1:l-1}}\), there is:
	\begin{align}\label{eq30}
		L_{ \boldsymbol{\pi }_{\textup{mar}}}^{i_{1:l}}\left( \bar{\boldsymbol{\pi}}^{i_{1:l-1}},\pi ^{i_l}|s,M \right) & =\mathbb{E} _{\boldsymbol{a}^{i_{1:l-1}}\thicksim \bar{\boldsymbol{\pi}}^{i_{1:l-1}}\left( \cdot |s,M \right) ,a^{i_l}\thicksim \pi ^{i_l}\left( \cdot |s,M \right)}\left[ A_{ \boldsymbol{\pi }_{\textup{mar}}}^{i_l}\left( \boldsymbol{a}^{i_{1:l-1}},a^{i_l}|s,M \right) \right] \notag \\&=\mathbb{E} _{\boldsymbol{a}^{i_{1:l-1}}\thicksim \bar{\boldsymbol{\pi}}^{i_{1:l-1}}\left( \cdot |s,M \right)}\left[ \mathbb{E} _{a^{i_l}\thicksim \pi ^{i_l}\left( \cdot |s,M \right)}\left[ A_{ \boldsymbol{\pi }_{\textup{mar}}}^{i_l}\left( \boldsymbol{a}^{i_{1:l-1}},a^{i_l}|s,M \right) \right] \right] \notag \\&=0.
	\end{align}
	Consider agents updated in order \(i_{1:N}\), \(\forall M^j, j\in \left\{ 1,2,...,K \right\} \), we have: 
	\begin{align}\label{eq31}
		\mathbb{E} _{\boldsymbol{a}\thicksim \bar{\boldsymbol{\pi}}(\cdot |s,M^j)}A_{\boldsymbol{\pi }_{\textup{mar}}}(\boldsymbol{a}|s,M^j) &=\mathbb{E} _{\boldsymbol{a}^{i_{1:N}}\thicksim \bar{\boldsymbol{\pi}}^{i_{1:N}}(\cdot |s,M^j)}A_{\boldsymbol{\pi }_{\textup{mar}}}^{i_{1:N}}(\boldsymbol{a}^{i_{1:N}}|s,M^j) \notag \\ \text{(by}~\eqref{eq14apx} \text{)}&=\mathbb{E} _{\boldsymbol{a}^{i_{1:N}}\thicksim \bar{\boldsymbol{\pi}}^{i_{1:N}}(\cdot |s,M^j)}\sum_{l=1}^N{A_{ \boldsymbol{\pi }_{\textup{mar}}}^{i_l}\left( \boldsymbol{a}^{i_{1:l-1}},a^{i_l}|s,M^j \right)} \notag \\ \text{(by}~\eqref{eq29} \text{)}&=\sum_{l=1}^N{L_{\boldsymbol{\pi }_{\textup{mar}}}^{i_{1:l}}\left( \bar{\boldsymbol{\pi}}^{i_{1:l-1}},\bar{\pi}^{i_l}|s,M^j \right)}.
	\end{align}
	In addition, for \(\mathrm{D}_{\mathrm{KL}}\) in the last item in~\eqref{eq26}, there is:
	\begin{align}\label{eq33new}
		 \mathrm{D}_{\mathrm{KL}}\left( \boldsymbol{\pi }(\cdot |s,M^j),\overline{\boldsymbol{\pi }}(\cdot |s,M^j) \right)&=\mathbb{E} _{\boldsymbol{a}\sim \boldsymbol{\pi }(\cdot |s,M^j)}\left[ \log \boldsymbol{\pi }(\boldsymbol{a}|s,M^j)-\log \bar{\boldsymbol{\pi}}(\boldsymbol{a}|s,M^j) \right] \notag \\&=\mathbb{E} _{\boldsymbol{a}\sim \boldsymbol{\pi }(\cdot |s,M^j)}\left.\Big[ \log \left.( \prod_{l=1}^N{\pi ^{i_l}}(a^{i_l}|s,M^j) \right.)  -\log \left.( \prod_{l=1}^N{\bar{\pi}^{i_l}}(a^{i_l}|s,M^j) \right.) \right. \Big] \notag \\&=\mathbb{E} _{\boldsymbol{a}\sim \boldsymbol{\pi }(\cdot |s,M^j)}\left[ \sum_{l=1}^N{\log}\pi ^{i_l}(a^{i_l}|s,M^j)-\sum_{l=1}^N{\log}\bar{\pi}^{i_l}(a^{i_l}|s,M^j) \right] \notag \\&=\sum_{l=1}^N{\mathbb{E} _{a^{i_l}\thicksim \pi ^{i_l}(\cdot |s,M^j),\boldsymbol{a}^{-i_l}\thicksim \boldsymbol{\pi }^{-i_l}(\cdot |s,M^j)}}\left.[ \log \pi ^{i_l}(a^{i_l}|s,M^j) -\log \bar{\pi}^{i_l}(a^{i_l}|s,M^j) \right.] \notag \\&=\sum_{l=1}^N{\mathrm{D}_{\mathrm{KL}}}\left( \pi ^{i_l}(\cdot |s,M^j),\bar{\pi}^{i_l}(\cdot |s,M^j) \right).
	\end{align}
	Then, we can prove the monotonic improvement guarantee for the performance of the joint policy \textcolor{black}{\(\boldsymbol{\pi}_{\textup{mar}}\)}, under the following two-level policy update mechanisms:
	
	(i)~The conductor's policy \(w\left(\cdot|s \right) =w_k\left(\cdot|s \right) \) is updated first according to the rule:
	\begin{equation}\label{eq34apx}
		\textcolor{black}{w_{k+1} = \underset{\bar{w}}{\mathrm{arg\,max}}~\Bigg[
			\mathbb{E}_{s\sim \rho_{\boldsymbol{\pi}_{\textup{mar},k}}, M \sim \bar{w}} A_{\boldsymbol{\pi}_{\textup{mar},k}}(M | s) - C \mathrm{D}_{\mathrm{KL}}^{\max} \left( w_k, \bar{w} \right) \Bigg].}
	\end{equation}
	(ii)~For each \( M^j\) where \(j\in \left\{ 1,2,...,K \right\} \), the agents update their policies sequentially according to the order \(i_{1:N}\), with the following update rule:
    \begin{align}\label{eq35apx}
    \pi _{k+1}^{i_l}(\cdot |s,M^j)=\underset{\bar{\pi}^{i_l}\left( \cdot |s,M^j \right)}{arg\max}\left[ w_{k+1}\left( M^j|s \right)  L_{ \boldsymbol{\pi }_{\textup{mar}, k}}^{i_{1:l}}\left( \boldsymbol{\pi }_{k+1}^{i_{1:l-1}},\bar{\pi}^{i_l}|s,M^j \right) \right. \notag \\     \left.-\textcolor{black}{\max_s C w_k\left( M^j|s \right) \mathrm{D}_{\mathrm{KL}}\left( \pi _{k}^{i_l}(\cdot |s,M^j),\bar{\pi}^{i_l}(\cdot |s,M^j) \right)  }\right].
\end{align}

\begin{proof}

	Starting from~\eqref{eq27apx}, combining~\eqref{eq34apx} and~\eqref{eq35apx}, we can obtain: 
		\begin{equation}
			\begin{split}\label{eq36}
				J(\boldsymbol{\pi }_{\textup{mar},k+1})
				&\geqslant J(\boldsymbol{\pi }_{\textup{mar},k})+\mathbb{E} _{\mathrm{s}\thicksim \rho _{\boldsymbol{\pi }_{\textup{mar},k}}}\left[ \mathbb{E} _{M\sim w_{k+1}\left( \cdot |s \right)}A_{\boldsymbol{\pi }_{\textup{mar},k}}(M|s) -C\textcolor{black}{ \mathrm{D}_{\mathrm{KL}}^{\max}\left( w_k ,w_{k+1} \right) } \right] \\ &   +\mathbb{E} _{\mathrm{s}\thicksim \rho _{\boldsymbol{\pi }_{\textup{mar},k}}}\Bigl[ \mathbb{E} _{M\sim w_{k+1}\left( \cdot |s \right) ,\boldsymbol{a}\thicksim \boldsymbol{\pi }_{k+1}^{}(\cdot |s,M)}A_{\boldsymbol{\pi }_{\textup{mar},k}}(\boldsymbol{a}|s,M)  \\ & \hspace{15em} - \textcolor{black}{\max_s C\sum_{j=1}^K{w_k\left( M^j|s \right) \mathrm{D}_{\mathrm{KL}}\left( \boldsymbol{\pi }_{k}^{}(\cdot |s,M^j),\boldsymbol{\pi }_{k+1}^{}(\cdot |s,M^j) \right)} } \Bigr] \\
				&=J(\boldsymbol{\pi }_{\textup{mar},k})+\mathbb{E} _{\mathrm{s}\thicksim \rho _{\boldsymbol{\pi }_{\textup{mar},k}}}\left[ \mathbb{E} _{M\sim w_{k+1}\left( \cdot |s \right)}A_{\boldsymbol{\pi }_{\textup{mar},k}}(M|s)-C\textcolor{black}{ \mathrm{D}_{\mathrm{KL}}^{\max}\left( w_k ,w_{k+1} \right) }  \right]  \\ & +\mathbb{E} _{\mathrm{s}\thicksim \rho _{\boldsymbol{\pi }_{\textup{mar},k}}}\Bigl[ \sum_{j=1}^K{w_{k+1}\left( M^j|s \right)}\bigl[ \mathbb{E} _{\boldsymbol{a}\thicksim \boldsymbol{\pi }_{k+1}^{}(\cdot |s,M^j)}A_{\boldsymbol{\pi }_{\textup{mar},k}}(\boldsymbol{a}|s,M^j) \bigr]  \Bigr.\\ & \hspace{9em} \Bigl.-\textcolor{black}{\max_s C\sum_{j=1}^K{w_k\left( M^j|s \right) \mathrm{D}_{\mathrm{KL}}\left( \boldsymbol{\pi }_{k}^{}(\cdot |s,M^j),\boldsymbol{\pi }_{k+1}^{}(\cdot |s,M^j) \right)} } \Bigr]~\left(\text{combine}~\eqref{eq28} \right) 
				\\ & \textcolor{black}{=} J(\boldsymbol{\pi }_{\textup{mar},k})+\mathbb{E} _{\mathrm{s}\thicksim \rho _{\boldsymbol{\pi }_{\textup{mar},k}}}\left[ \mathbb{E} _{M\sim w_{k+1}\left( \cdot |s \right)}A_{\boldsymbol{\pi }_{\textup{mar},k}}(M|s)-C\textcolor{black}{ \mathrm{D}_{\mathrm{KL}}^{\max}\left( w_k ,w_{k+1} \right) } \right] \\&+\mathbb{E} _{\mathrm{s}\thicksim \rho _{\boldsymbol{\pi }_{\textup{mar},k}}}\Bigl[ \sum_{j=1}^K{w_{k+1}\left( M^j|s \right)} \sum_{l=1}^N{L_{\boldsymbol{\pi }_{\textup{mar}, k}}^{i_{1:l}}\left( \boldsymbol{\pi }_{k+1}^{i_{1:l-1}},\pi _{k+1}^{i_l}|s,M^j \right)} \Bigr. \\&  \hspace{6em} \Bigl.-\textcolor{black}{\max_s C\sum_{j=1}^K{w_k\left( M^j|s \right) \sum_{l=1}^N{\mathrm{D}_{\mathrm{KL}}}\left( \pi _{k}^{i_l}(\cdot |s,M^j),\pi _{k+1}^{i_l}(\cdot |s,M^j) \right) }} \Bigr]~\left( \text{combine}~\eqref{eq31}\eqref{eq33new}\right) 
				\\& \geqslant J(\boldsymbol{\pi }_{\textup{mar},k})+\mathbb{E} _{\mathrm{s}\thicksim \rho _{\boldsymbol{\pi }_{\textup{mar},k}}}\left[ \mathbb{E} _{M\sim w_{k+1}\left( \cdot |s \right)}A_{\boldsymbol{\pi }_{\textup{mar},k}}(M|s)-C\textcolor{black}{ \mathrm{D}_{\mathrm{KL}}^{\max}\left( w_k ,w_{k+1} \right) } \right] \\&+\mathbb{E} _{\mathrm{s}\thicksim \rho _{\boldsymbol{\pi }_{\textup{mar},k}}} \Bigl[  \sum_{l=1}^N \sum_{j=1}^K \bigl[ w_{k+1}\left( M^j|s \right)  L_{\boldsymbol{\pi }_{\textup{mar}, k}}^{i_{1:l}}\left( \boldsymbol{\pi }_{k+1}^{i_{1:l-1}},\pi _{k+1}^{i_l}|s,M^j \right)  \bigr. \bigr. \\
				& \hspace{17em} \bigr. \bigr. -\textcolor{black}{\max_s C w_k\left( M^j|s \right) \mathrm{D}_{\mathrm{KL}}\left( \pi _{k}^{i_l}(\cdot |s,M^j),\pi _{k+1}^{i_l}(\cdot |s,M^j) \right)} \bigr] \Bigr]
				\\ &\geqslant J(\boldsymbol{\pi }_{\textup{mar},k})+\mathbb{E} _{\mathrm{s}\thicksim \rho _{\boldsymbol{\pi }_{\textup{mar},k}}}\left[ \mathbb{E} _{M\sim w_k\left( \cdot |s \right)}A_{\boldsymbol{\pi }_{\textup{mar},k}}(M|s)-C\mathrm{D}_{\mathrm{KL}}^{\max}\textcolor{black}{\left( w_k, w_k  \right)} \right] \\&+\mathbb{E} _{\mathrm{s}\thicksim \rho _{\boldsymbol{\pi }_{\textup{mar},k}}}
				\biggl[ \sum_{l=1}^N \sum_{j=1}^K \Bigl[ w_{k+1}\left( M^j|s \right) L_{\boldsymbol{\pi }_{\textup{mar}, k}}^{i_{1:l}}\left( \boldsymbol{\pi }_{k+1}^{i_{1:l-1}},\pi_{k}^{i_l} |s,M^j \right) \Bigr. \\ & \hspace{10em} \Bigl. -\textcolor{black}{\max_s C w_k\left( M^j|s \right) \mathrm{D}_{\mathrm{KL}}\left( \pi _{k}^{i_l}(\cdot |s,M^j), \pi _{k}^{i_l} (\cdot |s,M^j) \right) } \Bigr] \biggr]~\text{(combine}~\eqref{eq34apx}\eqref{eq35apx}\text{)} \\&=J(\boldsymbol{\pi }_{\textup{mar},k})+\mathbb{E} _{\mathrm{s}\thicksim \rho _{\boldsymbol{\pi }_{\textup{mar},k}}}\left[ \mathbb{E} _{M\sim w_k \left( \cdot |s \right)}A_{\boldsymbol{\pi }_{\textup{mar},k}}(M|s)-0 \right] \\& \hspace{3em} +\mathbb{E} _{\mathrm{s}\thicksim \rho _{\boldsymbol{\pi }_{\textup{mar},k}}}\biggl[ \sum_{l=1}^N{\sum_{j=1}^K{\Bigl[ w_{k+1}\left( M^j|s \right) L_{\boldsymbol{\pi }_{\textup{mar}, k}}^{i_{1:l}}\left( \boldsymbol{\pi }_{k+1}^{i_{1:l-1}}, \pi _{k}^{i_l} |s,M^j \right) -0 \Bigr]}}\biggr]. 
			\end{split}
		\end{equation}
	\normalsize{By the meaning of the advantage function and the equation~\eqref{eq30}, there is:}
	\begin{equation*}
		\mathbb{E} _{M\sim w_k\left( \cdot |s \right)}A_{\boldsymbol{\pi }_{\textup{mar},k}}(M|s)=0,~L_{\boldsymbol{\pi }_{\textup{mar}, k}}^{i_{1:l}}\left( \boldsymbol{\pi }_{k+1}^{i_{1:l-1}},\pi _{k}^{i_l}|s,M^j \right) =0.
	\end{equation*}
	Therefore,~\eqref{eq36} can be reduced to:
	\begin{equation}
		J\left(\boldsymbol{\pi }_{\textup{mar},k+1}\right)\geqslant J(\boldsymbol{\pi }_{\textup{mar},k})+0=J(\boldsymbol{\pi }_{\textup{mar},k}).
	\end{equation}

\end{proof}

\newpage
\section{B. Algorithms}

	\begin{algorithm}
		\caption{Pseudocode of HCPO}\label{algorithm1}
		\textbf{Input}: Minibatch size \(B_1\), number of agents \(N\), training times \(\Lambda\), steps per episode \(T\)
		\begin{algorithmic}[1]
			\STATE {\textbf{Initialization:} Global V-value network \(\left\{ \phi _{0}^{} \right\} \), centralized conductor's actor network \(\{\Psi _0\}\), local conductors' actor networks \(\{\psi_{0}^{i}, \forall i\in \mathcal{N} \}\), agents' actor networks \(\{\theta _{0}^{i}, \forall i\in \mathcal{N} \}\), replay buffer~\(\mathcal{B}_1 \)}
			\FOR {\(k=0,1,\cdots,\Lambda-1\)} 
			\STATE {Collect a set of trajectories by running \(\boldsymbol{\pi }_{\textup{mar}, k}(\boldsymbol{a}_t|s_t)=\mathbb{E} _{M_t\sim w^{\Psi _k}(\cdot | s_t)}\boldsymbol{\pi }_{\boldsymbol{\theta }_k}(\boldsymbol{a}_t|s_t,M_t)\) and save transitions \(\{(s_{t}^{},M_{t}^{},\boldsymbol{a}_{t}^{},s_{t+1}^{},r_t),\forall t\in T\}\) into \(\mathcal{B}_1\)}
			\STATE {Sample a random minibatch of \(B_1\) episodes from replay buffer \(\mathcal{B}_1\)}
			\STATE {Compute \(A_{\boldsymbol{\pi }_{\textup{mar}, k}}(M_{b,t}|s_{b,t})\) and \(A_{\boldsymbol{\pi }_{\textup{mar}, k}}(\boldsymbol{a}_{b,t}|s_{b,t},M_{b,t})\) based on global V-value network with generalized advantage estimation (GAE)}
			\STATE {Update \(\Psi _{k+1}=\Psi _k+\alpha ^{j_1}\hat{\beta}_{conductor,k}\hat{x}_{conductor, k}^{}\) as designed in Algorithm~\ref{algorithm2}}
            \STATE {Draw a random permutation of agents \(i_{1:N}\)}
			\STATE {Set \(\varTheta ^{i_1}(\boldsymbol{a}_{b,t}|s_{b,t},M_{b,t})=A_{\boldsymbol{\pi }_{\textup{mar}, k}}(\boldsymbol{a}_{b,t}|s_{b,t},M_{b,t})\)}
			\FOR {agent \(i_l=i_1,\cdots,i_N\)} 
            \STATE{Update \(\psi^{i_l}_{k+1}\) by minimizing the cross-entropy between \(\psi^{i_l}_{k}\) and \(\Psi_{k+1}\): \(\mathcal{L}_{ce}^{i_l}=-\mathbb{E}_{w^{\Psi_{k+1}}}[\log \psi^{i_l}_{k}]\)}
			\STATE {Update agent \(i_l\)'s policy by \(\theta_{k+1}^{i_l}=\theta_{k}^{i_l}+\alpha^{j_2}\hat{\beta}_{agent, k}^{i_{l}}\hat{x}_{agent, k}^{i_{l}}\)  as designed in Algorithm~\ref{algorithm3}}
			\STATE {Compute 
				\begin{equation*}
					\varTheta ^{i_{1:l+1}}(\boldsymbol{a}_{ b,t}|s_{ b,t},M_{ b,t})=\frac{\pi _{\theta _{k+1}^{i_l}}^{i_l}\left( a_{b,t}^{i_l}\mid o_{b,t}^{i_l},M_{b,t} \right)}{\pi _{\theta _{k}^{i_l}}^{i_l}\left( a_{b,t}^{i_l}\mid o_{b,t}^{i_l},M_{b,t} \right)}\varTheta ^{i_{1:l}}(\boldsymbol{a}_{b,t}|s_{b,t},M_{b,t}) 
				\end{equation*}
				unless \(l=N\)
			}
			\ENDFOR
			\STATE {Update V-value network by following formula:
				\begin{equation*}
					\phi _{k+1}=\mathrm{arg}\min_{\phi} \frac{1}{B_1T}\sum_{b=1}^{B_1}{\sum_{t=0}^T{\left( V_{\phi}(s_t)-\hat{R}_t \right) ^2}}
			\end{equation*}}
			\STATE {Clear the replay buffer \(\mathcal{B}_1 \) and minibatch \(B_1\)}
			\ENDFOR
		\end{algorithmic}
	\end{algorithm}

	\begin{algorithm}
		\caption{Conjugate Gradient Approach for Updating Centralized Conductor's Policy Parameters}\label{algorithm2}
		\textbf{Input}: {Minibatch size \(B_1\), episode number \(k\), steps per episode \(T\)}\\
		\textbf{Output}: {New centralized conductor's policy parameter \(\Psi_{k+1}\)} 
		\begin{algorithmic}[1]
			\STATE {Estimate the gradient of the conductor’s maximization objective 
				\begin{equation*}
					\hat{g}_{conductor, k}^{}=\frac{1}{B_1}\sum_{b=1}^{B_1}{\sum_{t=1}^T{\nabla _{\Psi _k}}}\log w^{\Psi _k}(M_{b,t} | s_{b,t})A_{\boldsymbol{\pi} _{\textup{mar},k}}(M_{b,t} | s_{b,t})
				\end{equation*}
				Use the conjugate gradient algorithm to compute the update direction
				\begin{equation*}
					\hat{x}_{conductor, k}^{}\approx {\hat{\boldsymbol{H}}_{conductor, k}}^{-1}\hat{g}_{conductor, k}^{}
				\end{equation*}
				where \(\hat{\boldsymbol{H}}_{conductor, k}^{}\) is the Hessian of the average KL-divergence
				\begin{equation*}
					\frac{1}{B_1T}\sum_{b=1}^{B_1}{\sum_{t=1}^T{\mathrm{D}_{\mathrm{KL}}}}\left(  w^{\Psi _k}(\cdot | s_{b,t}), w^{\Psi} (\cdot | s_{b,t}) \right) 
			\end{equation*}}
			\STATE {Estimate the maximal step size allowing for meeting the KL-constraint
				\begin{equation*}
					\hat{\beta}_{conductor,k}\approx \sqrt{\frac{2\delta _1}{(\hat{x}_{conductor,k})^{\top}\hat{\boldsymbol{H}}_{conductor, k}^{}(\hat{x}_{conductor,k})}} 
			\end{equation*}}
			\STATE {Update centralized conductor’s policy by \(\Psi _{k+1}=\Psi _k+\alpha ^{j_1}\hat{\beta}_{conductor,k}\hat{x}_{conductor, k}^{}\) }
		\end{algorithmic}
	\end{algorithm}

	\begin{algorithm}
		\caption{Conjugate Gradient Approach for Updating Local Agent's Policy Parameters}\label{algorithm3}
		\textbf{Input}: {Minibatch size \(B_1\), episode number \(k\), steps per episode \(T\), agent index \(i_l\)}\\
		\textbf{Output}:{New Agent's policy parameter \(\theta_{k+1}^{i_l}\)} 
		\begin{algorithmic}[1]

			\STATE {Estimate the gradient of the agent’s maximization objective
				\begin{equation*}
					\hat{g}_{agent, k}^{i_{l}}=\frac{1}{B_1}\sum_{b=1}^{B_1}{\sum_{t=1}^Tw^{\Psi_{k+1}}(M_{b,t}|s_{b,t}){\nabla _{\theta _{k}^{i_l}}}} \log \pi _{\theta _{k}^{i_l}}^{i_l}\left( a_{b,t}^{i_l}\mid o_{b,t}^{i_l},M_{ b,t}^{} \right) \varTheta ^{i_{1:l}}(\boldsymbol{a}_{ b,t}^{}|s_{ b,t}^{},M_{ b,t}^{})
			\end{equation*}}
			Use the conjugate gradient algorithm to compute the update direction
			\begin{equation*}
				\hat{x}_{agent, k}^{i_l}\approx {\hat{\boldsymbol{H}}_{agent, k}^{i_l}}{}^{-1} \hat{g}_{agent, k}^{i_l} 
			\end{equation*}
			where \(\hat{\boldsymbol{H}}_{agent, k}^{i_l}\) is the Hessian of the average KL-divergence
			\begin{equation*}
				\frac{1}{B_1T}\sum_{b=1}^{B_1}{\sum_{t=1}^T{ w^{\psi _k}\left( M_{ b,t}^{}|s_{ b,t}^{} \right) \mathrm{D}_{\mathrm{KL}}}}\left( \pi _{\theta _{k}^{i_l}}^{i_l}(\cdot |o_{ b,t}^{i_l},M_{ b,t}^{}),\pi _{\theta ^{i_l}}^{i_l}(\cdot |o_{ b,t}^{i_l},M_{ b,t}^{}) \right)
			\end{equation*}
			\STATE {Estimate the maximal step size allowing for meeting the KL-constraint
				\begin{equation*}
					\hat{\beta}_{agent, k}^{i_l}\approx \sqrt{\frac{2\delta _2}{(\hat{x}_{agent, k}^{i_{l}})^{\top}\boldsymbol{\hat{H}}_{agent, k}^{i_l}(\hat{x}_{agent, k}^{i_{l}})}}
			\end{equation*}}
			\STATE {Update agent \(i_l\)'s policy by \(\theta_{k+1}^{i_l}=\theta_{k}^{i_l}+\alpha^{j_2}\hat{\beta}_{agent, k}^{i_{l}}\hat{x}_{agent, k}^{i_{l}}\)}
		\end{algorithmic}
	\end{algorithm}

\newpage
\section{C. Experimental Details}\label{apxexpdetail}
\subsection{C.1 StarCraftII Multi-agent Challenge}\label{apxsmac}
In the SMAC environment, we basically adhere to the official implements and hyperparameter settings of MAVEN\footnote{https://github.com/starry-sky6688/MARL-Algorithms/}~\citet{mahajan2019maven}, HAPPO and HATRPO\footnote{https://github.com/PKU-MARL/TRPO-PPO-in-MARL}~\citet{kuba2022trust}, A2PO\footnote{https://github.com/xihuai18/A2PO-ICLR2023}~\citet{wang2023order}, HAA2C\footnote{https://github.com/PKU-MARL/HARL}~\citet{zhong2024heterogeneous}, as shown in Table~\ref{smacpara}. We adopt the evaluation method from MAPPO~\citet{yu2022surprising} and compare HCPO against other algorithms on five maps. After each training iteration, 32 evaluation games are played and the winning rate of these 32 games is calculated. Finally, we take the median winning rate of the last ten evaluations as the performance metric for each random seed and report the average median winning rate based on the five random seeds in Table~\ref{winratetableSMAC}. We can observe that HCPO achieves the highest average median winning rate across five maps, reaching an impressive 97.82\%. This result highlights its superior performance compared to other algorithms. Moreover, HCPO also exhibits the lowest standard deviation, demonstrating its high stability. 
\begin{table}[htbp]
	\centering
	\begin{tabular}{llllll}
		\toprule
		hyperparameters & value & hyperparameters & value & hyperparameters & value \\
		\cmidrule(r){1-2}
		\cmidrule(r){3-4}
		\cmidrule(r){5-6}
		critic lr & 5e-4 & optimizer & Adam & stacked-frames & 1 \\
		gamma & 0.95 &  gamma in corridor & 0.99 & optim eps & 1e-5 \\  batch size & 3200 & gain & 0.01 & hidden layer & 1 \\
		training threads & 32 & actor network & mlp & num mini-batch & 1 \\
		rollout threads & 16 & hypernet embed & 64 & max grad norm & 10 \\
		episode length & 200 & activation & ReLU & hidden layer dim & 64 \\
		use huber loss & True & conductor kl-threshold  &  0.01 & kl-threshold & 0.06\\
		accept-ratio & 0.5 & K  &  10 &  &  \\
		\bottomrule
	\end{tabular}
    \caption{Common hyperparameters in SMAC}
    \label{smacpara}
\end{table}

\begin{table}[htbp]
	\centering
	\begin{tabularx}{\textwidth}{lXXXXXXX}
		\toprule
		Task     & Difficulty     & HCPO  &   HAPPO   &  HATRPO  &  HAA2C  &  A2PO & MAVEN\\
		\midrule
		3s5z     & hard   & \textbf{100}\scriptsize(0.0) &   89.1\scriptsize(2.3)  &  \textbf{100}\scriptsize(1.4) &  67.2\scriptsize(11.8)  &  98.4\scriptsize(1.3)& 0.0\scriptsize(0.0)\\
		5m\_vs\_6m     & hard   & \textbf{93.8}\scriptsize(2.9) &   90.6\scriptsize(5.4)  &  70.3\scriptsize(8.7) &  0.0\scriptsize(0.0)  &  92.2\scriptsize(5.9) &43.8\scriptsize(1.4)\\
		8m\_vs\_9m    & hard   & \textbf{100}\scriptsize(0.7) &   81.2\scriptsize(6.9)  &  90.6\scriptsize(2.4) &  18.8\scriptsize(2.7)  &  93.8\scriptsize(4.5)& 18.8\scriptsize(2.9)\\
		10m\_vs\_11m  & hard   & \textbf{98.4}\scriptsize(1.6) &   85.9\scriptsize(8.4)  &  93.8\scriptsize(2.9) &  20.3\scriptsize(11.1)  &  96.9\scriptsize(2.2)&18.8\scriptsize(2.6)\\
		corridor  & super & \textbf{96.9}\scriptsize(1.6) &   \textbf{96.9}\scriptsize(0.7)  &  90.6\scriptsize(2.3) &  85.9\scriptsize(5.7)  &  93.8\scriptsize(2.4)&40.6\scriptsize(4.1)\\
		\midrule
		Overall  & /   & \textbf{97.82\scriptsize(1.36)} &   88.74\scriptsize(4.74)  &  89.06\scriptsize(3.54) &  38.44\scriptsize(6.26)  &  95.02\scriptsize(3.26) &24.4\scriptsize(2.2)\\
		\bottomrule
	\end{tabularx}
    \caption{Average evaluation median winning rate and standard deviation (across five seeds) within SMAC scenarios for distinct methods}
    \label{winratetableSMAC}
\end{table}
\textcolor{black}{Taking the \textit{3s5z} map as an example, after the completion of the training process (10 million steps) for all agents' local conductors and local actors, we visualize the gameplay scenarios involving our allies and the enemies in Figure~\ref{SMACvisfigs}. Figure~\ref{subfig9a} depicts the initial phase of the game, with both teams launching their attacks. The Figure~\ref{subfig9b} captures a moment where one of the stalkers is intentionally drawing the enemy fire, acting as a decoy to distract and absorb the enemy's attention and attacks. This instruction allows the other two stalkers to approach from behind the zealots, positioning them for a surprise attack on the enemy's flank. The green circles highlight the units under our control, with the health and energy bars visible above each unit, indicating their current status. The blue line pointing towards the enemy unit signifies the attack direction of our stalker, demonstrating the coordination between our units. Figure~\ref{subfig9c} illustrates a strategic (instruction-driven) moment where our forces are effectively split into two groups. The green arrows represent the movement direction of our stalkers. These two stalkers have identified allies in the yellow circle on the right with low health and automatically move to provide support. This self-directed action, triggered by the agent's local assessment of the situation (observation and instruction), exemplifies the autonomous decision-making capabilities of our agents under the guidance of HCPO's latent instructions. These instructions underscore the significance of actions and strategic positioning in achieving tactical advantage in the game.} The experiments are performed on a server with 3 RTX 4090 GPUs, and 3 Xeon(R) Gold 6430 16-core CPUs, and 128GB Ram.

\begin{figure*}[htbp]
	\centering
	\begin{subfigure}{0.31\textwidth}
		\includegraphics[width=\linewidth]{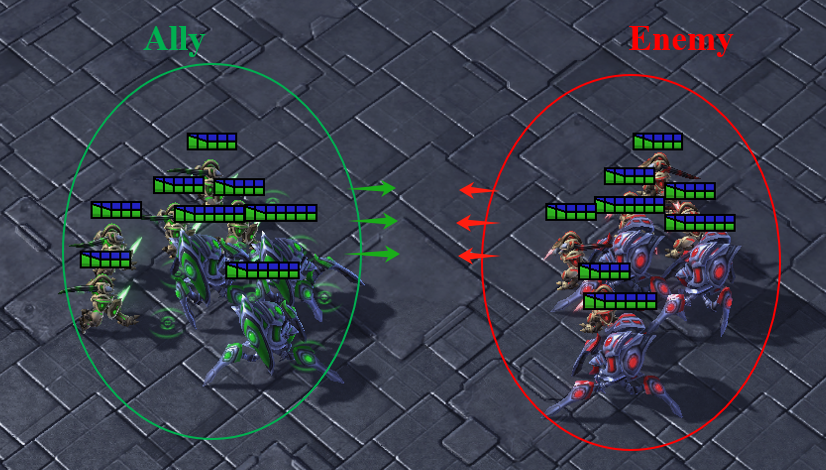}
		\caption{}
		\label{subfig9a}
	\end{subfigure}
	\begin{subfigure}{0.31\textwidth}
		\includegraphics[width=\linewidth]{visualfig2.png}
		\caption{}
		\label{subfig9b}
	\end{subfigure}
	\begin{subfigure}{0.31\textwidth}
		\includegraphics[width=\linewidth]{visualfig3.png}
		\caption{}
		\label{subfig9c}
	\end{subfigure}
	\caption{Effective coordination in SMAC on the \textit{3s5z} map: A visual analysis of agent strategies.}\label{SMACvisfigs}
\end{figure*}

\subsection{C.2 Multi-agent MuJoCo}
The detailed experimental hyperparameters are listed in Table~\ref{mujocopara}. We present a comparison of the return performance of different baselines across four tasks in Figure~\ref{boxfigsapx}. The median return values of HCPO algorithm reach the highest levels, demonstrating its superior performance in these tasks. Additionally, the return values exhibit a narrow interquartile range, indicating high stability across multiple experimental runs. Less outliers further imply the robustness of HCPO during training, as it generates returns without being compromised by occasional suboptimal runs. In summary, these findings suggest that HCPO maintains efficiency and reliability across a variety of tasks.

\textcolor{black}{In Figure~\ref{HCPOentropy1apx}, we employ the t-SNE (t-Distributed Stochastic Neighbor Embedding) technique to visualize the states explored by the HCPO, HATRPO, and A2PO during the early stages of training in the \textit{Walker2d-v2-6\(\times\)1}. Each dot in the plots represents a state explored by the algorithms. We use orange shading to indicate state probability density within each grid cell, defined as the fraction of samples falling in that cell relative to the total number of samples, with darker colors representing higher visitation frequencies. By examining the color distribution across the plots, we can visually assess the areas and densities explored by each algorithm. The figure also provides the average nearest neighbor distance (Avg. Dist.) and entropy values for each algorithm. A larger average nearest neighbor distance (e.g., HCPO’s 0.72) implies greater separation among visited states, reflecting better coverage diversity, while entropy measures exploration breadth. In summary, these metrics show that HCPO exhibits well-balanced exploration characteristics.}
\begin{table}[htbp]
	\centering
	\begin{tabular}{llllll}
		\toprule
		hyperparameters & value & hyperparameters & value & hyperparameters & value \\
		\cmidrule(r){1-2}
		\cmidrule(r){3-4}
		\cmidrule(r){5-6}
		critic lr & 5e-3 & optimizer & Adam & num mini-batch & 1 \\
		gamma & 0.99 &  optim eps & 1e-5 & batch size & 4000 \\  gain & 0.01 & hidden layer & 1 & training threads & 8 \\
		std y coef & 0.5 & actor network & mlp & rollout threads & 20 \\
		std x coef & 1 & max grad norm & 10 & episode length & 200 \\
		activation & ReLU & hidden layer dim & 128 & eval episode & 40 \\
		conductor kl-threshold  &  0.01 & kl-threshold & 0.005&accept-ratio & 0.5 \\
		K  &  10 &    &    &  &  \\
		\bottomrule
	\end{tabular}
    \caption{Common hyperparameters in MA-MuJoCo}
    \label{mujocopara}
\end{table}
\begin{figure}[htbp]
	\centering
	\begin{subfigure}{0.23\textwidth}
		\includegraphics[width=\linewidth]{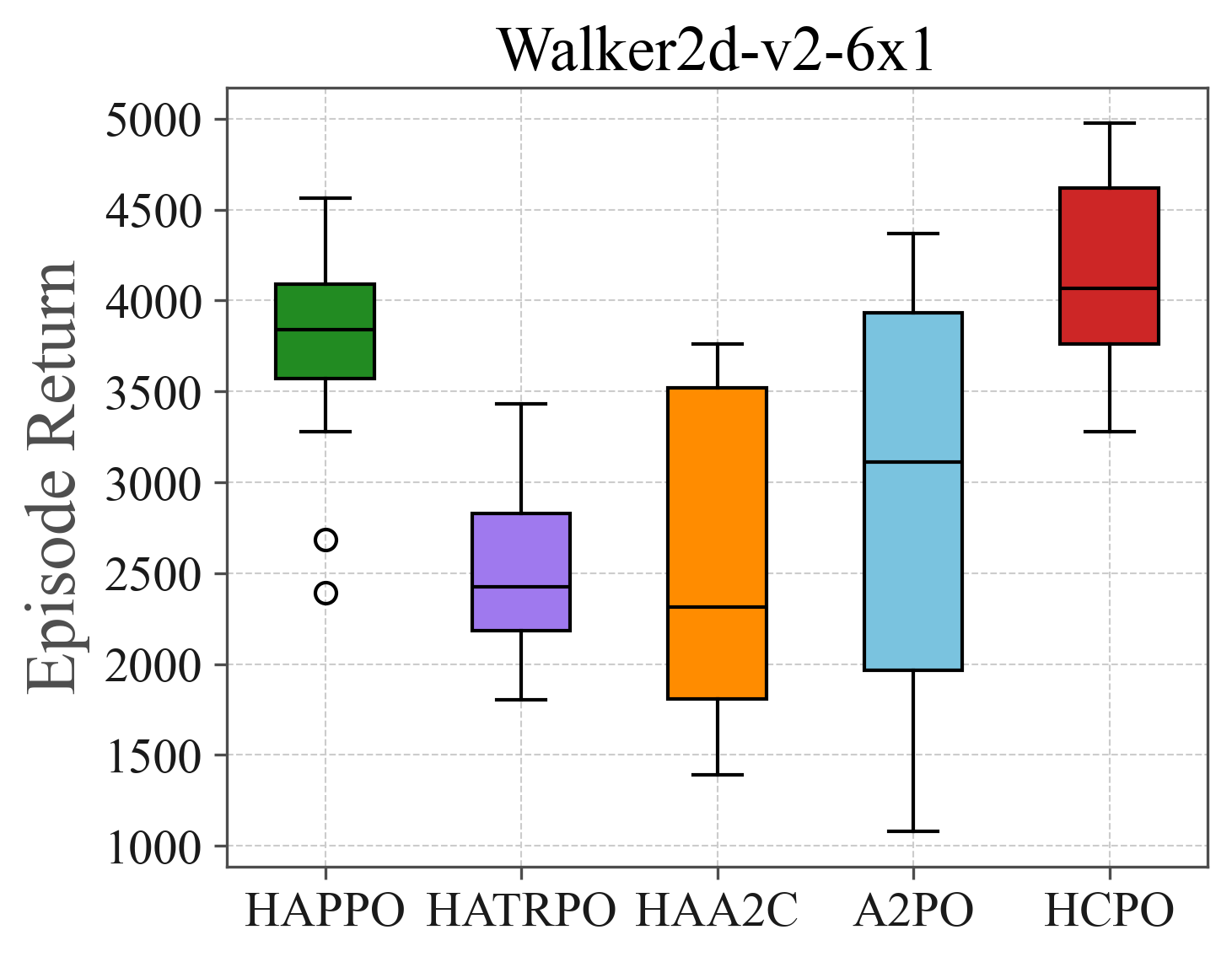}
	\end{subfigure}%
	\begin{subfigure}{0.23\textwidth}
		\includegraphics[width=\linewidth]{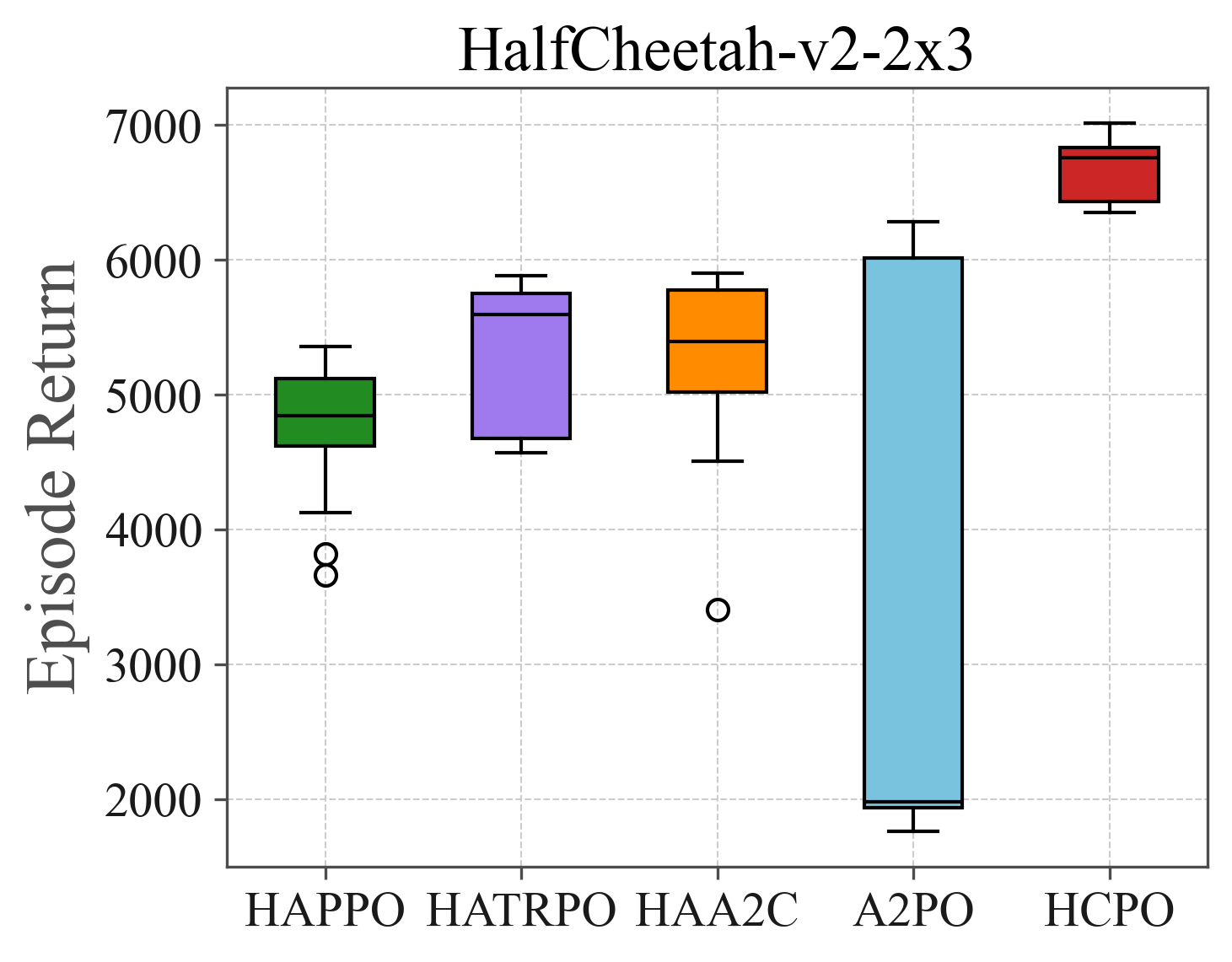}
	\end{subfigure}%
	\begin{subfigure}{0.23\textwidth}
		\includegraphics[width=\linewidth]{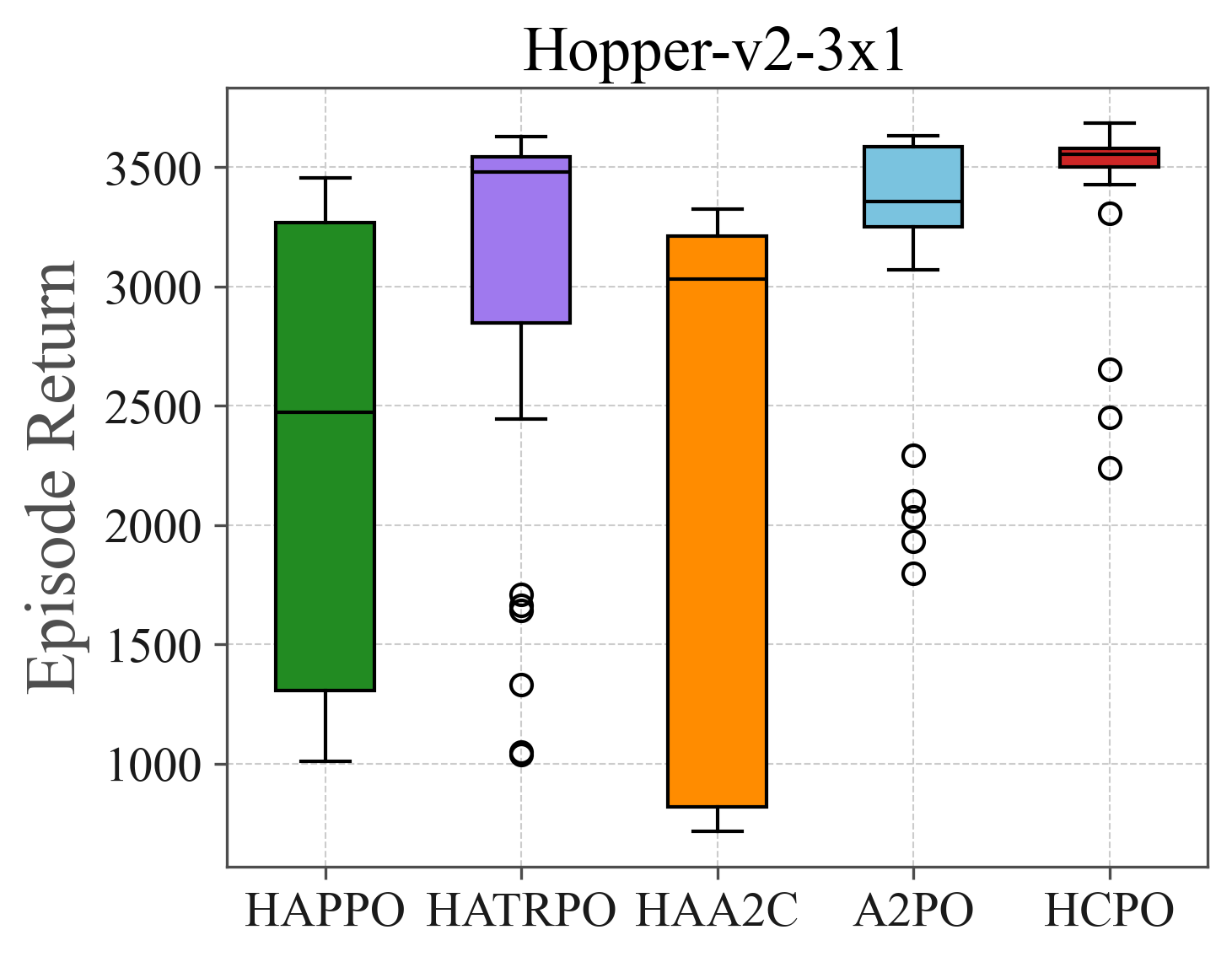}
	\end{subfigure}
	\begin{subfigure}{0.23\textwidth}
		\includegraphics[width=\linewidth]{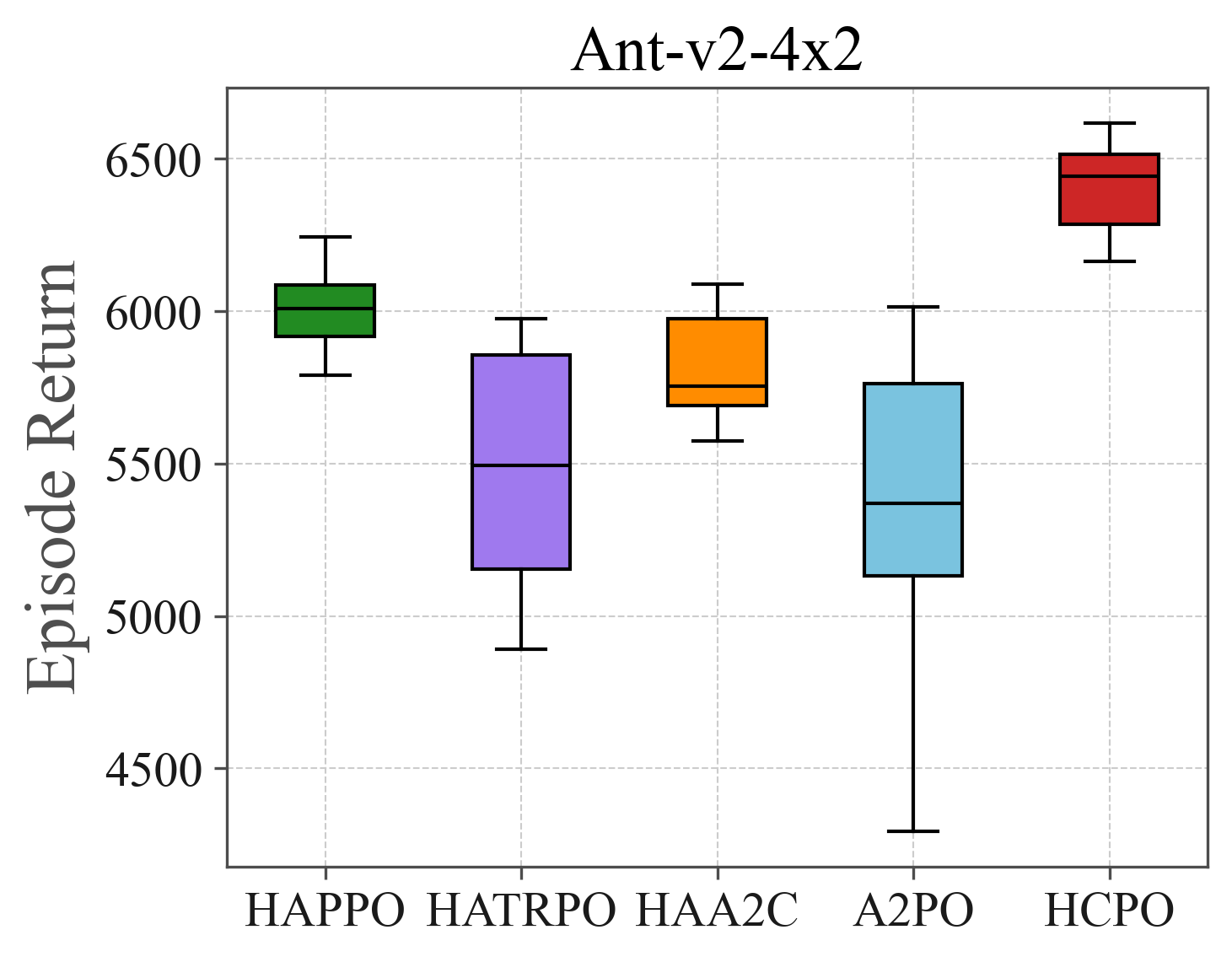}
	\end{subfigure}
	\caption{Comparison of the return performance of different baselines in the MuJoCo task.}\label{boxfigsapx}
\end{figure}
\begin{figure*}[t]
	\centering
	\includegraphics[width=0.9\textwidth]{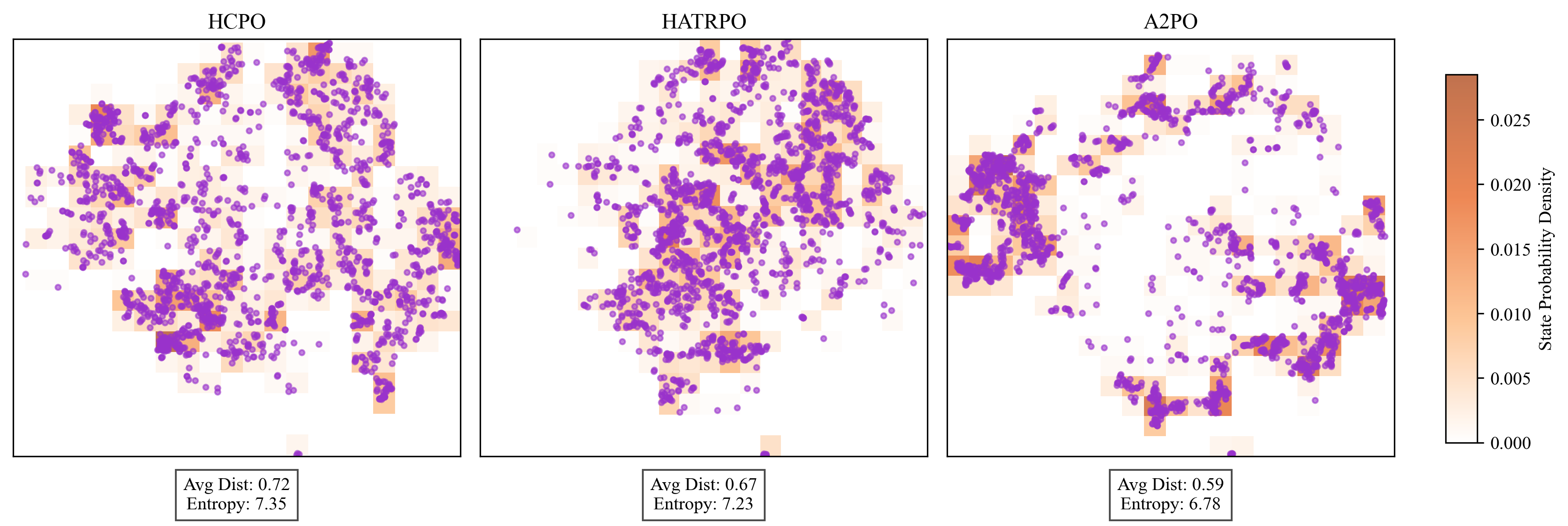}
	\caption{Exploration comparison: t-SNE visualization and entropy analysis in \textit{Walker2d-v2-6\(\times\)1}.}\label{HCPOentropy1apx}
\end{figure*}

\subsection{C.3 Multi-agent Particle Environment}
We present our detailed experimental hyperparameters in Table~\ref{mpepara}. The results of each algorithm under three random seeds on MPE benchmark are shown in Figure~\ref{mpefigsapx}. In general, HCPO demonstrates comparable performance across all three tasks, highlighting its adaptability and effectiveness in multi-agent settings. Specifically, in \textit{simple\_spread\_v2-discrete} and \textit{simple\_speaker\_listener\_v3-discrete} tasks, HCPO exhibits rapid policy improvement in the early stage of training (0-2 million steps). This indicates that HCPO algorithm has a high cooperative efficiency and sufficient exploration. Although HCPO shows slower initial cooperation efficiency in the \textit{simple\_reference\_v2-discrete} task compared to A2PO, it ultimately outperforms A2PO in final convergence performance, confirming its long-term learning advantage. Furthermore, compared with HATRPO and A2PO, HCPO algorithm shows significant stability and robustness. In summary, HCPO not only maintains competitive convergence speed but also matches or exceeds the performance of the strong baselines through its hierarchical mechanism. 
\begin{table}[htbp]
	\centering
	\begin{tabular}{llllll}
		\toprule
		hyperparameters & value & hyperparameters & value & hyperparameters & value \\
		\cmidrule(r){1-2}
		\cmidrule(r){3-4}
		\cmidrule(r){5-6}
		actor lr & 5e-4 & critic lr & 5e-4 &  actor mini batch & 1 \\
		critic mini batch & 1 & gamma & 0.99 &   batch size & 4000 \\
		network & mlp & linear lr decay & False & critic epoch	& 5 \\
		clip param & 0.2 & entropy coef & 0.01 & backtrack coef & 0.8 \\
		conductor kl-threshold  &  0.01 & kl-threshold & 0.005&accept-ratio & 0.5 \\
		K  &  10 &  eval episode  &  40  &  &  \\
		\bottomrule
	\end{tabular}
    \caption{Common hyperparameters in MPE}
    \label{mpepara}
\end{table}
\begin{figure}[htbp]  
	\centering
	\begin{subfigure}{0.32\columnwidth}
		\includegraphics[width=\linewidth]{mpefig1.png}
	\end{subfigure}
	\begin{subfigure}{0.32\columnwidth}
		\includegraphics[width=\linewidth]{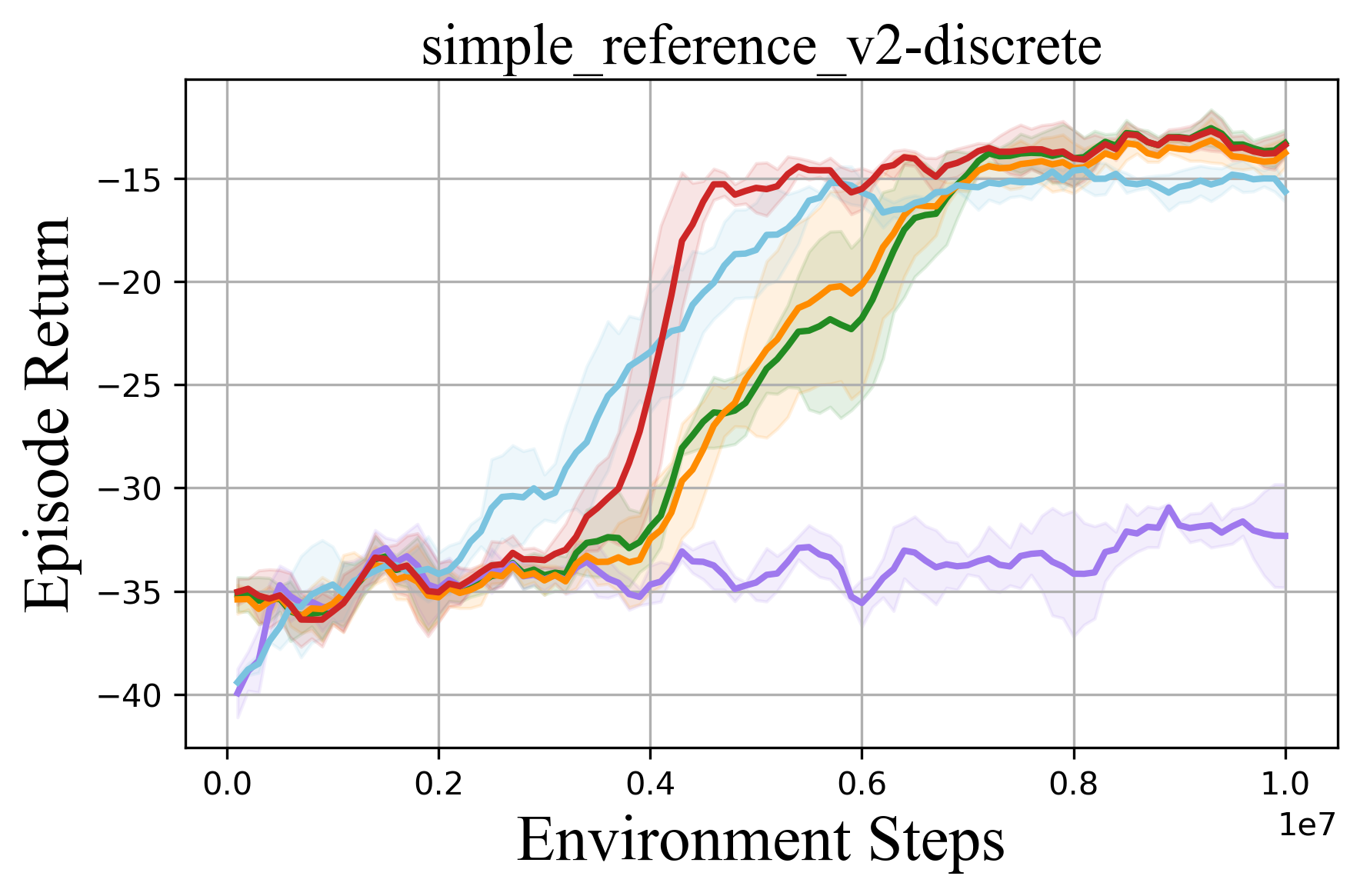}
	\end{subfigure}
	\begin{subfigure}{0.32\columnwidth}
		\includegraphics[width=\linewidth]{mpefig3.png}
	\end{subfigure}
	\caption{Performance comparison of HCPO and other strong MARL algorithms across different MPE tasks.}
	\label{mpefigsapx}
\end{figure}

\subsection{C.4 Ablation Studies}
For several key components of HCPO, we conduct ablation studies with results shown in Figure~\ref{ablfigsapx}. In Figure~\ref{subfig13a}, we explore the impact of removing the conductor and varying the number of instructions (hyperparameter \(K\)) on performance. \textcolor{black}{HCPO with the conductor shows a faster increase in winning rate and a higher final rate than without a conductor, demonstrating its effectiveness in boosting cooperation efficiency.} The performance is also influenced by \( K \). Although increasing \( K \) enhances adaptability to the environment, beyond a certain threshold, further increments in \( K \) introduce unnecessary complexity without substantial performance gains. Therefore, it is important to balance performance and resource consumption when selecting \( K \). In Figure~\ref{subfig13b}, we examine the hyperparameter \(\delta_1\), which represents the KL-divergence constraint of centralized conductor's policy. Results show that HCPO is relatively insensitive to \(\delta_1\), but the configuration with \(\delta_1 = 0.01\) demonstrates better performance compared to the other settings. \textcolor{black}{In Figure~\ref{subfig13c}, we evaluate HCPO under four conductor configurations: a centralized conductor with global information, a random conductor (non-learning baseline), no conductor and local conductors only based on local observations (trained via cross-entropy learning, constituting the core of our HCPO algorithm). Figure~\ref{subfig13d} presents the final episode returns in a boxplot format. The results show that HCPO with local conductors achieves a median return comparable to that of a centralized conductor, while substantially outperforming the variant without any conductor. Furthermore, replacing the learned instruction preference distribution with a non-learning conductor that outputs uniformly random instructions leads to inferior performance. These findings collectively validate the effectiveness of our proposed update mechanisms.}
\begin{figure}[htbp]
	\centering
	\begin{subfigure}{0.4\columnwidth}
		\includegraphics[width=\linewidth]{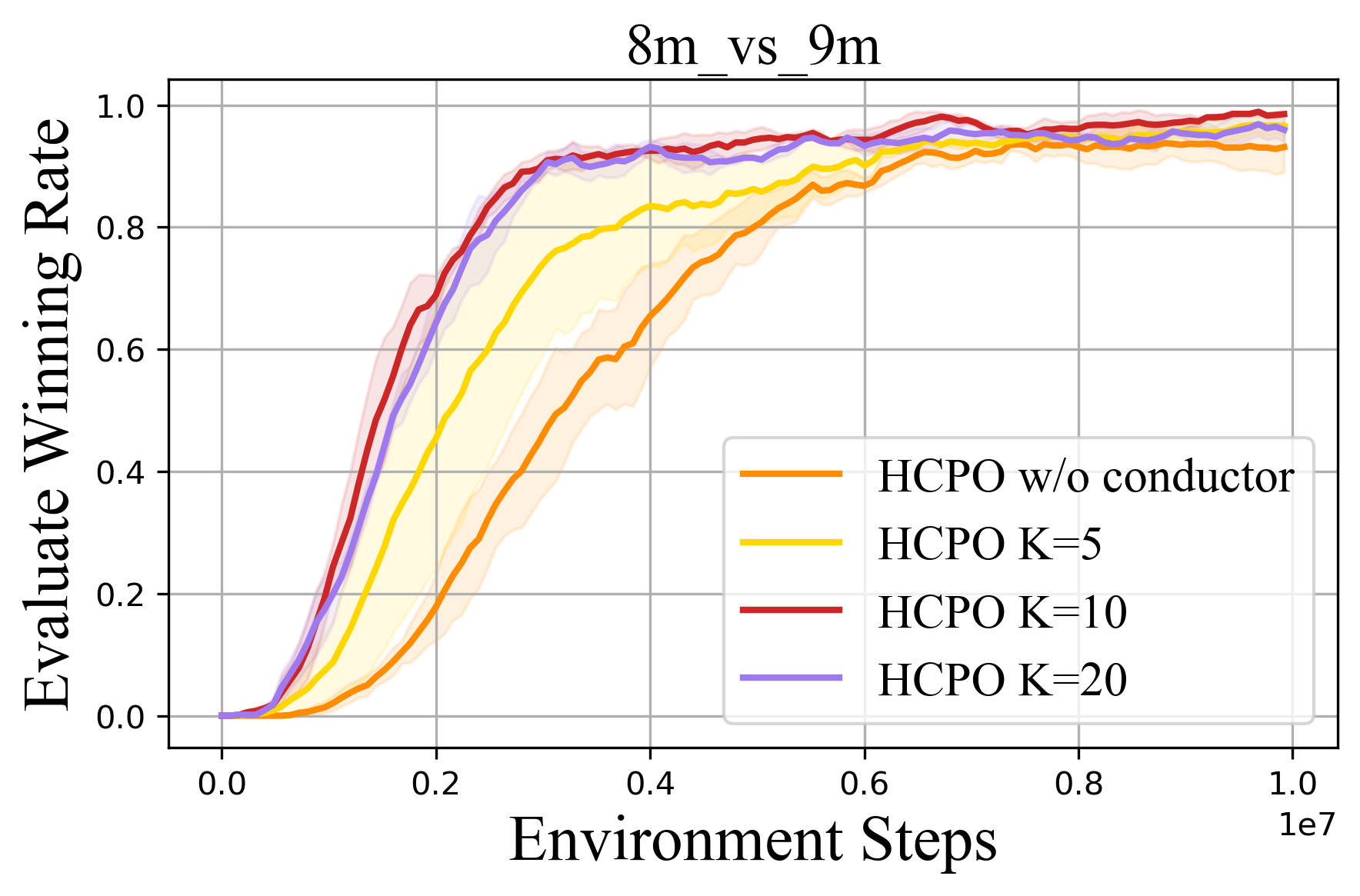}
		\caption{}
		\label{subfig13a}
	\end{subfigure}%
	\begin{subfigure}{0.4\columnwidth}
		\includegraphics[width=\linewidth]{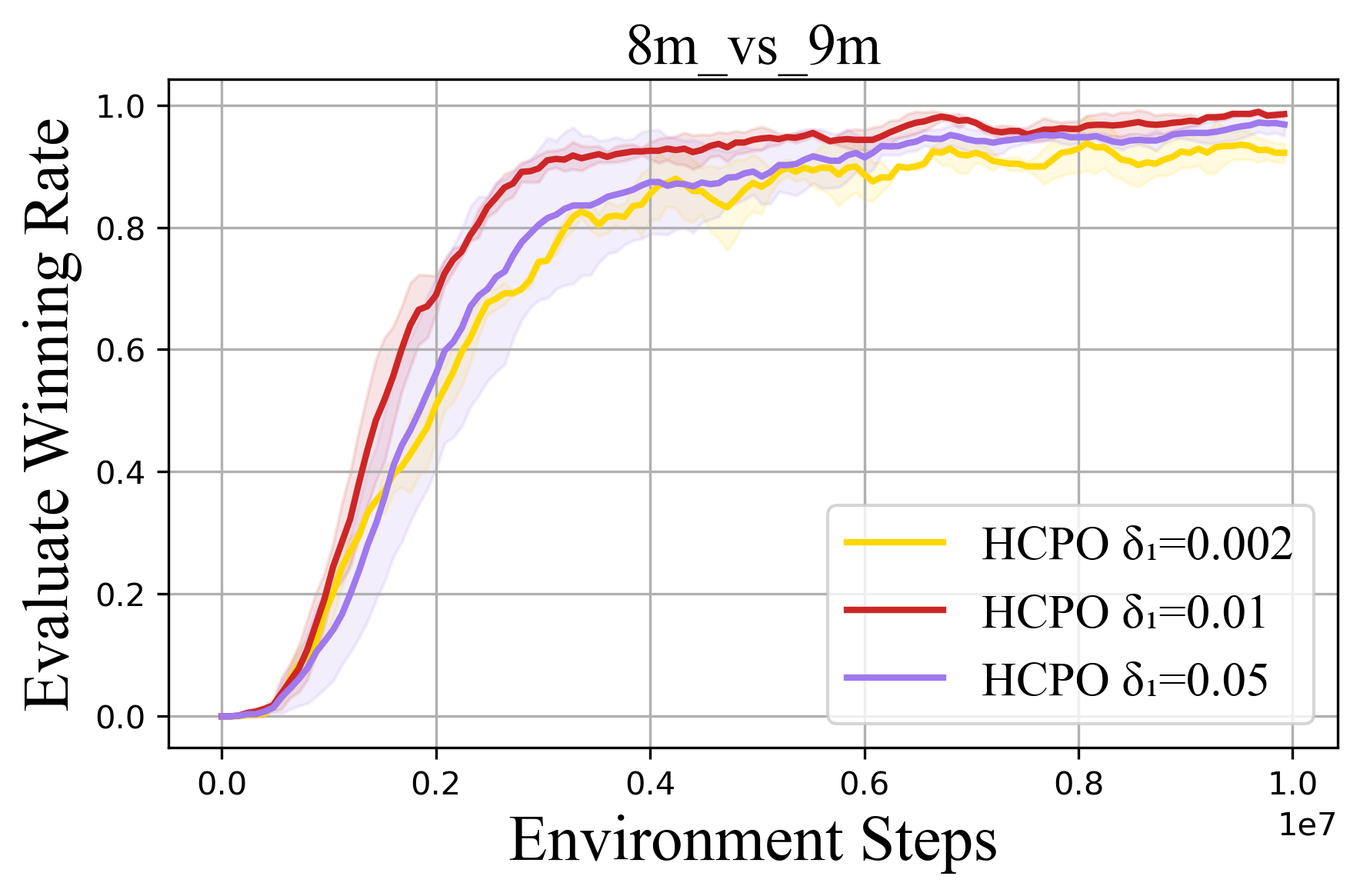}
		\caption{}
		\label{subfig13b}
	\end{subfigure}
	\begin{subfigure}{0.4\columnwidth}
		\includegraphics[width=\linewidth]{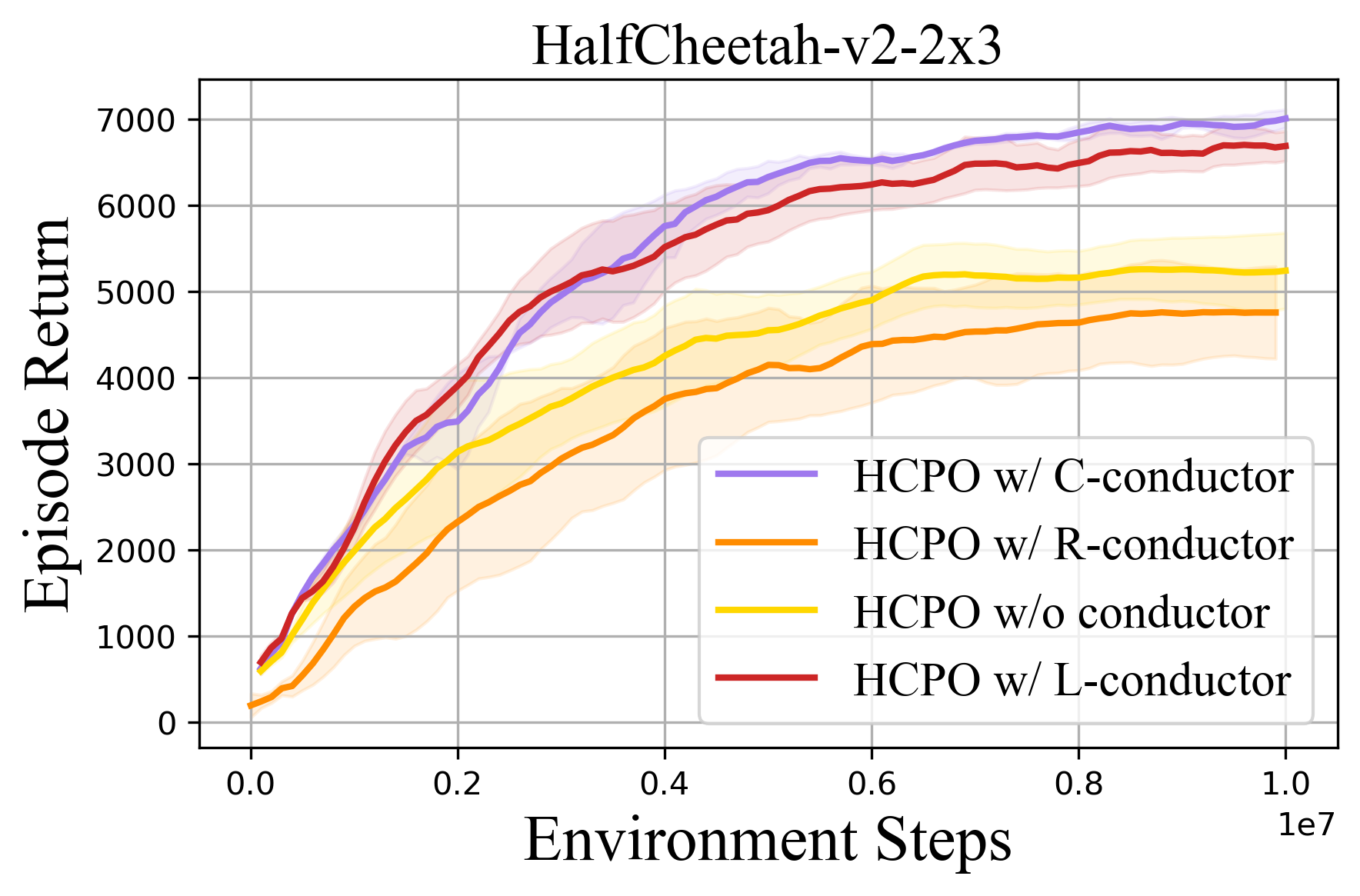}
		\caption{}
		\label{subfig13c}
	\end{subfigure}
		\begin{subfigure}{0.4\columnwidth}
		\includegraphics[width=\linewidth]{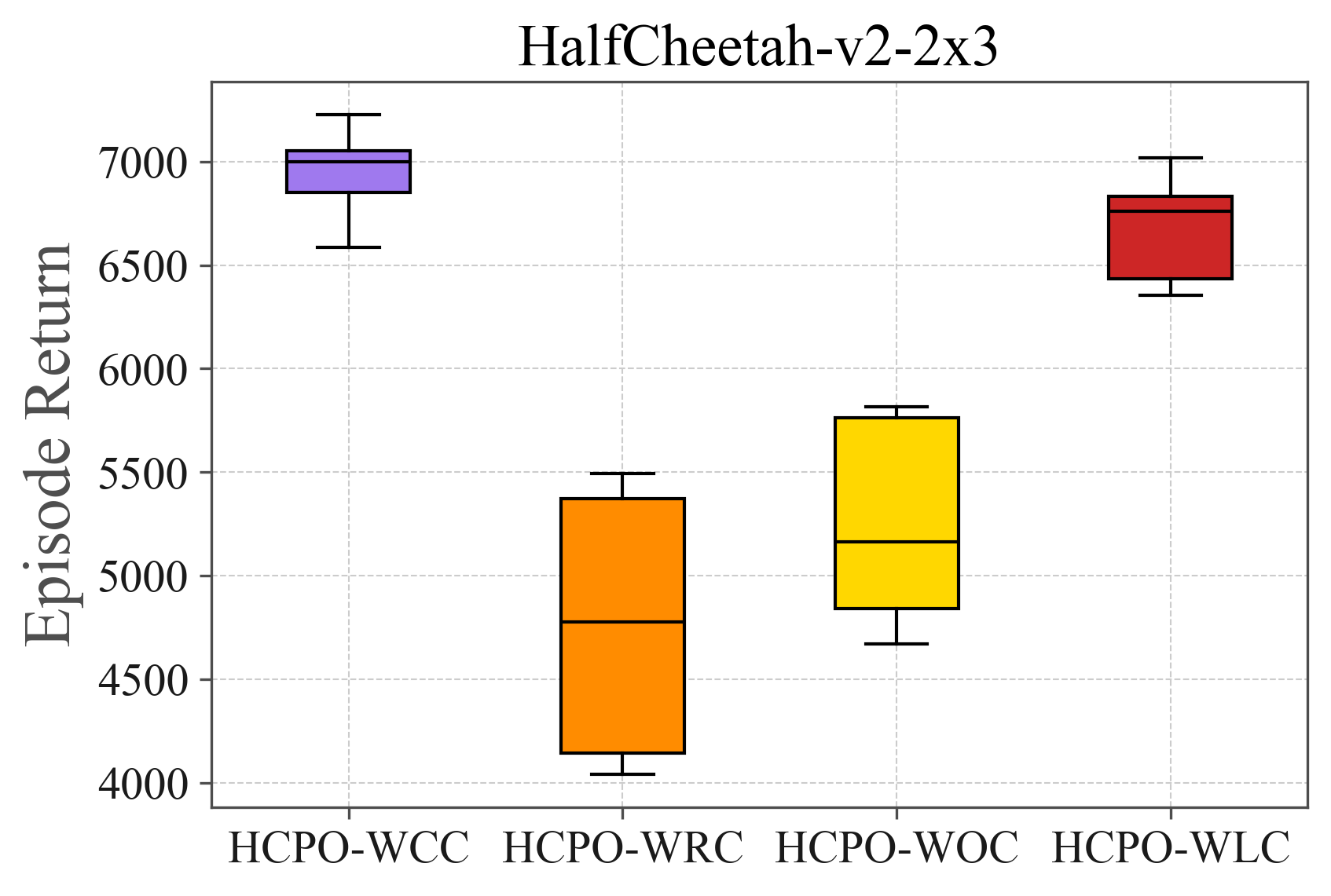}
		\caption{}
		\label{subfig13d}
	\end{subfigure}
	\caption{Ablation studies.}\label{ablfigsapx}
\end{figure}

\end{document}